\documentclass[letterpaper]{article} 
\usepackage{aaai2026}  
\usepackage{times}  
\usepackage{helvet}  
\usepackage{courier}  
\usepackage[hyphens]{url}  
\usepackage{graphicx} 
\usepackage{natbib}  
\usepackage{caption} 
\usepackage{algorithm}
\usepackage{algorithmic}

\usepackage{newfloat}
\usepackage{listings}
\DeclareCaptionStyle{ruled}{labelfont=normalfont,labelsep=colon,strut=off} 
\floatstyle{ruled}
\newfloat{listing}{tb}{lst}{}
\floatname{listing}{Listing}
\usepackage{subcaption}
\usepackage{booktabs} 

\usepackage{latexsym}
\usepackage{amssymb} 
\usepackage{amsmath}
\usepackage{mathtools}       
\usepackage{amsthm}          
\usepackage{bm}              

\providecommand{\argmax}{\operatornamewithlimits{argmax}} 
\DeclareMathOperator{\Tr}{Tr}     
\DeclareMathOperator{\Cov}{Cov}   
\providecommand{\R}{\mathbb{R}} 
\providecommand{\E}{\mathbb{E}} 
\providecommand{\T}{\ensuremath{^{\mathrm{T}}}} 
\providecommand{\rmd}{\mathrm{d}} 

\renewcommand{\geq}{\geqslant} 

\newtheorem{theorem}{Theorem}
\newtheorem{lemma}{Lemma}
\newtheorem{proposition}{Proposition}

\newtheorem{assumption}{Assumption}

\usepackage{comment}

\nocopyright
\title{Sample-Efficient Hypergradient Estimation\\for Decentralized Bi-Level Reinforcement Learning}
\author {
    Mikoto Kudo\textsuperscript{\rm 1,2},
    Takumi Tanabe\textsuperscript{\rm 3},
    Akifumi Wachi\textsuperscript{\rm 3},
    Youhei Akimoto\textsuperscript{\rm 1,2}
}
\affiliations {
    \textsuperscript{\rm 1}University of Tsukuba,\\
    \textsuperscript{\rm 2}RIKEN AIP,\\ 
    \textsuperscript{\rm 3}LY Corporation\\
    mikoto@bbo.cs.tsukuba.ac.jp,
    \{akifumi.wachi, takumi.tanabe\}@lycorp.co.jp, 
    akimoto@cs.tsukuba.ac.jp
}

\begin{document}

\maketitle

\begin{abstract}
Many strategic decision-making problems, such as environment design for warehouse robots, can be naturally formulated as bi-level reinforcement learning (RL), where a leader agent optimizes its objective while a follower solves a Markov decision process (MDP) conditioned on the leader's decisions.
In many situations, a fundamental challenge arises when the leader cannot intervene in the follower's optimization process; it can only observe the optimization outcome. 
We address this decentralized setting by deriving the hypergradient of the leader's objective, i.e., the gradient of the leader's strategy that accounts for changes in the follower's optimal policy.
Unlike prior hypergradient-based methods that require extensive data for repeated state visits or rely on gradient estimators whose complexity can increase substantially with the high-dimensional leader's decision space, we leverage the Boltzmann covariance trick to derive an alternative hypergradient formulation. 
This enables efficient hypergradient estimation solely from interaction samples, even when the leader's decision space is high-dimensional.
Additionally, to our knowledge, this is the first method that enables hypergradient-based optimization for 2-player Markov games in decentralized settings. 
Experiments highlight the impact of hypergradient updates and demonstrate our method's effectiveness in both discrete and continuous state tasks.
\end{abstract}

\begin{links}
    \link{Code}{https://github.com/akimotolab/BC-HG}
\end{links}


\section{Introduction}

Bi-level reinforcement learning (RL) is a hierarchical framework capable of representing complex objective structures.
In bi-level RL, the lower-level (i.e., \emph{follower}) typically involves policy optimization within a \emph{Configurable MDP}~\citep{metelli18a}, which is an MDP parameterized by external variables $\theta \in \Theta$.
Meanwhile, the objective function of the upper-level (i.e., \emph{leader}) depends on the lower-level optimal policy, creating a nested dependency structure as follows:
\begin{align*}
    &\max_{\theta \in \Theta, g^{\dagger} \in \mathcal{G}}\ J_L(\theta, g^{\dagger})&\text{(Leader)}\\
    &\mathrm{s.t.}\ g^{\dagger} \in \mathcal{G}^\dag(\theta) := \argmax_{g \in \mathcal{G}} J_F(\theta, g),&\text{(Follower)}
\end{align*}
where $\Theta$ is the set of possible parameters in the upper-level, and $\mathcal{G}$ is the set of all possible policies in the lower-level.
The set of lower-level optimal policies, $\mathcal{G}^\dag(\theta)$, depends on the upper-level variable and is called the \emph{best response} to $\theta$.

Due to its versatility in modeling hierarchical dependencies, bi-level RL has been applied to many domains, including reinforcement learning from human feedback (RLHF)~\citep{Chakraborty2023-gw,Shen2024-ma,Yang2025-lr}, reward shaping~\citep{Shen2024-ma,Yang2025-lr}, safe RL~\citep{Zhi2025-s}, model-based RL~\citep{Rajeswaran2020-gv}, and advertising planning~\citep{Muneeb2019-no}. 
Yet an important challenge remains under-explored: decentralized learning, where the leader cannot intervene in the follower's optimization process and must treat the follower algorithm as fixed.
This occurs when followers rely on pre-defined standard algorithms that are impractical or undesirable to modify. 
For example, in environment design for autonomous warehouse robots, the designer optimizes environment parameters while leaving each robot's built-in adaptation algorithm (e.g., LQR) unchanged.
The technical challenge is then to estimate leader updates that account for best-response shifts without controlling follower optimization.

To address this challenge, we propose a method for decentralized bi-level RL, where the leader is agnostic to the follower's algorithm and cannot control follower updates.
We study two settings: 1) configurable MDPs, where the leader influences the follower through environment/configuration parameters; and 2) Markov games (MGs), where both leader and follower have policies and influence each other through policy interaction.
In both settings, the upper-level objective is the leader's discounted cumulative reward, while the lower-level objective is the follower's entropy-regularized discounted cumulative reward.
For each setting, we derive the hypergradient of the upper-level objective, which can be estimated from online interaction samples consisting of states, follower and leader actions, and the leader's immediate rewards.
By updating the leader's variables or policy using the estimated hypergradient, our method enables steepest ascent updates that ``anticipate'' changes in the follower's best response.
We also assume access to the follower's actions (during exploration and exploitation) and learning outcomes.
This information-access assumption makes online hypergradient estimation tractable while preserving the decentralized, non-interventionist setting.

Two existing methods applicable to our configurable MDP setting, HPGD~\citep{Thoma2024-em} and SoBiRL~\citep{Yang2025-lr}, require multiple visits to the same state within a sample batch to estimate the hypergradient.
However, in continuous state spaces or in large discrete state spaces relative to batch size, this requirement is impractical without generating additional trajectories or arbitrary state resets, severely limiting real-world applicability.
Our derived hypergradient circumvents this issue via the ``Boltzmann covariance trick,'' enhancing scalability by enabling estimation solely from interaction samples, even in continuous-state tasks with high-dimensional configuration spaces.
Furthermore, we extend this result to 2-player MGs.
To our knowledge, this is the first work to derive the hypergradient for 2-player MGs under entropy regularization.

We empirically evaluate our proposed approach on Configurable MDPs and MGs with both discrete and continuous state spaces.
In Configurable MDP settings, our method demonstrates robustness against weak entropy regularization and scalability to high-dimensional leader parameters. 
Prior methods like HPGD often degrade or become trapped in local minima due to estimation errors in these scenarios.
Furthermore, in MG settings, our method identifies policies that successfully induce favorable follower behaviors, which the baselines fail to discover.
These results validate that our method enables effective hypergradient-based optimization solely from interaction samples.


\section{Preliminaries}

\subsubsection{Entropy Regularization}

Throughout this paper, we assume that a follower solves a Markov decision process (MDP) with entropy regularization. It is formulated as follows. Let $(\mathcal{S}, \mathcal{B}, p, \rho_0, r_F, \gamma_F)$ be the MDP, where  $\mathcal{S}$ is the state space; $\mathcal{B}$ the action space of the follower; $p: \mathcal{S}\times\mathcal{B}\times\mathcal{S}\to[0,1]$ the state transition probability; $\rho_0: \mathcal{S}\to[0,1]$ the initial state distribution; $r_F:\mathcal{S}\times\mathcal{B}\to\R$ the reward function of the follower; and $\gamma_F\in[0,1)$ the discount factor for the follower. 
The objective of the follower is to find a stochastic policy $g\in\mathcal{G}: \mathcal{S} \times\mathcal{B}\to[0,1]$ that maximizes the expected cumulative reward under entropy regularization:
\begin{equation}
    \max_{g \in \mathcal{G}} J(g) := 
    \E_{\tau}\left[ \sum_{t=0}^\infty \gamma_F^t \left(r_F(s_t, b_t) + \beta H(g;s_t)\right)\right],
    \label{eq:lower}
\end{equation}%
where $\mathcal{G}$ is the set of all possible policies on $\mathcal{S} \times \mathcal{B}$; $\beta > 0$ is the constant determining the strength of the entropy regularization; $H(g; s) := -\E_{b \sim g(\cdot \mid s)}[\log g(b \mid s)]$ is the entropy of $g$ given $s \in \mathcal{S}$; and the expectation is taken for a trajectory $\tau := (s_0, b_0, s_1, b_1, \dots)$ generated by $g$, i.e., $p(\tau\mid g) = \rho_0(s_0) \prod_{t=0}^{\infty} p(s_{t+1} \mid s_{t}, b_{t}) g(b_{t}\mid s_{t})$.

The main technical reason to assume entropy regularization is to guarantee the existence and uniqueness of the optimal policy as a stochastic policy. Entropy regularization is often assumed in inverse RL algorithms as a natural approach to model the stochasticity in the expert policy.
Under entropy regularization, the value and action value functions (also called the soft value and soft action value functions) are defined as 
\begin{equation*}
    V_F^{g}(s):=\E_\tau\left[\sum_{t=0}^\infty\gamma_F^t \left(r_F(s_t,b_t) + \beta H(g;s_t)\right)~\middle|~s_0=s\right]
\end{equation*}%
and $Q_F^{g}(s,b) := r_F(s, b) + \gamma_F\E_{s' \sim p(\cdot \mid s, b)}[ V_F^{g}(s') ]$.
When the policy model is sufficiently expressive, the optimal policy is uniquely determined by the optimal value functions
$V_F^{\star}(s) = \max_{g\in \mathcal{G}} V_F^{g}(s)$ and $Q_F^{\star}(s,b) = \max_{g\in \mathcal{G}} Q_F^{g}(s, b)$ for each $s \in \mathcal{S}$ as a Boltzmann distribution $g^\star(b \mid s) = \exp \left(\beta^{-1}\left(Q_F^\star(s, b) - V_F^\star(s)\right)\right)$.
The uniqueness of the optimal policy is essential for bi-level RL problems.

\subsubsection{Configurable Markov Decision Process}

A configurable Markov decision process (configurable MDP) generalizes the standard MDP by introducing a configurable parameter that parameterizes the state transition probability and the reward function.
A configurable MDP $\mathcal{M}_{\theta}$, parameterized by ${\theta} \in \Theta$, is defined by the tuple ($\mathcal{S}$, $\mathcal{B}$, $p^{\theta}$, $\rho_0^{\theta}$, $r_F^{\theta}$, $\gamma_F$), where, differently from the standard MDP, $p^{\theta}$, $\rho_0^{\theta}$, and $r_F^{\theta}$ are parameterized by ${\theta} \in \Theta$.
Given $\mathcal{M}_{\theta}$ for a fixed ${\theta} \in \Theta$, the follower maximizes the expected cumulative reward under entropy regularization; that is, it solves \eqref{eq:lower} with the parameterized objective $J_{\theta}$, where $p$, $\rho_0$, and $r_F$ are replaced by their parameterized counterparts $p^{\theta}$, $\rho_0^{\theta}$, and $r_F^{\theta}$.

The leader's objective is to maximize its utility given the follower's best response.
This is formulated as a bi-level reinforcement learning problem:
\begin{equation}
    \max_{{\theta} \in \Theta} J_L({\theta},g^{\theta\dagger}), \quad \text{where} \quad g^{\theta\dagger} := \argmax_{g\in\mathcal{G}} J_{\theta}(g) ,\label{eq:upper}
\end{equation}
where $J_L$ denotes the leader's utility under configuration ${\theta}$, and $g^{\theta\dagger}$ is the follower's best response to ${\theta}$.
Due to entropy regularization in the lower-level problem, $g^{\theta\dagger}$ is uniquely determined for every $\theta\in\Theta$ and corresponds to the $\theta$-conditioned optimal Boltzmann distribution:
\begin{equation*}
    g^{\theta\dagger}(b\mid s)=\exp\left(\dfrac{Q_F^{\theta\dagger}(s,b)-V_F^{\theta\dagger}(s)}{\beta}\right),
\end{equation*}
where $Q_F^{\theta\dagger}$ and $V_F^{\theta\dagger}$ are the optimal value functions of $\mathcal{M}_{\theta}$.
Consequently, this problem is well-defined due to the existence and uniqueness of the best response $g^{{\theta} \dagger}$.
In the subsequent sections, we assume that the follower’s policy always coincides with the optimal policy specified above.
This assumption is not overly restrictive, as non-asymptotic convergence to an $\epsilon$-optimal policy for the follower has been established by \citet{Thoma2024-em} for standard entropy-regularized RL algorithms, such as Soft Value Iteration, Q-learning, and Natural Policy Gradient.

In this paper, we focus on the scenario where the leader's utility is defined by the discounted cumulative reward, which depends on the follower's policy, and a regularization term that depends only on the configuration:
\begin{equation}
    J_L({\theta}, g^{\theta\dagger}) := \E_\tau\left[ \sum_{t=0}^\infty \gamma_L^t r_L^{\theta}(s_t, b_t)\right] + \Phi_L({\theta}), \label{eq:leader-utility}
\end{equation}
where  $\gamma_L$ is the leader's discount factor, $r_L^{\theta}$ is the leader's reward function, and $\Phi_L$ defines the regularization for configuration ${\theta}$.
Here, $\Phi_L$ is a task-specific regularizer for the leader (e.g., an $\ell_2$ penalty on $\theta$), and is distinct from the entropy regularization used in the follower objective.


\subsubsection{2-Player Markov Game}

A 2-player MG can be viewed as an extension of configurable MDPs, where the leader acts as an agent with a policy parameterized by $\theta \in \Theta$, and the environment is influenced indirectly through the leader's actions.
It is defined by the tuple ($\mathcal{S}$, $(\mathcal{A},\mathcal{B})$, $p$, $\rho_0$, $(r_L,r_F)$, $(\gamma_L,\gamma_F)$), where $\mathcal{S}$ is the state space; $\mathcal{A}$ and $\mathcal{B}$ the action spaces of the leader and the follower, respectively; $p: \mathcal{S}\times\mathcal{A}\times\mathcal{B}\times\mathcal{S}\to[0,1]$ the state transition probability; $\rho_0: \mathcal{S}\to[0,1]$ the initial state distribution; $r_L$ and $r_F:\mathcal{S}\times\mathcal{A}\times\mathcal{B}\to\mathbb{R}$ the reward functions of the leader and the follower; and $\gamma_L, \gamma_F \in [0, 1)$ the discount factors for the leader and the follower.
The leader's policy and the follower's policy are denoted as $f_{\theta}: \mathcal{S}\times\mathcal{A}\to[0,1]$ and $g:\mathcal{S} \times \mathcal{A}\times\mathcal{B}\to[0,1]$, respectively.\footnote{In this paper, we consider the case where the follower observes the leader's action before making a decision. However, it is straightforward to extend our results to cases where the follower does not observe the leader's action, i.e., $g:\mathcal{S} \times\mathcal{B}\to[0,1]$.}
Let the leader's policies be parameterized by ${\theta} \in \Theta$.
Similar to the configurable MDP case, the follower maximizes the expected cumulative reward under entropy regularization, i.e., it solves \eqref{eq:lower}.
The difference lies in that $r_F^{\theta}(s, b)$ is replaced by $r_F(s, a, b)$ and $p^{\theta}(s' \mid s, b)$ is replaced by $p(s' \mid s, a, b) f_{{\theta}}(a' \mid s')$.
That is, instead of directly controlling the follower's reward and state transition, these are indirectly affected by the leader's action.
The leader's optimization problem is identical to \eqref{eq:upper}, where the expected cumulative reward $J_L$ is defined as:
\begin{equation}
    J_L({\theta}, g^{\theta\dagger}) := \E_\tau\left[ \sum_{t=0}^\infty \gamma_L^t r_L(s_t, a_t, b_t)\right] + \Phi_L({\theta}). \label{eq:leader-utility-MG}
\end{equation}%
Analogous to the configurable MDP case, we assume that the follower's policy is optimal.

We define the value and action-value functions for the follower $g$ under the leader $f_{\theta}$ as:
\begin{multline*}
        V_F^{{\theta} {g}}(s, a):=\\ \E_\tau\left[\sum_{t=0}^\infty\gamma_F^t \left(r_F(s_t,a_t,b_t) + \beta H(g;s_t, a_t)\right)\middle|\begin{array}{l}s_0=s\\a_0=a\end{array}\right]
\end{multline*}%
and $Q_F^{{\theta} {g}}(s, a, b) := r_F(s, a, b) + \gamma_F\E_{s'}\E_{a'}\big[ V_F^{\theta {g}}(s', a') \big]$, where the expectation $\E_\tau$ is taken over a trajectory generated by $f_{\theta}$ and $g$, i.e., $p(\tau\mid \theta, g) = \rho_0(s_0) \prod_{t=0}^{\infty} p(s_{t+1} \mid s_{t}, a_{t}, b_{t}) g(b_{t}\mid s_{t}, a_{t}) f_{{\theta}}(a_{t}\mid s_{t})$, $\E_{s'}$ is taken over next state $s' \sim p(\cdot \mid s, a, b)$, and $\E_{a'}$ is taken over next leader action $a' \sim f_\theta(\cdot \mid s')$.
In the definition of $V_F^{\theta g}$, the leader's action $a$ is given as an input because it is observed before the follower's decision.

The leader's action-value functions are defined as:
\begin{equation*}
    Q_L^{{\theta} {g}}(s,a,b) :=\E_\tau\left[\sum_{t=0}^\infty\gamma_L^t r_L(s_t,a_t,b_t)~\middle|~\begin{array}{l}s_0=s\\a_0=a\\b_0=b\end{array}\right].
\end{equation*}
Note that the leader does not employ entropy regularization; hence, the value functions for the leader are the standard ones.
For conciseness, the leader's action-value function under the leader policy $f_{{\theta}}$ and the corresponding best response $g^{\theta\dagger}$ is denoted as
$Q_L^{{\theta} \dagger} := Q_L^{{\theta} {g^{\theta\dagger}}}$.
In addition, the follower's optimal value and action-value functions under $f_{{\theta}}$, i.e., the value functions of $f_{{\theta}}$ and the optimal policy $g^{\theta\dagger}$, are denoted as
$V_F^{{\theta} \dagger} :=  V_F^{{\theta} {g^{\theta\dagger}}}$ and $Q_F^{{\theta} \dagger} := Q_F^{{\theta} {g^{\theta\dagger}}}$.



\section{Related Works}


\paragraph{Bi-Level RL with Cumulative Upper-Level Objectives in Configurable MDPs}
Several studies address bi-level RL with cumulative reward objectives at the upper level, such as \citet{Thoma2024-em} and \citet{Yang2025-lr}. 
Both derive hypergradients assuming access to the follower's entropy-regularized $\epsilon$-optimal best response, similar to our setting. 
Consequently, their methods are applicable to decentralized follower scenarios.

However, these approaches share a common limitation: hypergradient estimation requires multiple trajectories originating from the same initial state with diverse initial follower actions.
Such trajectories are generally inaccessible unless the state space is discrete and small enough to allow repeated visits to identical states.
\citet{Thoma2024-em} circumvent this issue by assuming access to an oracle capable of generating trajectories from arbitrary initial states and actions.  
Although this challenge can be mitigated without such an oracle, either by constraining the dimensionality of upper-level variables as in \citet{Thoma2024-em} or by truncating the hypergradient to its first-step component as in \citet{Yang2025-lr}, these simplifications and approximations degrade performance in practical settings as seen later.

In contrast, we derive an alternative hypergradient formulation that can be estimated solely from experienced trajectories.
Our method eliminates the need for additional trajectories, enabling scalability to large state spaces and high-dimensional leader parameters.

\paragraph{Bi-Level RL in Markov Games}
In 2-player MG settings with a decentralized follower, three main approaches have been proposed: two-timescale gradient updates~\citep{Rajeswaran2020-gv,Vu2022-jq}, total derivative gradient approaches~\citep{Yang2023-tf}, and operator-based methods~\citep{Zhang2020-hw}.
Two-timescale gradient updates~\citep{Rajeswaran2020-gv,Vu2022-jq} employ the first-order gradient for the leader's updates and two different update frequencies to maintain the follower's optimality.
However, this approach may fail to find the optimal leader policy, particularly when the leader's and follower's objectives conflict, as demonstrated in Section~\ref{sec:experiment_MG_discrete}.
\citet{Yang2023-tf} propose a total-derivative gradient method, ST-MADDPG.
However, they assume a deterministic follower policy unconditioned on the leader’s action. This, combined with the lack of mathematical detail regarding policy updates, makes adapting their method to our setting nontrivial.
Moreover, this method requires computing the inverse Hessian product with respect to policy parameters, often incurring substantial computational costs and resulting in unstable updates.
Bi-AC~\citep{Zhang2020-hw}, on the other hand, is applicable with minor adjustments. 
This method builds on the Stackelberg stage-game operator, similar to Nash-Q~\citep{Hu2003-vz}.
However, this operator is not guaranteed to converge to the optimal leader policy or even to improve the leader's performance.
Further discussion is provided in Section~\ref{sec:experiment_MG_discrete}.

Our proposed method leverages the hypergradient to account for changes in the follower's best response during the leader's update, similar to total derivative gradient approaches~\citep{Yang2023-tf}, but without the computationally expensive inverse Hessian product.




\section{Proposed Approach}

A key difference from prior studies~\citep{Thoma2024-em,Yang2025-lr} is that we leverage the \emph{Boltzmann covariance trick} to resolve their limitation of requiring expectations over follower actions from the same state for hypergradient estimation.
This constraint necessitates elaborate trajectory collection procedures, which restricts its applicability.
Our method avoids this computation and enables hypergradient estimation from limited interaction history. 
To our knowledge, this is the first work deriving the hypergradient for 2-player MGs under entropy regularization.

\subsection{Hypergradient for Configurable MDPs}

The hypergradient of $J_L$ is the hypergradient of the cumulative reward (first term in \eqref{eq:leader-utility}) plus the gradient of the regularization term (second term), $\nabla \Phi_L({\theta})$, which is assumed to be analytically differentiable. Consequently, we omit the regularization term in the following derivation for simplicity.

To derive the hypergradient, we define the leader's value function under the follower's action as 
\begin{equation*}
    Q_L^{{\theta} {g}}(s, b) :=\E_\tau\left[\sum_{t=0}^\infty\gamma_L^t r_L(s_t,b_t)~\middle|\begin{array}{l}s_0=s\\b_0=b\end{array}\right].
\end{equation*}
Unlike the standard action-value function, the inputs are the state $s$ and the follower's action $b$, rather than the leader's action. 
Let $Q_L^{{\theta} \dagger} := Q_L^{{\theta} {g^{\theta\dagger}}}$ be the leader's value function under the best response $g^{\theta\dagger}$. 
Then, we define the \emph{Benefit} of the leader when the follower takes an action $b$ at state $s$ as
\begin{equation*}
B_L^{{\theta}\dagger}(s, b) := Q_L^{{\theta}\dagger}(s,b)-\E_{b\sim g^{{\theta}\dagger}(\cdot\mid s)}\left[Q_L^{{\theta}\dagger}(s,b)\right].
\end{equation*}
A positive Benefit indicates that the follower's action is relatively beneficial to the leader. 
Thus, the leader aims to encourage the follower to select actions with a positive Benefit more frequently.

Using the Benefit, we can compute the hypergradient as follows.
Its proof is provided in Appendix~\ref{apdx:proof_hg-configurable}.
\begin{theorem}
\label{thm:hg-configurable}
Suppose the gradients $\nabla_{\theta} r_L^{\theta}$, $\nabla_{\theta} r_F^{\theta}$, $\nabla_{\theta} \log \rho_0^{\theta}$, and $\nabla_{\theta} \log p^{\theta}$ are computable everywhere.
The hypergradient of \eqref{eq:leader-utility} is given by
\begin{multline}
    \nabla_{{\theta}}J_L({\theta}, g^{\theta\dagger})
    =\E_\tau\Bigg[\sum_{t=0}^\infty \gamma_L^t \bigg(\nabla_{\theta} r_L^{\theta}(s_t, b_t)\\
    + V_L^{{\theta}\dagger}(s_t) \nabla_{\theta}\log
    p^{\theta} (s_{t} \mid s_{t-1}, b_{t-1})\\
    + \frac{B_L^{{\theta}\dagger}(s_t, b_t)}{\beta} \nabla_{\theta} Q_F^{{\theta} \dagger}(s_t, b_t) \bigg) \Bigg],
\label{eq:hg_configurable}
\end{multline}
where $p^{\theta}(s_0\mid s_{-1},b_{-1})$ refers to $\rho_0^{\theta}(s_0)$, and
\begin{multline}
\nabla_{{\theta}}Q_F^{{\theta}\dag}(s,b) = \E_\tau\Bigg[\sum_{t=0}^\infty\gamma_F^t\nabla_{{\theta}}r_F^{{\theta}}(s_t,b_t)\\
+ \gamma_F^{t+1}V_F^{{\theta}\dag}(s_{t+1})\nabla_{{\theta}}\log p^{{\theta}}(s_{t+1}\mid s_t,b_t) \Bigg|\begin{array}{l}s_0=s\\b_0=b\end{array}\Bigg].\label{eq:nablaQF}
\end{multline}
\end{theorem}
The first two terms of \eqref{eq:hg_configurable} are interpreted as the direct effect of the differentiation of the cumulative reward resulting from the differentiation of the reward function and the transition probability. 
The third term is interpreted as the indirect effect due to the change in the follower's best response. 

\citet{Thoma2024-em} derive a hypergradient for the same bi-level RL problem, albeit with a different formulation.
Importantly, since our hypergradient formulation is mathematically equivalent to that of \citet{Thoma2024-em}, the non-asymptotic convergence guarantees established in their work apply directly to our method.
Our contribution is thus complementary: we retain the same gradient target but introduce a more efficient and oracle-free estimation mechanism, making the approach applicable in practical decentralized settings.

More specifically, the key distinction lies in the form of the third term of the hypergradient, leading to different implementations of hypergradient estimation.
The third term of the hypergradient in~\citep{Thoma2024-em} is derived as 
\begin{equation}
    \frac{1}{\beta} \E_\tau\left[\sum_{t=0}^\infty \gamma_L^t Q_L^{{\theta} \dagger}(s_t, b_t) \nabla_{\theta} A_F^{{\theta} \dagger}(s_t, b_t)  \right], \label{eq:third_term_thoma_grad}
\end{equation}
where $A_F^{{\theta} \dagger}(s, b) := Q_F^{{\theta} \dagger}(s, b) - V_F^{{\theta} \dagger}(s)$. 
This requires estimating the gradients $\nabla_{\theta} Q_F^{{\theta} \dagger}(s, b)$ derived in \eqref{eq:nablaQF} and $\nabla_{\theta} V_F^{{\theta} \dagger}(s) = \E_{b \sim g^{{\theta}\dagger}(\cdot\mid s)}\left[\nabla_{\theta} Q_F^{{\theta} \dagger}(s, b)\right]$ (derived in the proof). 
However, there is a difficulty in estimating these two gradients from samples. 
These gradients are estimated by Monte Carlo using trajectories of the follower's policy. The estimation of $\nabla_{\theta} Q_F^{{\theta} \dagger}(s, b)$ requires trajectories starting at state $s$ and follower action $b$. On the other hand, the estimation of $\nabla_{\theta} V_F^{{\theta} \dagger}(s)$ requires other trajectories starting at state $s$. 
That is, we need multiple trajectories starting at the same state with different follower actions to estimate $\nabla_{\theta} A_F^{{\theta} \dagger}(s, b)$. 
This assumption is challenging to satisfy when the state-action space is either discrete but large or continuous.
To tackle this challenge, \citet{Thoma2024-em} rely on an oracle that generates trajectories starting at an arbitrary state-action pair.
However, the availability of such an oracle is also difficult to guarantee, especially when the leader does not have full control of the follower.

With our hypergradient formula \eqref{eq:hg_configurable}, we do not need to estimate $\nabla_{\theta} V_F^{{\theta} \dagger}(s)$, and thus do not need multiple trajectories. 
Instead, we require estimating the Benefit $B_L^{{\theta} \dagger}(s, b)$, which we argue is easier to estimate, as it does not require multiple trajectories or an oracle.
The equivalence between \eqref{eq:third_term_thoma_grad} and the third term in \eqref{eq:hg_configurable} is described by the Boltzmann Covariance trick, that is:
\begin{equation*}
\begin{split}
    &\E_{b\sim g^{\theta\dagger}(\cdot\mid s)}\left[Q_L^{\theta\dag}(s,b)\nabla_{\theta} A_F^{{\theta} \dagger}(s, b)\right]\nonumber\\
    &=\E_{b}\left[Q_L^{\theta\dag}(s,b)\cdot\left(\nabla_{\theta} Q_F^{{\theta} \dagger}(s, b) - \E_{b}\left[\nabla_{\theta} Q_F^{{\theta} \dagger}(s, b)\right]\right)\right]\\
    &=\mathrm{Cov}_{b}\left[Q_L^{\theta\dag}(s,b), \nabla_{\theta}Q_F^{\theta\dag}(s,b)~\middle|~ s\right]\\
    &=\E_{b}\left[\left(Q_L^{\theta\dag}(s,b) - \E_{b}\left[Q_L^{{\theta} \dagger}(s, b)\right]\right)\cdot \nabla_{\theta} Q_F^{{\theta} \dagger}(s, b)\right]\\
    &=\E_{b\sim g^{\theta\dagger}(\cdot\mid s)}\left[B_L^{\theta\dag}(s,b)\nabla_{\theta} Q_F^{{\theta} \dagger}(s, b)\right],
\end{split}    
\end{equation*}
where $\mathrm{Cov}_{b}[\cdot,\cdot\mid s]$ is a covariance with respect to $b\sim g^{\theta\dagger}(\cdot\mid s)$.
Because this transformation is an algebraic identity, replacing one form with the other preserves the expectation and therefore, preserves unbiasedness of the corresponding Monte Carlo estimator.
We refer to this manipulation as the Boltzmann Covariance trick\footnote{This name is inspired by the Boltzmann Covariance Theorem~\citep{Movellan1998-xi}, which relies on a similar covariance transformation principle.
While our hypergradient can be derived directly using this theorem, we instead present a derivation based on its underlying covariance transformation principle, which we call the Boltzmann covariance trick in this paper.}.
Additionally, it leads to the interpretation of the third term in \eqref{eq:hg_configurable} as enhancing the probability of the follower's actions that are favorable to the leader.

\subsection{Hypergradient for 2-Player Markov Game}

We extend Theorem~\ref{thm:hg-configurable} to 2-player MGs. 
Analogous to the policy gradient theorem for single-agent MDPs, we introduce the discounted state visitation distribution $d_{\gamma}^{\theta g}(s)$ (defined in Appendix~\ref{apdx:discounted_cumulative_state_visitation}).
For conciseness, we denote $d_\gamma^{\theta g^{\theta\dagger}}$ as $d_\gamma^{\theta \dag}$. 
The policy gradient theorem for single-agent MDPs~\citep{Sutton1999-su,Sutton2018-jf} can be adapted to the leader's objective under a fixed follower as:
\begin{equation*}
    \nabla_{\theta} J_L({\theta}, g) = \frac{1}{1 - \gamma_L} \E\left[ Q_L^{\theta g}(s, a, b) \nabla_{\theta} \log f_{\theta}(a\mid s)\right],
\end{equation*}
where the expectation $\E$ is taken over $(s,a,b) \sim g(b\mid s,a)f_{{\theta}}(a\mid s) d_{\gamma_L}^{\theta g}(s)$.

The following theorem extends the policy gradient theorem to 2-player MGs, providing the hypergradient for a stochastic leader policy. Note that, in general, the optimal leader policy in 2-player MGs may be stochastic.
In this context, the \emph{Benefit} of the leader $f_{{\theta}}$, given that the follower takes action $b$ at state $s$ under the leader's action $a$, is defined as:
\begin{equation*}
    B_L^{{\theta}\dagger}(s, a, b) := Q_L^{{\theta}\dagger}(s,a,b)-\E_{b\sim g^{{\theta}\dagger}(\cdot\mid s,a)}\left[Q_L^{{\theta}\dagger}(s,a,b)\right].
\end{equation*}

\begin{theorem}[Hypergradient for MGs]\label{thm:hg-mg-stochastic}
The hypergradient of \eqref{eq:leader-utility-MG} with respect to ${\theta}$ is given by
\begin{multline}
    \nabla_{{\theta}}J_L({\theta}, g^{\theta\dagger})
    =\dfrac{1}{1-\gamma_L}\E\bigg[Q_L^{{\theta}\dagger}(s,a,b)\nabla_{{\theta}}\log f_{{\theta}}(a\mid s)
   \\+\dfrac{B_L^{{\theta}\dagger}(s,a,b)}{\beta}\nabla_{\theta}Q_F^{\theta\dagger}(s,a,b)\bigg],\label{eq:hg_MG}
\end{multline}
where the expectation $\E$ is taken for $(s,a,b) \sim g^{{\theta}\dagger}(b\mid s,a)f_{{\theta}}(a\mid s) d_{\gamma_L}^{\theta\dagger}(s)$, and
\begin{align}
    &\nabla_{\theta}Q_F^{\theta\dagger}(s,a,b) \label{eq:nablaQF_MG}\\ &=\E_{\tau}\Bigg[\sum_{t=0}^\infty \gamma_F^t V_F^{{\theta}\dag}(s_t,a_t)\nabla_{{\theta}}\log f_{{\theta}}(a_t\mid s_t)\Bigg|\begin{array}{l}s_0=s\\a_0=a\\b_0=b\end{array}\Bigg].\notag
\end{align}
\end{theorem}


The first term of \eqref{eq:hg_MG} corresponds to the policy gradient theorem~\citep{Sutton1999-su,Sutton2018-jf}, which reflects the improvement of the leader's utility by changing its own policy.
The second term reflects the improvement in the leader's utility resulting from the change in the follower's best response due to the leader's policy update. 

\subsection{Implementation}

We propose Actor-Critic algorithms for both configurable MDP and MG settings, termed \emph{Boltzmann Covariance HyperGradient (BC-HG)}. 
These algorithms comprise standard critic updates to estimate the current leader's Q-function using trajectories sampled under the best response, and actor updates utilizing a hypergradient estimated via the Benefit, based on \eqref{eq:hg_configurable} and \eqref{eq:hg_MG}. 
While our algorithm shares similarities with that of \citet{Thoma2024-em}, it distinguishes itself by leveraging the Benefit. 
As discussed previously, this distinction is crucial for practical tractability. 
Algorithm~\ref{algo:proposal} provides an overview of the proposed method.

We proceed under the following assumption, which allows us to focus on hypergradient estimation without the computational burden of estimating the follower’s policy and value function:
\begin{assumption}[White-box follower]\label{thm:assume_wb_follower}
    The follower’s policy and the value functions resulting from its learning process are accessible to the leader.
\end{assumption}
Although our overall framework is designed for a decentralized learning setting, where the leader cannot directly intervene in the follower's updates, this constraint does not necessarily imply a black-box follower setting. 
Assumption~\ref{thm:assume_wb_follower} is often justifiable when the follower's best response is determined analytically with a task-specific method; accessing the best response is natural, but we do not want to change the follower's optimization process. 
This setup holds regardless of whether the follower is centralized or decentralized. 
This scenario is specifically addressed in our experiments (Sections~\ref{sec:experiment_CMDP_continuous} and \ref{sec:experiment_MG_continuous}). 
Furthermore, even if the follower is not strictly white-box, its policy and values can often be accurately estimated from interaction samples, mitigating the practical difference in many real-world applications.

We decompose the hypergradient into two components: the partial derivative $\partial_{\theta}J_L$ and the guiding term $\beta^{-1}\E[B_L^{\theta\dagger}\nabla_{\theta}Q_F^{\theta\dagger}]$.
The partial derivative $\partial_{\theta}J_L$ corresponds to the first two terms of the hypergradient in configurable MDPs \eqref{eq:hg_configurable}, and to the first term in MGs \eqref{eq:hg_MG}.
The leader's value function $V_L^{\theta\dagger}$\footnote{
In Line~\ref{algo:line_L_value} of Algorithm~\ref{algo:proposal}, we slightly abuse notation by defining $V_L(\bar{s})$ as the discounted cumulative expected leader reward conditioned on the initial state-action $\bar{s}=(s,a)$ in MG settings.
} can be computed via the Bellman equation using the estimated critic $Q_L^{\theta\dagger}$ and the best response $g^{\theta\dagger}$.
The guiding term, corresponding to the final term of each hypergradient, is estimated by averaging the gradient of the follower's Q-function, $\nabla_{\theta}Q_F^{\theta\dagger}(s,b)$ (or $\nabla_{\theta}Q_F^{\theta\dagger}(s,a,b)$), weighted by the Benefit $B_L^{\theta\dagger}(s,b)$ (or $B_L^{\theta\dagger}(s,a,b)$), over the sampled transitions. 
The gradient $\nabla_{\theta}Q_F^{\theta\dagger}$, as shown in \eqref{eq:nablaQF} (or \eqref{eq:nablaQF_MG}), is estimated as the discounted cumulative gradient along a trajectory segment following each transition $(s,b)$ (or $(s,a,b)$) sampled for the outer expectation in \eqref{eq:hg_configurable} (or \eqref{eq:hg_MG}).

\begin{algorithm}[t]
\caption{The BC-HG Algorithm}
\label{algo:proposal}
\begin{algorithmic}[1]\small
\STATE{Initialize leader parameters $\theta^0$ and critic $Q_L^0$}
\FOR{$i=0$ to $N-1$}
\STATE{(Follower computes best response $g^{\theta^i\dagger}$ for $\mathcal{M}_{\theta^i}$)}
\STATE{Sample a batch of trajectories $B\gets\{\tau_j\}$ of length $T$}
\STATE{Get follower's best response $g^{\theta^i\dagger}$ and value function $V_F^{\theta^i\dagger}$}
\STATE{$Q_L^{i+1}\gets \texttt{CriticUpdate}(Q_L^{i},B)$ (Algorithm~\ref{algo:critic_update_sarsa},\ref{algo:critic_update_td})}
\STATE\COMMENT{$\bar{s}$ denotes state $s$ for configurable MDPs}
\STATE\COMMENT{$\bar{s}$ denotes state-action pair $(s,a)$ for MGs}
\FOR{$(\bar{s},b)\in\tau\in B$}
\STATE{$V_L(\bar{s})\gets\E_{b\sim g^{\theta^i\dagger}(\cdot\mid \bar{s})}\left[Q_L^{i+1}(\bar{s},b)\right]$}\label{algo:line_L_value}
\STATE{$B_L(\bar{s},b)\gets Q_L^{i+1}(\bar{s},b)-V_L(\bar{s})$}
\STATE{$B_{\bar{s}b}\gets\left\{\tau^{k:T}\mid (\bar{s}_k, b_k) = (\bar{s},b) \text{ for } (\bar{s}_k,b_k)\in\tau\in B\right\}$}%
\STATE{$\widehat{\nabla_{\theta}Q_F}(\bar{s},b)\gets\texttt{FollowerQGrad}(B_{\bar{s}b},\theta^i,V_F^i)$ (Algorithm~\ref{algo:estimate_f_q_grad_cmdp},\ref{algo:estimate_f_q_grad_mg})}
\ENDFOR
\STATE{$\widehat{\partial_{\theta}J_L}\gets\texttt{PartialDerivative}(B,\theta^i,V_L,Q_L^{i+1})$ (Algorithm~\ref{algo:estimate_partial_grad_cmdp},\ref{algo:estimate_partial_grad_mg})}
\STATE{$\widehat{\nabla_{\theta}J_L}\gets \widehat{\partial_{\theta}J_L}$\\ \hfill + $\frac{1}{\beta|B|}\displaystyle\sum_{\tau\in B}\sum_{(\bar{s}_t,b_t)\in\tau}\gamma_L^tB_L(\bar{s}_t,b_t)\widehat{\nabla_{\theta}Q_F}(\bar{s}_t,b_t)$}%
\STATE{$\theta^{i+1}\gets \theta^i+\alpha\widehat{\nabla_{\theta}J_L}$}
\ENDFOR
\end{algorithmic}
\end{algorithm}


\section{Experiments}

We demonstrate the advantages of our proposed BC-HG over baseline methods designed for configurable MDPs and 2-player MGs in both discrete and continuous state settings.

\subsection{Configurable MDPs (Discrete)}\label{sec:experiment_CMDP_discrete}

\paragraph{Four-Rooms Task}
This task is adapted from \citet{Thoma2024-em}.
The follower can move up, down, right, or left (4 actions) within the level (104 states) visualized in Figure~\ref{fig:four_rooms_env}. The follower starts at the cell denoted $S$ and receives a positive reward upon reaching the goal cell $G$. 
In contrast, the leader receives an immediate reward of $+1$ when the follower visits the target cell, which is denoted by ``1'' and highlighted in green.
To guide the follower, the leader can place a negative incentive (penalty) in each cell, incurring a cost proportional to the total negative incentive placed when the follower reaches the goal. 
See Appendix~\ref{apdx:fourrooms} for details. 

\paragraph{Baselines}
We compare the proposed approach against the following baselines: \textsc{Naive-PGD}, which employs first-order policy gradient updates~\citep{Rajeswaran2020-gv,Vu2022-jq} ignoring changes in the follower's best response (i.e., using $\partial_{\theta} J_L(\theta, g)\mid_{g = g^{\theta\dagger}}$); \textsc{SoBiRL}, a hypergradient approach by \citet{Yang2025-lr}, originally designed for reward shaping from human feedback; and \textsc{HPGD}, another hypergradient approach by \citet{Thoma2024-em}. 
We evaluate three variations of \textsc{HPGD}: \textsc{HPGD (oracle)}, which relies on an oracle for hypergradient estimation (as presented in the original paper); \textsc{HPGD (MC)}, which uses Monte-Carlo estimation with a batch of trajectories collected during interaction (as implemented by the authors); and \textsc{HPGD (SA)}, which utilizes SARSA-type estimation similar to our proposed approach. 
Note that \textsc{HPGD (oracle)} requires oracle access, a condition not assumed by the other methods. 
In this experiment, additional trajectories of size $10^4$ per outer iteration are generated by the oracle.
See Appendix~\ref{apdx:baseline} for details.

\paragraph{Settings}
For all approaches, the follower's best response is computed via soft value iteration with entropy regularization coefficients $\beta\in\{1\times10^{-3}, 3\times10^{-3}, 5\times10^{-3}\}$.
For the main results in Figure~\ref{fig:fourrooms}, we use a batch size of $100$, which is small relative to the state-space size of 104.
For each approach, we perform a grid search over hyperparameters and select the best combination based on average performance across 10 random seeds.
Results with different batch sizes (200, 400, 1000) are reported in Appendix~\ref{apdx:fourrooms}.

\begin{figure}[t]
    \centering%
    \includegraphics[width=0.36\linewidth]{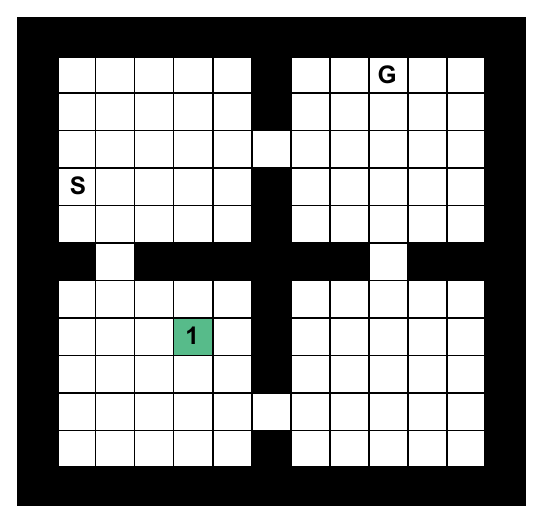}%
    \caption{Four-Rooms Environment}%
    \label{fig:four_rooms_env}%
\end{figure}%

\paragraph{Results}

\begin{figure*}[t]
    \centering
    \begin{subfigure}{0.33\hsize}
        \includegraphics[width=0.95\hsize,clip,trim=5 10 5 5]{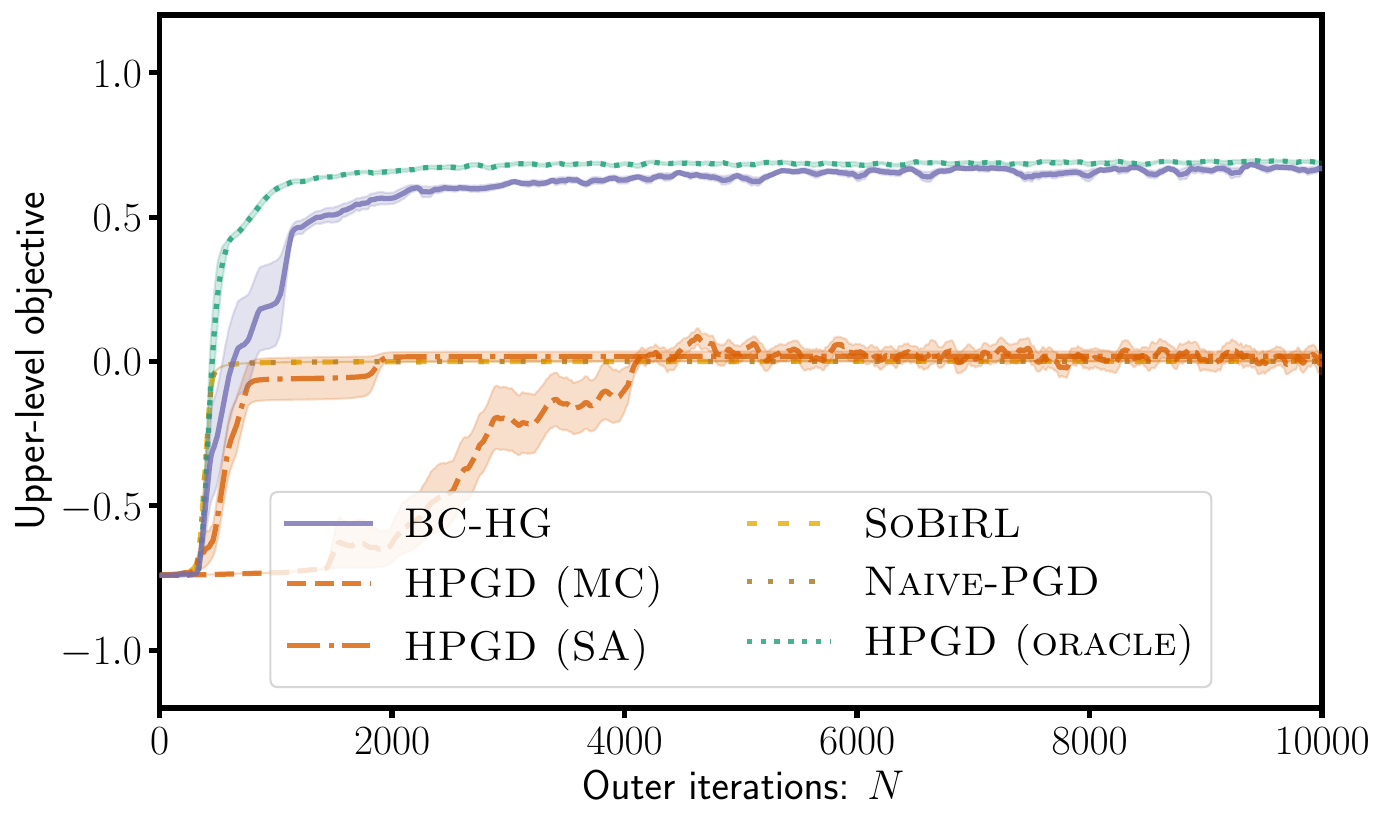}%
        \caption{$\beta=1\times10^{-3}, \text{BatchSize}=100$}%
        \label{fig:four_rooms_UL_rewards_grid_search_reg_lambda_5.0_beta_0_001_steps_100}%
    \end{subfigure}%
    \begin{subfigure}{0.33\hsize}
        \includegraphics[width=0.95\hsize,clip,trim=5 10 5 5]{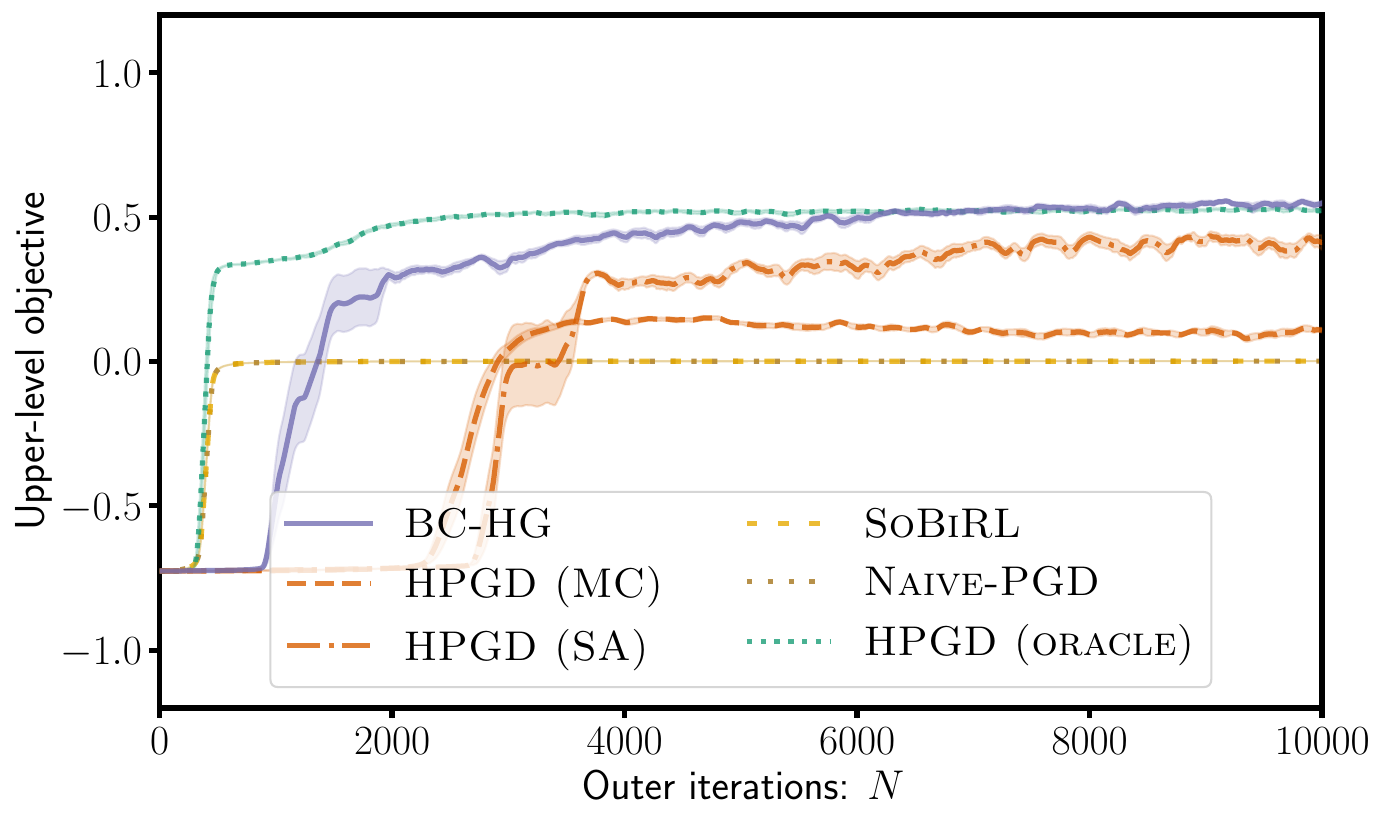}%
        \caption{$\beta=3\times10^{-3}, \text{BatchSize}=100$}%
        \label{fig:four_rooms_UL_rewards_grid_search_reg_lambda_5.0_beta_0_003_steps_100}%
    \end{subfigure}%
    \begin{subfigure}{0.33\hsize}
        \includegraphics[width=0.95\hsize,clip,trim=5 10 5 5]{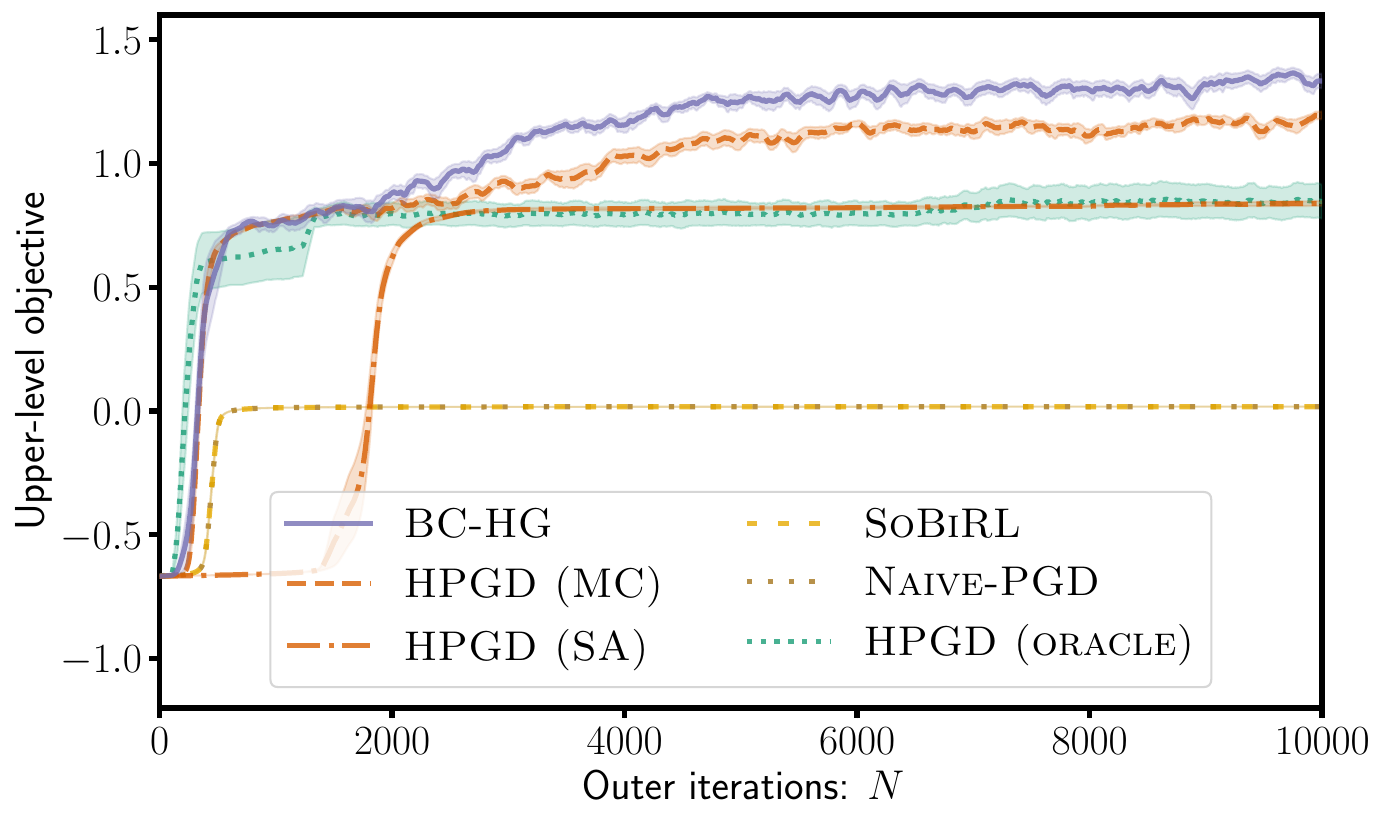}%
        \caption{$\beta=5\times10^{-3}, \text{BatchSize}=100$}%
        \label{fig:four_rooms_UL_rewards_grid_search_reg_lambda_5.0_beta_0_005_steps_100}%
    \end{subfigure}%
    \caption{Results on Four-Rooms Task ($\text{BatchSize}=100$).}\label{fig:fourrooms}
\end{figure*}

The results are presented in Figures~\ref{fig:four_rooms_UL_rewards_grid_search_reg_lambda_5.0_beta_0_005_steps_100}, \ref{fig:four_rooms_UL_rewards_grid_search_reg_lambda_5.0_beta_0_003_steps_100}, and \ref{fig:four_rooms_UL_rewards_grid_search_reg_lambda_5.0_beta_0_001_steps_100} for varying entropy regularization parameters $\beta$.
The proposed approach, \textsc{BC-HG}, consistently outperforms the baselines. 
Generally, the performance of \textsc{HPGD} variants (excluding the oracle version) degrades as entropy regularization weakens (i.e., smaller $\beta$). 
With a smaller $\beta$, selecting different actions at the same state becomes less likely. 
Consequently, the core assumption of \textsc{HPGD} variants, that trajectories starting at the same state with different actions are collected during interaction, becomes harder to satisfy. 
Results with different batch sizes are shown in Appendix~\ref{apdx:fourrooms}.
The superiority of \textsc{BC-HG} over \textsc{HPGD (SA)} underscores the impact of the Boltzmann Covariance trick, as the only difference lies in the guiding term's form in the hypergradients.
\textsc{SoBiRL}, \textsc{Naive-PGD}, and \textsc{HPGD} variants with $\beta = 1 \times 10^{-3}$ become stuck in a local minimum where no negative incentive is placed, resulting in zero penalty and zero positive reward for the leader.
This occurs because the third term of the hypergradient in Theorem~\ref{thm:hg-configurable}, capturing the shift in the best response, is poorly estimated, emphasizing the penalty effect in the first term.

\subsection{Configurable MDPs (Continuous)}\label{sec:experiment_CMDP_continuous}

\paragraph{Task Description and Settings}
Inspired by \citet{AFRAM2014343}, we designed an \emph{$n$-zone building thermal control task} where the follower controls heating, ventilation, and air conditioning (HVAC) units to minimize temperature deviations in $n$ zones using a linear quadratic regulator (LQR). 
The leader aims to enhance temperature stability and energy efficiency by adjusting building parameters, such as insulation and airflow.

This task is formulated as a bi-level extension of infinite-horizon discounted LQR with entropy regularization. 
The state $s\in\R^n$ is temperature deviations, with transitions defined as:
$s_{t+1} = A_{\theta} s_{t} + B b_{t} + w_{t}$, where $w_t\sim \mathcal{N}(0,W)$,
$A_{\theta} \in \R^{n\times n}$, $B \in \R^{n \times m}$, and $W \in \R^{n \times n}$.
$m$ is the number of HVAC units.
The follower's reward function is: $r_F(s, b) := - s^{\top} \bar{Q} s - b^{\top} \bar{R} b$, where $\bar{Q} \in \R^{n\times n}$ is positive semi-definite and $\bar{R} \in \R^{m\times m}$ is positive definite.
The leader's parameter $\theta\in \R^{2n}$ represents the insulation level of $n$ zones and the airflow level of $n$ inter-zone ventilations.

The number of zones and HVAC units is set to $n=4$ and $m=2$, respectively.
The follower's best response is computed via Riccati iteration for all approaches (see Appendix~\ref{apdx:lqr} for details).
We conduct a grid search over hyperparameters for each approach using 10 random seeds.
Further details are provided in Appendix~\ref{apdx:building_thermal_control}.

\paragraph{Baselines}
We consider the following baselines: \textsc{Naive-PGD}, \textsc{HPGD (MC)}, and \textsc{HPGD (TD)}.
In continuous state settings, \citet{Thoma2024-em} implement \textsc{HPGD} by estimating the follower's value function gradient using a function approximator rather than additional trajectories, estimating the gradient as a vector-valued function in the cumulative form of \eqref{eq:nablaQF}.
This implementation does not rely on an oracle.
However, the learning cost of the vector-valued function increases with higher-dimensional leader parameters.
\textsc{HPGD} estimates the leader's Q-function via Monte-Carlo estimation in \textsc{HPGD (MC)} and via the TD-method with neural networks in \textsc{HPGD (TD)}.
Note that \textsc{SoBiRL} is inapplicable here because the leader parameterizes the transitions.

\paragraph{Results}
Figure~\ref{fig:env_3_2_SR_1_0_EC_0_5_IC_0_1_AC_0_1_return_UL_no_outliers} displays the results.
\textsc{BC-HG} achieves the highest convergence in the fewest iterations.
\textsc{Naive-PGD} shows more gradual improvement, suggesting that while partial derivative information is useful in this non-competitive task, it is insufficient without hypergradient information.
\textsc{HPGD} exhibits slower improvement or even degradation with certain hyperparameters.
These results imply that the estimation accuracy of the follower's value gradient was insufficient for the 8-dimensional leader parameter.
Given that the learning cost of the value gradient increases with the dimensionality of the leader parameter, \textsc{BC-HG} is more scalable. 
See Appendix~\ref{apdx:building_thermal_control} for more results.

\begin{figure}[t]
    \centering
    \includegraphics[width=0.7\hsize,clip,trim=5 5 5 0]{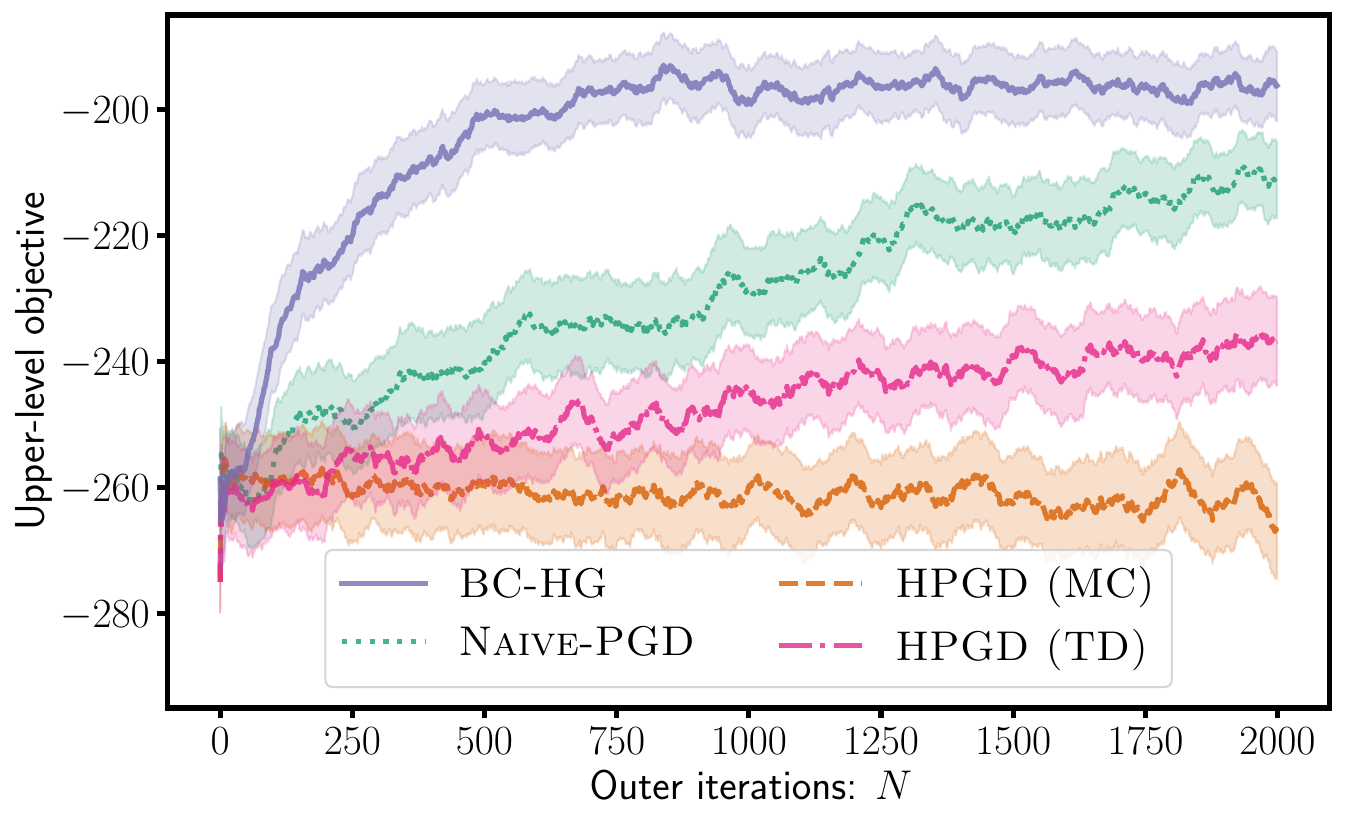}%
    \caption{Results on Building Thermal Control.}%
    \label{fig:env_3_2_SR_1_0_EC_0_5_IC_0_1_AC_0_1_return_UL_no_outliers}%
\end{figure}

\subsection{2-Player Markov Games (Discrete)}\label{sec:experiment_MG_discrete}

\paragraph{Task Description and Setting}
This task involves finite state and action spaces with deterministic transitions.
The state space is $\mathcal{S} = \{S, A, B\}$, and action spaces are $\mathcal{A} = \{0, 1\}$ (leader) and $\mathcal{B} = \{s, a, b\}$ (follower).
Deterministic transition dynamics and reward functions are illustrated in Figure~\ref{fig:DiscreteToy4_1b-v1}.
Leader actions at states $S$ and $B$ do not influence rewards or transitions; thus, we focus on the leader's action selection probability at state $A$.
If $f_{\theta}(0\mid A) \approx 1$, the follower prefers actions $a$, $b$, and $s$ at states $S$, $A$, and $B$, respectively. This results in the cycle $S \xrightarrow{*|a} A \xrightarrow{0|b} B \xrightarrow{*|s} S$, yielding rewards of $1, 2, 0$ for the follower and $1, 0, 0$ for the leader.
Conversely, if $f_{\theta}(1\mid A) \approx 1$, the follower prefers actions $b$, $b$, and $s$ at states $S$, $A$, and $B$, respectively. This leads to the cycle $S \xrightarrow{*|b} B \xrightarrow{*|s} S$, yielding rewards of $1, 0$ for the follower and $0, 0$ for the leader.
Since the former scenario yields a higher return for the follower, the follower prefers action $a$ at $S$ even when $f_{\theta}(0 \mid A) = p \in (0, 1)$, despite the risk of a negative reward $-1$ with probability $1-p$.
From the leader's perspective, the $S \leftrightarrow A$ cycle yields the highest cumulative reward.
Therefore, the optimal leader strategy is to minimize $p$ while maintaining a high probability of the follower selecting action $a$ at $S$. 

For all approaches, the follower's best response is computed via soft Q-iteration with entropy regularization coefficient $\beta=5\times10^{-2}$.
We conduct a grid search over hyperparameters (number of critic/actor updates per outer iteration, on-policy vs. off-policy sampling) and select the best combination based on average performance across 10 random seeds.
All learning rates are fixed.
Performance is evaluated by averaging the cumulative leader reward across 10 rollouts.
See Appendix~\ref{apdx:toymarkovgame} for further details.

\paragraph{Baselines}


We consider two baselines.
The first is \textsc{Naive-PGD}, and the second is a variant of Bi-AC~\cite{Zhang2020-hw}, originally proposed for centralized learning, adapted to our decentralized setting.
The core idea of Bi-AC is to update the leader and follower Q-tables, $Q_L$ and $Q_F$, as follows:
\begin{multline*}
Q_i'(s, a, b) \gets (1 - \alpha)  Q_i(s, a, b)\\+ \alpha (r_i(s, a, b) + \gamma_i Q_i(s', a^*, b^*(a^*)))~(i\in\{L,F\})
\end{multline*}
where
$b^(a') := \argmax_{b'} Q_F(s', a', b')$,
and
$a^ := \argmax_{a'} Q_L(s', a', b^*(a'))$.
That is, for each $s'$, the Q-tables are updated using the Stackelberg equilibrium of the bimatrix game defined by the pair of Q-value matrices
$(Q_L(s', \cdot, \cdot), Q_F(s', \cdot, \cdot))$.
In the modified version of \textsc{Bi-AC} used in this study, we replace the critic update rule with the above update and substitute $Q_F$ with the follower's optimal Q-function $Q_F^{\theta\dag}$.
The actor is updated using only the direct gradient term, in the same manner as \textsc{Naive-PGD}.

\paragraph{Results}
\begin{figure}[t]
    \centering
    \begin{subfigure}{0.9\hsize}%
        \centering
        \includegraphics[width=\hsize,clip,trim=10 10 10 10]{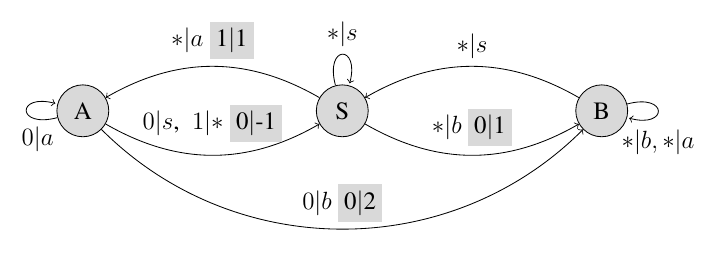}%
        \caption{State Transition Diagram. Nodes represent states and edges represent transitions. Labels on edges indicate 
        $\mathcal{A}\mid\mathcal{B}$
        and 
        \colorbox{gray!20}{$r_L \mid r_F$}.         
        The absence of shaded labels implies zero rewards. ``$*$'' represents a wildcard.}%
        \label{fig:DiscreteToy4_1b-v1}%
    \end{subfigure}%
    \\%
    \begin{subfigure}{\hsize}%
        \centering%
        \includegraphics[width=0.7\hsize,clip,trim=5 5 5 0]{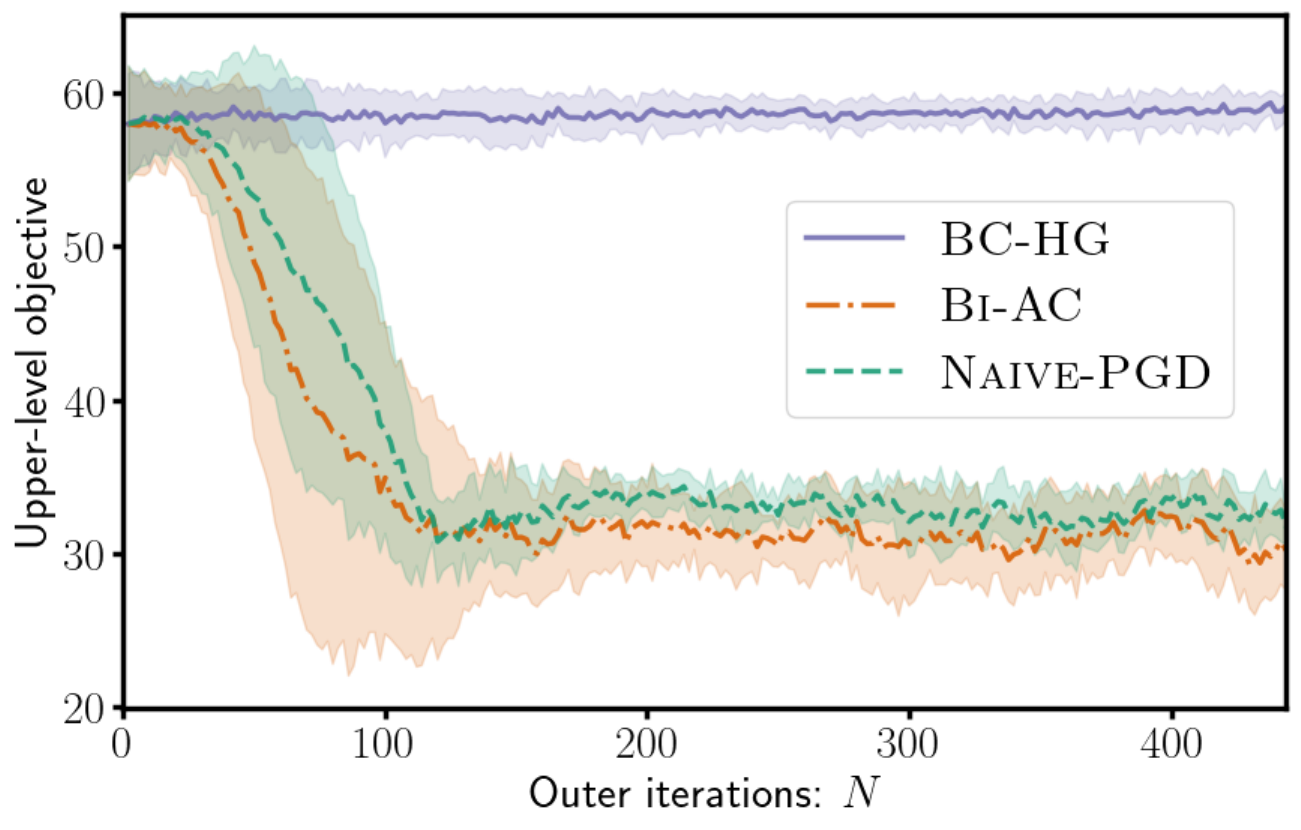}%
        \caption{Result}%
        \label{fig:dicrete_toy_Evaluation_AverageTargetReturn_AveStd}%
    \end{subfigure}%
    \caption{Task Description and Results on ToyMarkovGame.}%
    \label{fig:toymarkovgame}
\end{figure}

Figure~\ref{fig:dicrete_toy_Evaluation_AverageTargetReturn_AveStd} shows the results. 
\textsc{BC-HG} clearly outperforms the baselines. 
\textsc{BC-HG} maintains $f_{\theta}(0 \mid A) \approx 0.53$, with the corresponding follower best response $g^{\theta\dagger}(a \mid S, *) \approx 1$. Consequently, the $S \to A$ loop and $S \to A \to B$ loop occur with probabilities around $0.47$ and $0.53$, respectively.
In contrast, baseline approaches maintain $f_{\theta}(0 \mid A) \approx 0.4$ and $g^{\theta\dagger}(\cdot \mid S, *) \approx 0, 0.47, 0.53$ for $s, a, b$, respectively. As $f_{\theta}(0 \mid A)$ is too small, the $S \leftrightarrow B$ loop occurs, resulting in a loss for the leader.
This trend persists across a wide range of hyperparameters, as discussed in Appendix~\ref{apdx:toymarkovgame}. The policy gradient with Bi-AC-type Q-update exhibits a similar trend.

\subsection{2-Player Markov Games (Continuous)}\label{sec:experiment_MG_continuous}

\paragraph{Bi-Level LQR in 2-Player Markov Games}
We consider the extension of LQR tasks to MGs.
Introducing the leader's action, the state transition is defined as:
$s_{t+1} = A s_{t} + B b_{t} + C a_{t}$,
where $A \in \R^{n\times n}$, $B \in \R^{n \times m}$, and $C \in \R^{n \times k}$ are constant. The follower's reward function is $r_F(s, a, b) = - s^{\top} \bar{Q} s - b^{\top} \bar{R} b$. 
The leader's policy $f_{\theta}(a \mid s)$ is modeled by a Gaussian distribution $\mathcal{N}(K_{\theta} s, W)$, where $K_{{\theta}} \in \R^{k \times n}$ is the leader's policy parameter and $W \in \R^{k \times k}$ is the fixed covariance matrix.
In this experiment, the leader's reward function is:
\begin{equation*}
    r_L(s, a, b) := \exp\Big( -\frac{1}{2} (s - s^\star)^\top \Sigma (s - s^\star) \Big),
\end{equation*}
where $\Sigma \in \R^{n \times n}$ defines the reward sharpness. The leader aims to guide the follower to visit $s^\star \in \mathcal{S}$.  
The follower's best response is derived by solving the Riccati equation.
Further details are provided in Appendix~\ref{apdx:bilevel_lqr_task_mg}.

\paragraph{Results}
Figure~\ref{fig:lqr} presents the results for the bi-level LQR task.
\textsc{BC-HG} outperforms the baseline. 
A possible reason for the high variance in the upper-level performance of the proposed approach is that leader policies achieving higher cumulative rewards tend to exhibit higher variance in cumulative reward, as demonstrated in Figure~\ref{fig:lqr_hist} (Appendix~\ref{apdx:bilevel_lqr_task_mg}). 

\begin{figure}[t]
    \centering
    \includegraphics[width=0.7\hsize,clip,trim=10 5 10 5]{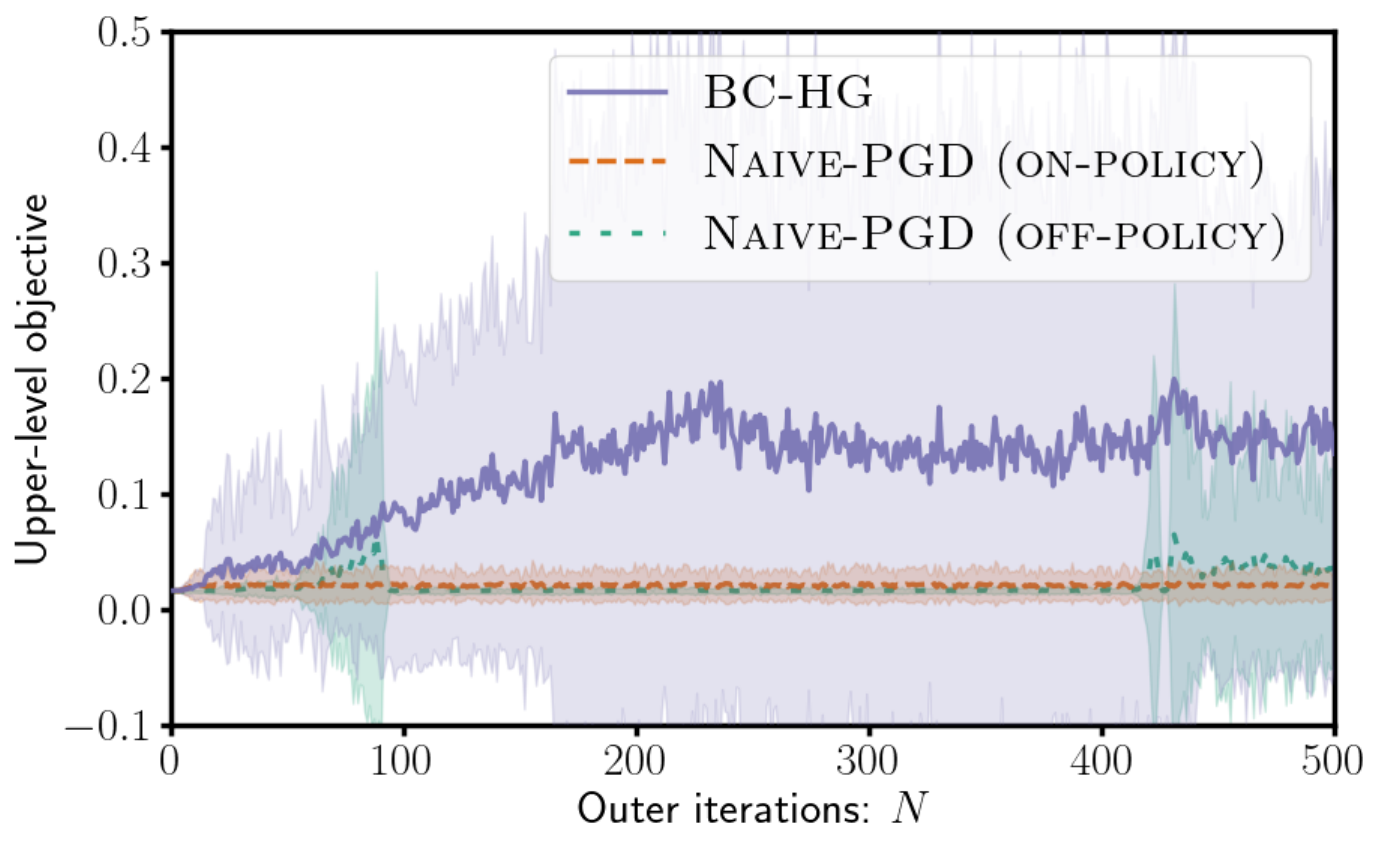}%
    \caption{Results on the Bi-Level LQR Task.}%
    \label{fig:lqr}%
\end{figure}



\section{Conclusion}

We developed hypergradient-based bi-level RL algorithms for configurable MDPs and MGs. 
By leveraging the Boltzmann covariance trick, we circumvented the need for oracle-based trajectory generation while simultaneously enabling efficient and unbiased hypergradient estimation. 
In numerical experiments, our approach demonstrated scalability to large state spaces and high-dimensional parameterizations in configurable MDPs. We also validated the effectiveness of our method in Markov game settings with both discrete and continuous state spaces. 
Future work includes addressing more realistic decentralized follower scenarios where the follower is either suboptimal or black-box.





\appendix

\onecolumn

\section{Algorithms}

\subsection{Modules within Proposed Algorithm~\ref{algo:proposal}}

Note that $\bar{s}$ denotes state $s$ in configurable MDP settings and denotes a pair of state $s$ and leader action $a$ in 2-player MGs.

\begin{algorithm}[H]
\caption{$\texttt{CriticUpdate}(Q_L^i,B)$ (with SARSA)}
\label{algo:critic_update_sarsa}
\begin{algorithmic}[1]
\REQUIRE{Current critic $Q_L^i$, sample batch $B$, learning rate $\alpha_{SA}$}
\FOR{$(\bar{s},b,r_L,s',b')\in B$}
\STATE{$Q_L^i(\bar{s},b)\gets Q_L^i(\bar{s},b)+\alpha_{SA}\left(r_L+\gamma_LQ_L^i(\bar{s}',b')-Q_L^i(\bar{s},b)\right)$}
\ENDFOR
\STATE{$Q_L^{i+1}\gets Q_L^i$}
\ENSURE{$Q_L^{i+1}$}
\end{algorithmic}
\end{algorithm}

\begin{algorithm}[H]
\caption{$\texttt{CriticUpdate}(Q_{\omega_L^{i}},B)$ (with TD Method)}
\label{algo:critic_update_td}
\begin{algorithmic}[1]
\REQUIRE{Current critic $Q_{\omega_L^{i}}$, sample batch $B$, learning rate $\alpha_{TD}$}
\FOR{$(\bar{s}_j,b_j,r_L^{(j)},\bar{s}_j',b_j')\in B$}
\STATE{Compute target Q-value $y_j\gets r_L^{(j)}+\gamma_LQ_{\omega_L^i}(\bar{s}_j',b_j')$}
\ENDFOR
\STATE{$\omega_L^{i+1} \gets \omega_L^{i}-\alpha_{TD}\nabla_{\omega_L}\dfrac{1}{|B|}\sum_{(\bar{s}_j,b_j)\in B}\left(y_j-Q_{\omega_L^i}(\bar{s}_j,b_j)\right)^2$}
\ENSURE{$Q_{\omega_L^{i+1}}$}
\end{algorithmic}
\end{algorithm}


\begin{algorithm}[H]
\caption{$\texttt{PartialDerivative}(B,\theta^i,V_L,\cdot)$ in Configurable MDP Settings}
\label{algo:estimate_partial_grad_cmdp}
\begin{algorithmic}[1]\small
\REQUIRE{Trajectory batch $B$, current leader parameter $\theta^i$, leader value function $V_L$}
\STATE{$\widehat{\partial_{\theta}J_L}\gets\frac{1}{|B|}\sum_{\tau\in B}\sum_{(s_t,b_t)\in\tau}\gamma_L^{t}\left(\nabla_{\theta}r_L^{\theta^i}(s_t,b_t)+V_L(s_t) \nabla_{\theta}\log p^{\theta^i} (s_{t} \mid s_{t-1}, b_{t-1})\right)$}
\ENSURE{$\widehat{\partial_{\theta}J_L}$}
\end{algorithmic}
\end{algorithm}

\begin{algorithm}[H]
\caption{$\texttt{PartialDerivative}(B,\theta^i,\cdot,Q_L)$ in Markov Game Settings}
\label{algo:estimate_partial_grad_mg}
\begin{algorithmic}[1]\small
\REQUIRE{Trajectory batch $B$, current leader parameter $\theta^i$, leader Q-function $Q_L$}
\STATE{$\widehat{\partial_{\theta}J_L}\gets\dfrac{1}{|B|}\sum_{(s,a,b)\in B}Q_L(s,a,b)\nabla_{\theta}\log f_{\theta^i}(a\mid s)$}\label{algo:partial_deriv_MG}
\ENSURE{$\widehat{\partial_{\theta}J_L}$}
\end{algorithmic}
\end{algorithm}

\begin{algorithm}[H]
\caption{$\texttt{FollowerQGrad}(B_{sb},\theta^i,V_F)$ in Configurable MDP Settings}
\label{algo:estimate_f_q_grad_cmdp}
\begin{algorithmic}[1]\small
\REQUIRE{Trajectory segments $B_{sb}$, current leader parameter $\theta^i$, follower value $V_F$}
\STATE{$\widehat{\nabla_{\theta}Q_F}(s,b)\gets\frac{1}{|B_{sb}|}\sum_{\tau^{k:T}\in B_{sb}}\sum_{(s_t,b_t)\in\tau^{k:T}}\gamma_F^{t-k}\nabla_{\theta}r_F^{\theta^i}(s_t,b_t)+ \gamma_F^{t-k+1}V_F(s_{t+1})\nabla_{{\theta}}\log p^{{\theta^i}}(s_{t+1}\mid s_t,b_t)$}
\ENSURE{$\widehat{\nabla_{\theta}Q_F}(s,b)$}
\end{algorithmic}
\end{algorithm}

\begin{algorithm}[H]
\caption{$\texttt{FollowerQGrad}(B_{sb},\theta^i,V_F)$ in Markov Game Settings}
\label{algo:estimate_f_q_grad_mg}
\begin{algorithmic}[1]\small
\REQUIRE{Trajectory segments $B_{sb}$, current leader parameter $\theta^i$, follower value $V_F$}
\STATE{$\widehat{\nabla_{\theta}Q_F}(s,a,b)\gets\dfrac{1}{|B_{sab}|}\sum_{\tau^{k:T}\in B_{sab}}\sum_{(s_t,a_t,b_t)\in\tau^{k:T}}\gamma_F^{t-k} V_F(s_t,a_t)\nabla_{\theta}\log f_{\theta^i}(a_t\mid s_t)$}
\ENSURE{$\widehat{\nabla_{\theta}Q_F}(s,a,b)$}
\end{algorithmic}
\end{algorithm}

\subsection{Algorithms for Four-Rooms Task Experiments}

\begin{algorithm}[H]
\caption{Algorithms for Four-Rooms Task Experiments}
\label{algo:fourrooms}
\begin{algorithmic}[1]
\STATE{Initialize the leader's parameter $\theta^0$ randomly}
\STATE{\textcolor{red}{(\textsc{HPGD (SA)}} / \textcolor{blue}{\textsc{BC-HG})} Initialize the leader's critic $Q_L(s,b)$ randomly for all $(s,b)\in\mathcal{S}\times\mathcal{B}$}
\FOR{$i=0$ to $N-1$}
\STATE{Compute the follower's best response $g^{\theta^i\dagger}$ for $\mathcal{M}_{\theta^i}$ via soft-value iteration}
\STATE{Sample a batch of trajectories $B\gets\{\tau_j:=(s_t,b_t)_{t=0}^{T-1}\sim\texttt{SampleTrajectory}(g^{\theta^i\dagger},\texttt{initial\_state}=\texttt{S},\texttt{episode\_length}=T)\}_{j=1}^M$}
\IF{\texttt{use\_oracle}}
\STATE{$\bar{B}\gets\{\bar{\tau}_j:=(s_t,b_t)_{t=0}^{T-1}\sim\texttt{SampleTrajectory}(g^{\theta^i\dagger},\texttt{initial\_state}=\mathrm{Uniform},\texttt{episode\_length}=T)\}_{j=1}^{\bar{M}}$}\label{algo:fourroom_oracle_rollouts}
\ELSE
\STATE{$\bar{B}\gets B$}
\ENDIF
\STATE\COMMENT{Estimate the upper-level partial derivative}
\FOR{$s$ in $\mathcal{S}$}
\STATE{Extract trajectory segments $B_{s}\gets\left\{\tau_j^{k:T}\mid\exists j,\exists k,s=s_k\in\tau_j\in B,\tau_j\succeq_{\mathrm{suffix}}\tau_j^{k:T}\right\}$}
\STATE{$\widehat{\partial_{\theta}V_L}(s)\gets\frac{1}{|B_{s}|}\sum_{\tau^{k:T}\in B_{s}}\sum_{(s_t,b_t)\in\tau^{k:T}}\gamma_L^{t-k}\nabla_{\theta}r_L^{\theta^i}(s_t,b_t)$}
\ENDFOR
\STATE{$\widehat{\partial_{\theta}J_L}\gets \frac{1}
{|B|}\sum_{\tau\in B}\widehat{\partial_{\theta}V_L}(s_0)$}\label{algo:partial_deriv}
\STATE\COMMENT{Estimate the lower-level advantage gradient}
\FOR{$(s,a)$ in $\mathcal{S}\times\mathcal{A}$}
\STATE{Extract trajectory segments $\bar{B}_{sb}\gets\left\{\bar{\tau}_j^{k:T}\mid\exists j,\exists k,(s,b)=(s_k, b_k)\in\bar{\tau}_j\in\bar{B},\bar{\tau}_j\succeq_{\mathrm{suffix}}\bar{\tau}_j^{k:T}\right\}$}%
\STATE{$\widehat{\partial_{\theta}Q_F}(s,b)\gets\frac{1}{|\bar{B}_{sb}|}\sum_{\bar{\tau}^{k:T}\in \bar{B}_{sb}}\sum_{(s_t,b_t)\in\bar{\tau}^{k:T}}\gamma_F^{t-k}\nabla_{\theta}r_F^{\theta^i}(s_t,b_t)$}\label{algo:fourrooms_f_q_grad}
\STATE{\textcolor{red}{(HPGD) Extract trajectory segments $\bar{B}_{s}\gets\left\{\bar{\tau}_j^{k:T}\mid\exists j,\exists k, s=s_k\in\bar{\tau}_j\in\bar{B},\bar{\tau}_j\succeq_{\mathrm{suffix}}\bar{\tau}_j^{k:T}\right\}$}}
\STATE{\textcolor{red}{(HPGD) $\widehat{\partial_{\theta}V_F}(s)\gets\frac{1}{|\bar{B}_{s}|}\sum_{\bar{\tau}^{k:T}\in \bar{B}_{s}}\sum_{(s_t,b_t)\in\bar{\tau}^{k:T}}\gamma_F^{t-k}\nabla_{\theta}r_F^{\theta^i}(s_t,b_t)$}}\label{algo:fourrooms_f_v_grad}
\ENDFOR
\STATE\COMMENT{Estimate the Benefit}
\STATE{\textcolor{red}{(\textsc{HPGD (MC)}} / \textcolor{magenta}{\textsc{SoBiRL})} $Q_L,V_L\gets \texttt{EstimateValueByMonteCarlo}(\bar{B}_{sb},\bar{B}_{s},\theta^i)$ (Algorithm~\ref{algo:benefit_estimation_montecarlo})}
\STATE{\textcolor{red}{(\textsc{HPGD (SA)}} / \textcolor{blue}{\textsc{BC-HG})} $Q_L,V_L \gets \texttt{EstimateValueBySARSA}(B,Q_L,\theta^i,g^{\theta^i\dagger})$ (Algorithm~\ref{algo:benefit_estimation_sarsa})}
\STATE{$B_L(s,b)\gets Q_L(s,b)-V_L(s)$ for all $(s,b)\in\tau\in B$}
\STATE\COMMENT{Estimate the hypergradient}
\STATE{\textcolor{brown}{(\textsc{Naive-PGD}) $\widehat{\nabla_{\theta}J_L}\gets\widehat{\partial_{\theta}J_L}$}}
\STATE{\textcolor{blue}{(\textsc{BC-HG}) $\widehat{\nabla_{\theta}J_L}\gets\widehat{\partial_{\theta}J_L}+\frac{1}{\beta}\frac{1}{|B|}\sum_{\tau\in B}\sum_{(s_t,b_t)\in\tau}\gamma_L^tB_1(s_t,b_t)\widehat{\partial_{\theta}Q_F}(s_t,b_t)$}}
\STATE{\textcolor{red}{(\textsc{HPGD}) $\widehat{\nabla_{\theta}J_L}\gets\widehat{\partial_{\theta}J_L}+\frac{1}{\beta}\frac{1}{|B|}\sum_{\tau\in B}\sum_{(s_t,b_t)\in\tau}\gamma_L^tB_1(s_t,b_t)\left(\widehat{\partial_{\theta}Q_F}(s,a)-\widehat{\partial_{\theta}V_F}(s)\right)$}}
\STATE{\textcolor{magenta}{(\textsc{SoBiRL}) $\widehat{\nabla_{\theta}J_L}\gets\widehat{\partial_{\theta}J_L}+\frac{1}{\beta}\frac{1}{|B|}\sum_{\tau\in B}\sum_{(s_t,b_t)\in\tau}Q_L(s_0,b_0)\left(\nabla_{\theta}r_F^{\theta^i}(s_t,b_t)-\E_{b\sim g^{\theta\dagger}(\cdot\mid s_t)}\left[\nabla_{\theta}r_F^{\theta^i}(s_t,b)\right]\right)$}}
\STATE{$\theta^{i+1}\gets \theta^i+\alpha\widehat{\nabla_{\theta}J_L}$}
\ENDFOR
\end{algorithmic}
\end{algorithm}

\begin{algorithm}[H]
\caption{$\texttt{EstimateValueByMonteCarlo}(\bar{B}_{sb},\bar{B}_{s},\theta^i)$}
\label{algo:benefit_estimation_montecarlo}
\begin{algorithmic}[1]
\REQUIRE{trajectory segments $\bar{B}_{sb}$ and $\bar{B}_{s}$, current leader's parameter $\theta^i$}
\FOR{$(s,a)$ in $\mathcal{S}\times\mathcal{A}$}
\STATE{$Q_L(s,b)\gets\frac{1}{|\bar{B}_{sb}|}\sum_{\bar{\tau}^{k:T}\in \bar{B}_{sb}}\sum_{(s_t,b_t)\in\bar{\tau}^{k:T}}\gamma_L^{t-k} r_L^{\theta^i}(s_t,b_t)$}
\STATE{$V_L(s)\gets\frac{1}{|\bar{B}_{s}|}\sum_{\bar{\tau}^{k:T}\in \bar{B}_{s}}\sum_{(s_t,b_t)\in\bar{\tau}^{k:T}}\gamma_L^{t-k}r_L^{\theta^i}(s_t,b_t)$}
\ENDFOR
\ENSURE{$Q_L,V_L$}
\end{algorithmic}
\end{algorithm}

\begin{algorithm}[H]
\caption{$\texttt{EstimateValueBySARSA}(B,Q_L,\theta^i, g^{\theta\dagger})$}
\label{algo:benefit_estimation_sarsa}
\begin{algorithmic}[1]
\REQUIRE{transition sample batch $B$, current critic $Q_L$, current leader's parameter $\theta^i$, follower's best response $g^{\theta\dagger}$}
\FOR{$(s,b,s',b')\in B$}
\STATE{$Q_L(s,b)\gets Q_L(s,b)+\alpha_{SA}\left(r_L^{\theta^i}(s,b)+\gamma_LQ_L(s',b')-Q_L(s,b)\right)$}
\ENDFOR
\FOR{$s\in B$}
\STATE{$V_L(s)\gets\sum_{b\in\mathcal{B}}g^{\theta\dagger}(b\mid s)Q_L(s,b)$}
\ENDFOR
\ENSURE{$Q_L,V_L$}
\end{algorithmic}
\end{algorithm}


\section{Proofs}

\subsection{Preliminaries}\label{apdx:discounted_cumulative_state_visitation}

We introduce the discounted state visitation distribution $d_{\gamma}^{\theta g}(s)$.
For conciseness, we define the next-state transition probability under policies $f_\theta$ and $g$ as:
\begin{equation}
    p^{\theta g}(s' \mid s) := \int_{a}\int_{b} p(s'\mid s, a, b) g(b \mid s, a) f_{\theta}(a \mid s) \rmd a \rmd b .
\end{equation}
Given the initial state-action tuple $(\bar{s}, \bar{a}, \bar{b}) \in \mathcal{S}\times\mathcal{A}\times\mathcal{B}$, the state visitation probability $\rho_t^{\theta g}(s\mid \bar{s},\bar{a},\bar{b})$ at time step $t\geq 1$ is recursively defined as:
\begin{align}
    \rho_1^{\theta g}(s\mid \bar{s},\bar{a},\bar{b}) &:= p(s\mid \bar{s},\bar{a},\bar{b}),\\
    \rho_t^{\theta g}(s\mid \bar{s},\bar{a},\bar{b}) &:= \int_{\dot{s}} p^{\theta g}(s\mid \dot{s}) \rho_{t-1}^{\theta g}(\dot{s}\mid \bar{s},\bar{a},\bar{b}) \rmd \dot{s} .
\end{align}
The conditional $\gamma$-discounted state visitation distribution is then defined as: 
\begin{equation}
    d_\gamma^{\theta g}(s\mid\bar{s}, \bar{a}, \bar{b}) := (1 - \gamma) \sum_{t=1}^{\infty} \gamma^{t} \rho_t^{\theta g}(s_t\mid\bar{s}, \bar{a}, \bar{b}).
\end{equation}
The corresponding unconditional version is: 
\begin{equation}
    d_\gamma^{\theta g}(s) := \int_{\bar{s}}\int_{\bar{a}}\int_{\bar{b}} d_\gamma^{\theta g}(s\mid\bar{s}, \bar{a}, \bar{b}) 
    g(\bar{b}\mid\bar{s},\bar{a})f_{\theta}(\bar{a}\mid\bar{s})\rho_0(\bar{s})  \rmd \bar{s} \rmd \bar{a} \rmd \bar{b}.
\end{equation}
For succinctness, we use the notation $d_\gamma^{\theta \dag}$ to denote $d_\gamma^{\theta g^{\theta\dagger}}$.

We define $\E_{d_{\gamma}^{\theta g}}^{f_{\theta}g}[\cdot]$ as the expectation over the state-action tuple $(s,a,b)\sim g(b\mid s,a)f_{\theta}(a\mid s)d_{\gamma}^{\theta g}(s)$. Separately, we define $\E_\tau^{f_{\theta}g}[\cdot]$ as the expectation over the trajectory $\tau:=(s_0,a_0,b_0,s_1,a_1,b_1,\dots)$ generated under $f_{\theta}$ and $g$, whose probability is given by:
\begin{equation}
    p_\tau^{{\theta} {g}}(\tau) := \rho_0(s_0) \prod_{t=0}^{\infty} p(s_{t+1} \mid s_{t}, a_{t}, b_{t}) g(b_{t}\mid s_{t}, a_{t}) f_{{\theta}}(a_{t}\mid s_{t}).
\end{equation}
It is a known result in discounted MDPs that these two expectations are related for any bounded function $h:\mathcal{S}\times\mathcal{A}\times\mathcal{B}\to\R$:
\begin{equation}
    \E_{d_{\gamma}^{\theta g}}^{f_{\theta}g}\left[h(s,a,b)\right]=(1-\gamma)\E_\tau^{f_{\theta}g}\left[\sum_{t=0}^\infty\gamma^t h(s_t,a_t,b_t)\right].
\end{equation}

We also use $\mathbb{E}_\tau^{\theta g}[\cdot]$ to denote the corresponding trajectory expectation in the configurable MDPs setting, where $\theta$ represents the leader's configurable parameters of the environment.

\subsection{Proof of Theorem~\ref{thm:hg-configurable}}\label{apdx:proof_hg-configurable}

\begin{proof}
By Theorem 2 and Proposition 1 in \citet{Thoma2024-em}, we have 
\begin{subequations}
\begin{align}
    \MoveEqLeft[0]\nabla_{{\theta}}J_L({\theta}, g^{\theta\dagger})\nonumber\\
    &=\E_\tau^{\theta g^{\theta\dagger}}\bigg[\sum_{t=0}^\infty \gamma_L^t \bigg(\nabla_{\theta} r_L^{\theta}(s_t, b_t) 
    + V_L^{{\theta}\dagger}(s_t) \nabla_{\theta}\log
    p^{\theta} (s_{t} \mid s_{t-1}, b_{t-1}) + \frac{1}{\beta} Q_L^{{\theta}\dagger}(s_t, b_t)\nabla_{\theta} \left(Q_F^{{\theta} \dagger}(s_t, b_t) - V_F^{{\theta} \dagger}(s_t)\right)\bigg) \bigg]\\
    &=\E_\tau^{\theta g^{\theta\dagger}}\bigg[\sum_{t=0}^\infty \gamma_L^t \bigg(\nabla_{\theta} r_L^{\theta}(s_t, b_t) 
    + V_L^{{\theta}\dagger}(s_t) \nabla_{\theta}\log
    p^{\theta} (s_{t} \mid s_{t-1}, b_{t-1}) + \frac{1}{\beta} Q_L^{{\theta}\dagger}(s_t, b_t) \left(\nabla_{\theta}Q_F^{{\theta} \dagger}(s_t, b_t) - \nabla_{\theta}V_F^{{\theta} \dagger}(s_t)\right)\bigg) \bigg]\qquad \label{eq:thoma_hg}
\end{align}
\end{subequations}
where $p^{\theta}(s_0\mid s_{-1},b_{-1})$ refers to $\rho_0^{\theta}(s_0)$.

Differentiating both terms of the optimal Bellman equation
\begin{align}
    V_F^{\theta\dag}(s)=\beta\log\int_{\mathcal{B}}\exp\left(\beta^{-1}Q_F^{\theta\dag}(s,b)\right)\rmd b,
\end{align}
we obtain
\begin{subequations}
\begin{align}
    \nabla_{\theta}V_F^{\theta\dagger}(s)&=\beta\dfrac{\nabla_{\theta}\int_{\mathcal{B}}\exp\left(\beta^{-1}Q_F^{\theta\dagger}(s,b)\right)\rmd b}{\int_{\mathcal{B}}\exp\left(\beta^{-1}Q_F^{\theta\dagger}(s,b)\right)\rmd b}\\
    &=\beta\int_{\mathcal{B}}\dfrac{\nabla_{\theta}\exp\left(\beta^{-1}Q_F^{\theta\dagger}(s,a,b)\right)}{\int_{\mathcal{B}}\exp\left(\beta^{-1}Q_F^{\theta\dagger}(s,b)\right)\rmd b}\rmd b\\
    &=\beta\int_{\mathcal{B}}\dfrac{\exp\left(\beta^{-1}Q_F^{\theta\dagger}(s,b)\right)}{\int_{\mathcal{B}}\exp\left(\beta^{-1}Q_F^{\theta\dagger}(s,b)\right)\rmd b}\nabla_{\theta}\beta^{-1}Q_F^{\theta\dagger}(s,b)\rmd b\\
    &=\int_{\mathcal{B}}g^{\theta \dagger}(b\mid s)\nabla_{\theta}Q_F^{\theta\dagger}(s,b)\rmd b\\
    &=\E_{b\sim g^{\theta \dagger}(\cdot\mid s)}\left[\nabla_{\theta}Q_F^{\theta\dagger}(s,b)\right].
\end{align}
\end{subequations}
Then, using the Boltzmann covariance trick, the second term in \eqref{eq:thoma_hg} is transformed as
\begin{subequations}
\begin{align}
    &\E_\tau^{\theta g^{\theta\dagger}}\left[\sum_{t=0}^\infty \gamma_L^t \frac{1}{\beta} Q_L^{{\theta}\dagger}(s_t, b_t) \left(\nabla_{\theta}Q_F^{{\theta} \dagger}(s_t, b_t) - \nabla_{\theta}V_F^{{\theta} \dagger}(s_t)\right) \right]\nonumber\\
    &=\E_\tau^{\theta g^{\theta\dagger}}\left[\sum_{t=0}^\infty \gamma_L^t \frac{1}{\beta} 
    \cdot\E_{b_t\sim g^{\theta\dagger}(\cdot\mid s_t)}\left[Q_L^{{\theta}\dagger}(s_t, b_t) \left(\nabla_{\theta}Q_F^{{\theta} \dagger}(s_t, b_t) - \E_{b_t\sim g^{\theta\dagger}(\cdot\mid s_t)}\left[\nabla_{\theta}Q_F^{{\theta} \dagger}(s_t,b_t)\right]\right)\right] \right]\\
    &=\E_\tau^{\theta g^{\theta\dagger}}\left[\sum_{t=0}^\infty \gamma_L^t \frac{1}{\beta} 
    \cdot\mathrm{Cov}_{b_t}\left[Q_L^{{\theta}\dagger}(s_t, b_t), \nabla_{\theta}Q_F^{{\theta} \dagger}(s_t, b_t)~\middle|~ s_t\right] \right]\\
    &=\E_\tau^{\theta g^{\theta\dagger}}\left[\sum_{t=0}^\infty \gamma_L^t \frac{1}{\beta} 
    \cdot\E_{b_t\sim g^{\theta\dagger}(\cdot\mid s_t)}\left[\left(Q_L^{{\theta}\dagger}(s_t, b_t) - \E_{b_t\sim g^{\theta\dagger}(\cdot\mid s_t)}\left[Q_L^{{\theta}\dagger}(s_t, b_t)\right]\right) \nabla_{\theta}Q_F^{{\theta} \dagger}(s_t, b_t) \right] \right]\\
    &=\E_\tau^{\theta g^{\theta\dagger}}\left[\sum_{t=0}^\infty \gamma_L^t \frac{1}{\beta} \left(Q_L^{{\theta}\dagger}(s_t, b_t) - \E_{b_t\sim g^{\theta\dagger}(\cdot\mid s_t)}\left[Q_L^{{\theta}\dagger}(s_t, b_t)\right]\right) \nabla_{\theta}Q_F^{{\theta} \dagger}(s_t, b_t)\right].
\end{align}
\end{subequations}

Therefore, we have
\begin{multline}
    \nabla_{{\theta}}J_L({\theta}, g^{\theta\dagger})
    =\E_\tau^{\theta g^{\theta\dagger}}\bigg[\sum_{t=0}^\infty \gamma_L^t \bigg(\nabla_{\theta} r_L^{\theta}(s_t, b_t) 
    + V_L^{{\theta}\dagger}(s_t) \nabla_{\theta}\log
    p^{\theta} (s_{t} \mid s_{t-1}, b_{t-1})\\
    + \frac{1}{\beta} \left(Q_L^{{\theta}\dagger}(s_t, b_t) - \E_{b_t\sim g^{\theta\dagger}(\cdot\mid s_t)}\left[Q_L^{{\theta}\dagger}(s_t, b_t)\right]\right) \nabla_{\theta} Q_F^{{\theta} \dagger}(s_t, b_t) \bigg) \bigg],
\end{multline}
Since it holds that
\begin{equation}
\nabla_{{\theta}}Q_F^{{\theta}\dag}(s,b) = \E_\tau^{\theta g^{\theta\dagger}}\bigg[\sum_{t=0}^\infty\gamma_F^t\nabla_{{\theta}}r_F^{{\theta}}(s_t,b_t) 
+ \gamma_F^{t+1}V_F^{{\theta}\dag}(s_{t+1})\nabla_{{\theta}}\log p^{{\theta}}(s_{t+1}\mid s_t,b_t) \bigg|s_0=s,b_0=b\bigg]
\end{equation}
by Theorem 2 in \citet{Thoma2024-em}, Theorem~\ref{thm:hg-configurable} is proved.
\end{proof}

\subsection{Proof of Theorem~\ref{thm:hg-mg-stochastic}}

First, we derive two lemmas.

\begin{lemma}
\begin{equation}
\nabla_{\theta} J_L(\theta, g^{\theta\dagger})=\dfrac{1}{1-\gamma_L}\E_{d_{\gamma_L}^{\theta\dagger}}^{f_{\theta}g^{\theta \dagger}}\left[Q_L^{\theta \dagger}(s,a,b)\nabla_{\theta}\log f_{\theta}(a\mid s)
        +\mathrm{Cov}_{b}\left[Q_L^{\theta \dagger}(s,a,b), \nabla_{\theta}\log g^{\theta \dagger}(b\mid s,a)~\middle|~s,a\right]\right].
\end{equation}
\end{lemma}

\begin{proof}
Viewing $f_{\theta}(a| s)g^{\theta \dagger}(b| s,a)$ as a joint policy parameterized by $\theta$, the derivative of $J_L(\theta, g^{\theta\dagger})=\E_{\tau}\left[\sum_{t=0}^\infty \gamma_L^tr_L(s_t,a_t,b_t)\right]$ with respect to $\theta$ can be computed by using the policy gradient theorem for single-agent MDPs (without entropy regularization) as
\begin{subequations}
\begin{align}
    \nabla_{\theta}J_L(\theta, g^{\theta\dagger})
    &=\dfrac{1}{1-\gamma_L}\E_{d_{\gamma_L}^{\theta\dagger}}^{f_{\theta}g^{\theta \dagger}}\left[Q_L^{\theta \dagger}(s,a,b) \cdot \nabla_{\theta}\log\left(f_{\theta}(a\mid s)g^{\theta \dagger}(b\mid s,a)\right)\right]\\
    &=\dfrac{1}{1-\gamma_L}\E_{d_{\gamma_L}^{\theta\dagger}}^{f_{\theta}g^{\theta \dagger}}\left[Q_L^{\theta \dagger}(s,a,b)\nabla_{\theta}\log f_{\theta}(a\mid s)+Q_L^{\theta \dagger}(s,a,b)\nabla_{\theta}\log g^{\theta \dagger}(b\mid s,a)\right] .\label{eq:nabla_L_J_L_qr_bbf_first}
\end{align}
\end{subequations}
Differentiating both terms of the optimal Bellman equation
\begin{align}
    V_F^{f\dag}(s,a)=\beta\log\int_{\mathcal{B}}\exp\left(\beta^{-1}Q_F^{f\dag}(s,a,b)\right)\rmd b,
\end{align}
we obtain
\begin{subequations}
\begin{align}
    \nabla_{\theta}V_F^{\theta\dagger}(s,a)&=\beta\dfrac{\nabla_{\theta}\int_{\mathcal{B}}\exp\left(\beta^{-1}Q_F^{\theta\dagger}(s,a,b)\right)\rmd b}{\int_{\mathcal{B}}\exp\left(\beta^{-1}Q_F^{\theta\dagger}(s,a,b)\right)\rmd b}\\
    &=\beta\int_{\mathcal{B}}\dfrac{\nabla_{\theta}\exp\left(\beta^{-1}Q_F^{\theta\dagger}(s,a,b)\right)}{\int_{\mathcal{B}}\exp\left(\beta^{-1}Q_F^{\theta\dagger}(s,a,b)\right)\rmd b}\rmd b\\
    &=\beta\int_{\mathcal{B}}\dfrac{\exp\left(\beta^{-1}Q_F^{\theta\dagger}(s,a,b)\right)}{\int_{\mathcal{B}}\exp\left(\beta^{-1}Q_F^{\theta\dagger}(s,a,b)\right)\rmd b}\nabla_{\theta}\beta^{-1}Q_F^{\theta\dagger}(s,a,b)\rmd b\\
    &=\int_{\mathcal{B}}g^{\theta \dagger}(b\mid s,a)\nabla_{\theta}Q_F^{\theta\dagger}(s,a,b)\rmd b\\
    &=\E_{b\sim g^{\theta \dagger}(\cdot\mid s,a)}\left[\nabla_{\theta}Q_F^{\theta\dagger}(s,a,b)\right].
\end{align}
\end{subequations}
Noting that $g^{\theta \dagger}(b\mid s,a)=\exp\left(\beta^{-1}\left(Q_F^{\theta\dagger}(s,a,b)-V_F^{\theta\dagger}(s,a)\right)\right)$ and 
\begin{subequations}
\begin{align}
    \E_{b\sim g^{\theta \dagger}(\cdot\mid s)} \left[ \nabla_{\theta}\log g^{\theta \dagger}(b\mid s,a) \right]
    &=\beta^{-1} \E_{b\sim g^{\theta \dagger}(\cdot\mid s)} \left[\nabla_{\theta}Q_F^{\theta\dagger}(s,a,b)-\nabla_{\theta}V_F^{\theta\dagger}(s,a)\right]\\
    &=\beta^{-1} \E_{b\sim g^{\theta \dagger}(\cdot\mid s)} \left[\nabla_{\theta}Q_F^{\theta\dagger}(s,a,b)\right] - \nabla_{\theta}V_F^{\theta\dagger}(s,a)\\
    &=\beta^{-1} \left( \nabla_{\theta}V_F^{\theta\dagger}(s,b) - \nabla_{\theta}V_F^{\theta\dagger}(s,a) \right)\\
    & = 0,
\end{align}\label{eq:exp-nabla-log}
\end{subequations}
substituting $\nabla_{\theta}\log g^{\theta \dagger}(b\mid s,a) = \nabla_{\theta}\log g^{\theta \dagger}(b\mid s,a) - \E_{b\sim g^{\theta \dagger}(\cdot\mid s)} \left[ \nabla_{\theta}\log g^{\theta \dagger}(b\mid s,a) \right]$ into
\eqref{eq:nabla_L_J_L_qr_bbf_first} leads to
\begin{subequations}
\begin{align}
    \nabla_{\theta}J_L(\theta, g^{\theta\dagger})
    &\begin{multlined}[t]
        =\dfrac{1}{1-\gamma_L}\E_{d_{\gamma_L}^{\theta\dagger}}^{f_{\theta}g^{\theta \dagger}}\Big[Q_L^{\theta \dagger}(s,a,b)\nabla_{\theta}\log f_{\theta}(a\mid s)\\
        +Q_L^{\theta \dagger}(s,a,b)\left(\nabla_{\theta}\log g^{\theta \dagger}(b\mid s,a)-\E_{b\sim g^{\theta \dagger}(\cdot\mid s)}\left[\nabla_{\theta}\log g^{\theta \dagger}(b\mid s,a)\right]\right)\Big]
    \end{multlined}\\
    &\begin{multlined}[t]
        =\dfrac{1}{1-\gamma_L}\E_{d_{\gamma_L}^{\theta\dagger}}^{f_{\theta}g^{\theta \dagger}}\Big[Q_L^{\theta \dagger}(s,a,b)\nabla_{\theta}\log f_{\theta}(a\mid s)\\
        +\E_{b\sim g^{\theta \dagger}(\cdot\mid s,a)}\left[Q_L^{\theta \dagger}(s,a,b)\left(\nabla_{\theta}\log g^{\theta \dagger}(b\mid s,a)-\E_{b\sim g^{\theta \dagger}(\cdot\mid s,a)}\left[\nabla_{\theta}\log g^{\theta \dagger}(b\mid s,a)\right]\right)\right]\Big]
    \end{multlined}\\
    &\begin{multlined}[t]
        =\dfrac{1}{1-\gamma_L}\E_{d_{\gamma_L}^{\theta\dagger}}^{f_{\theta}g^{\theta \dagger}}\Big[Q_L^{\theta \dagger}(s,a,b)\nabla_{\theta}\log f_{\theta}(a\mid s)
        +\mathrm{Cov}_{b}\left[Q_L^{\theta \dagger}(s,a,b), \nabla_{\theta}\log g^{\theta \dagger}(b\mid s,a)~\middle|~s,a\right]\Big].
    \end{multlined}\label{eq:hypergrad_with_cov}
\end{align}
\end{subequations}
This completes the proof.
\end{proof}

\begin{lemma}
\begin{multline}
\mathrm{Cov}_{b}\left[Q_L^{\theta \dagger}(s,a,b), \nabla_{\theta}\log g^{\theta \dagger}(b\mid s,a)~\middle|~s,a\right]
\\
=\dfrac{1}{\beta}\E_{b\sim g^{\theta \dagger}(\cdot\mid s,a)}\left[\left(Q_L^{\theta \dagger}(s,a,b)
        -\E_{b\sim g^{\theta \dagger}(\cdot\mid s,a)}\left[Q_L^{\theta \dagger}(s,a,b)\right]\right)
        \E_\tau^{f_{\theta}g^{\theta\dagger}}\left[\sum_{t=0}^\infty\gamma_F^t V_F^{\theta\dag}(s_t,a_t)\nabla_{\theta}\log f_{\theta}(a_t\mid s_t)\middle| s,a,b\right]\right]
\end{multline}
\end{lemma}
\begin{proof}
Note first that
\begin{subequations}
\begin{align}
 \nabla_{\theta}\log g^{\theta \dagger}(b\mid s,a)
 &=\dfrac{1}{\beta}\left(\nabla_{\theta}Q_F^{\theta\dagger}(s,a,b)-\nabla_{\theta}V_F^{\theta\dagger}(s,a)\right)\\
 &=\dfrac{1}{\beta}\left(\nabla_{\theta}Q_F^{\theta\dagger}(s,a,b)-\E_{b\sim g^{\theta \dagger}(\cdot\mid s,a)}\left[Q_F^{\theta\dagger}(s,a,b)\right]\right).
\end{align}
\end{subequations}
Then, we have
\begin{subequations}
\begin{align}
    \mathrm{Cov}_{b}\left[Q_L^{\theta \dagger}(s,a,b), \nabla_{\theta}\log g^{\theta \dagger}(b\mid s,a)~\middle|~s,a\right]
    &
        =\E_{b\sim g^{\theta \dagger}(\cdot\mid s,a)}\left[\dfrac{1}{\beta}Q_L^{\theta \dagger}(s,a,b)\left(\nabla_{\theta}Q_F^{\theta\dagger}(s,a,b)-\E_{b\sim g^{\theta \dagger}(\cdot\mid s,a)}\left[\nabla_{\theta}Q_F^{\theta\dagger}(s,a,b)\right]\right)\right]
   \\
    &
        =\E_{b\sim g^{\theta \dagger}(\cdot\mid s,a)}\left[\dfrac{1}{\beta}\left(Q_L^{\theta \dagger}(s,a,b) -\E_{b\sim g^{\theta \dagger}(\cdot\mid s,a)}\left[Q_L^{\theta \dagger}(s,a,b)\right]\right)\nabla_{\theta}Q_F^{\theta\dagger}(s,a,b)\right].\label{eq:nabla1_J1_nabla1_Q2}
\end{align}
\end{subequations}
Differentiating both sides of the Bellman optimal equation,
\begin{align}
    Q_F^{\theta\dagger}(s,a,b)=r_F(s,a,b)+\gamma_F\E_{s'}\left[\int_{\mathcal{A}}f_{\theta}(a'\mid s')\beta\log\int_{\mathcal{B}}\exp\left(\beta^{-1}Q_F^{\theta\dagger}(s',a',b')\right)\rmd b'\rmd a'\right],
\end{align}
with respect to $\theta$, we obtain
\begin{subequations}
\begin{align}
    \nabla_{\theta}Q_F^{\theta\dagger}(s,a,b)
    &\begin{multlined}[t]
        =\gamma_F\E_{s'}\Bigg[\int_{\mathcal{A}}f_{\theta}(a'\mid s')\Bigg(\nabla_{\theta}\log f_{\theta}(a'\mid s')\cdot \beta\log\int_{\mathcal{B}}\exp\left(\beta^{-1}Q_F^{\theta\dagger}(s',a',b')\right)\rmd b'\\   +\nabla_{\theta}\beta\log\int_{\mathcal{B}}\exp\left(\beta^{-1}Q_F^{\theta\dagger}(s',a',b')\right)\rmd b'\Bigg)\rmd a' \Bigg]
    \end{multlined}\\
    &=\gamma_F\E_{s'\sim p(\cdot\mid s,a,b)}\E_{a'\sim f(\cdot\mid s')}\left[\nabla_{\theta}\log f_{\theta}(a'\mid s')V_F^{\theta\dagger}(s',a')+\int_{\mathcal{B}}g^{\theta \dagger}(b'|s',a')\nabla_{\theta}Q_F^{\theta\dagger}(s',a',b')\rmd b'\right]\\
    &=\gamma_F\E_{s'\sim p(\cdot\mid s,a,b)}\E_{a'\sim f(\cdot\mid s')}\left[\nabla_{\theta}\log f_{\theta}(a'\mid s')V_F^{\theta\dagger}(s',a')\right]+\gamma_F\E_{s'}\E_{a'}\E_{b'\sim g^{\theta \dagger}(\cdot\mid s',a')}\left[\nabla_{\theta}Q_F^{\theta\dagger}(s',a',b')\right].\label{eq:nabla_L_Q_F_qr_bbf_bellman_eq}
\end{align}
\end{subequations}
The form of \eqref{eq:nabla_L_Q_F_qr_bbf_bellman_eq} corresponds to the Bellman expectation equation. Therefore, because of the existence and the uniqueness of the solution to the Bellman expectation equation, we have
\begin{subequations}
\begin{align}
    \nabla_{\theta}Q_F^{\theta\dagger}(s,a,b)
    &=\E_\tau^{f_{\theta}g^{\theta \dagger}}\left[\sum_{t=0}^\infty \gamma_F^t\cdot \gamma_F\E_{s'\sim p(\cdot\mid s_t,a_t,b_t),a'\sim f(\cdot\mid s')}\left[V_F^{\theta\dagger}(s',a')\nabla_{\theta}\log f_{\theta}(a'\mid s')\right]~\middle|~s,a,b\right]\\
    &\begin{multlined}[t]
        =\E_\tau^{f_{\theta}g^{\theta \dagger}}\left[\sum_{t=0}^\infty\gamma_F^t V_F^{\theta\dagger}(s_t,a_t)\nabla_{\theta}\log f_{\theta}(a_t\mid s_t)~\middle|~s,a,b\right] - V_F^{\theta\dagger}(s,a)\nabla_{\theta}\log f_{\theta}(a\mid s). \label{eq:nabla_L_V_F_pgt_qr_bbf}
    \end{multlined}
\end{align}
\end{subequations}
Substituting it into \eqref{eq:nabla1_J1_nabla1_Q2}, we have
\begin{subequations}
\begin{align}
    \MoveEqLeft[2]\mathrm{Cov}_{b}\left[Q_L^{\theta \dagger}(s,a,b), \nabla_{\theta}\log g^{\theta \dagger}(b\mid s,a)~\middle|~s,a\right]\nonumber\\
    &
        =\E_{b\sim g^{\theta \dagger}(\cdot\mid s,a)}\left[\dfrac{1}{\beta}\left(Q_L^{\theta \dagger}(s,a,b) -\E_{b\sim g^{\theta \dagger}(\cdot\mid s,a)}\left[Q_L^{\theta \dagger}(s,a,b)\right]\right)\nabla_{\theta}Q_F^{\theta\dagger}(s,a,b)\right]
   \\
    &\begin{multlined}[t]
        =\E_{b\sim g^{\theta \dagger}(\cdot\mid s,a)}\left[\dfrac{1}{\beta}\left(Q_L^{\theta \dagger}(s,a,b) -\E_{b\sim g^{\theta \dagger}(\cdot\mid s,a)}\left[Q_L^{\theta \dagger}(s,a,b)\right]\right) 
        \cdot \E_\tau^{f_{\theta}g^{\theta \dagger}}\left[\sum_{t=0}^\infty\gamma_F^t V_F^{\theta\dagger}(s_t,a_t)\nabla_{\theta}\log f_{\theta}(s_t\mid s_t)~\middle|~s,a,b\right]\right]\\
        - 0\cdot V_F^{\theta\dagger}(s,a)\nabla_{\theta}\log f_{\theta}(a\mid s)
    \end{multlined}\\
    &\begin{multlined}[t]
        =\dfrac{1}{\beta}\E_{b\sim g^{\theta \dagger}(\cdot\mid s,a)}\left[\left(Q_L^{\theta \dagger}(s,a,b)
        -\E_{b\sim g^{\theta \dagger}(\cdot\mid s,a)}\left[Q_L^{\theta \dagger}(s,a,b)\right]\right)
        \cdot \E_\tau^{f_{\theta}g^{\theta \dagger}}\left[\sum_{t=0}^\infty\gamma_F^t V_F^{\theta\dagger}(s_t,a_t)\nabla_{\theta}\log f_{\theta}(a_t\mid s_t)~\middle|~s,a,b\right]\right].\tag*{\qedhere}
    \end{multlined}
\end{align}
\end{subequations}
\end{proof}

Finally, Substituting it into \eqref{eq:hypergrad_with_cov}, we have
\begin{subequations}
\begin{align}
    \MoveEqLeft[0]\nabla_{\theta}J_L(\theta,g^{\theta\dagger})\nonumber\\
    &\begin{multlined}[t]
        =\dfrac{1}{1-\gamma_L}\E_{d_{\gamma_L}^{\theta\dagger}}^{f_{\theta}g^{\theta \dagger}}\Bigg[Q_L^{\theta \dagger}(s,a,b)\nabla_{\theta}\log f_{\theta}(a\mid s)\\
        +\dfrac{1}{\beta}\E_{b\sim g^{\theta \dagger}(\cdot\mid s,a)}\left[\left(Q_L^{\theta \dagger}(s,a,b)
        -\E_{b\sim g^{\theta \dagger}(\cdot\mid s,a)}\left[Q_L^{\theta \dagger}(s,a,b)\right]\right)
        \cdot\E_\tau^{f_{\theta}g^{\theta \dagger}}\left[\sum_{t=0}^\infty\gamma_F^t V_F^{\theta\dagger}(s_t,a_t)\nabla_{\theta}\log f_{\theta}(a_t\mid s_t)~\middle|~s,a,b\right]\right]\Bigg]
    \end{multlined}\\
    &\begin{multlined}[t]
        =\dfrac{1}{1-\gamma_L}\E_{d_{\gamma_L}^{\theta\dagger}}^{f_{\theta}g^{\theta \dagger}}\Bigg[Q_L^{\theta \dagger}(s,a,b)\nabla_{\theta}\log f_{\theta}(a\mid s)\\
        +\dfrac{1}{\beta}\left(Q_L^{\theta \dagger}(s,a,b)
        -\E_{b\sim g^{\theta \dagger}(\cdot\mid s,a)}\left[Q_L^{\theta \dagger}(s,a,b)\right]\right)
        \cdot\E_\tau^{f_{\theta}g^{\theta \dagger}}\left[\sum_{t=0}^\infty\gamma_F^t V_F^{\theta\dagger}(s_t,a_t)\nabla_{\theta}\log f_{\theta}(a_t\mid s_t)~\middle|~s,a,b\right]\Bigg].
    \end{multlined}
\end{align}
\end{subequations}


\section{Experiments}

\subsection{Baseline Approaches}\label{apdx:baseline}

The pseudo-code for the baseline algorithms used in the Four-Rooms Task experiments is presented in Algorithm~\ref{algo:fourrooms}.

\paragraph{\textsc{Naive-PGD}}
\textsc{Naive-PGD} refers to the standard policy gradient approach that does not account for changes in the follower’s best response.  
Specifically, it relies solely on the partial derivative $\partial_{\theta} J_L(\theta, g)\big|_{g = g^{\theta\dagger}}$,  
which is computed in line~\ref{algo:partial_deriv} of Algorithm~\ref{algo:fourrooms} for the Four-Rooms Task,  
and in line~\ref{algo:partial_deriv_MG} for the MGs.

\paragraph{\textsc{HPGD}~\citep{Thoma2024-em}}
\textsc{HPGD} employs a hypergradient that is equivalent to ours in principle but differs in the formulation of the final term:
\begin{align}
    \dfrac{1}{\beta}\mathbb{E}_{s,b}\left[Q_L^{\theta\dagger}(s,b)\left(\nabla_{\theta}Q_F^{\theta\dagger}(s,b) - \mathbb{E}_{b \sim g^{\theta\dagger}(\cdot \mid s)}\left[\nabla_{\theta}Q_F^{\theta\dagger}(s,b)\right]\right)\right], \label{eq:hpgd_guidance}
\end{align}
where the identity 
\begin{equation}
    \nabla_{\theta}V_F^{\theta\dagger}(s) = \mathbb{E}_{b \sim g^{\theta\dagger}(\cdot \mid s)}\left[\nabla_{\theta}Q_F^{\theta\dagger}(s,b)\right] \label{eq:bellman_eq_value_q_gradient}
\end{equation} holds.  
In the authors’ implementation, the term $Q_L^{\theta\dagger}(s,b)$ is replaced by $Q_L^{\theta\dagger}(s,b) - V_L^{\theta\dagger}(s)$, which is referred to as the ``Benefit'' in this paper.

In the discrete-state task, Four-Rooms, \textsc{HPGD (MC/SA)} estimates each of the two gradient components in \eqref{eq:hpgd_guidance} using the Monte Carlo method based on trajectory segments from the same sampled rollouts, as shown in lines~\ref{algo:fourrooms_f_q_grad} and~\ref{algo:fourrooms_f_v_grad} of Algorithm~\ref{algo:fourrooms}.
In contrast, \textsc{HPGD (oracle)} uses additional trajectories sampled in line~\ref{algo:fourroom_oracle_rollouts} for these Monte Carlo estimations.

In the continuous-state task, $n$-zone Building Thermal Control, \textsc{HPGD (MC/TD)} estimates the first gradient $\nabla_{\theta}Q_F^{\theta\dagger}(s,b)$ via Monte Carlo. 
In contrast, the second gradient is estimated using function approximation.
Considering \eqref{eq:bellman_eq_value_q_gradient} and 
\begin{equation}
\nabla_{{\theta}}Q_F^{{\theta}\dag}(s,b) = \E_\tau\bigg[\sum_{t=0}^\infty\gamma_F^t\nabla_{{\theta}}r_F^{{\theta}}(s_t,b_t) 
+ \gamma_F^{t+1}V_F^{{\theta}\dag}(s_{t+1})\nabla_{{\theta}}\log p^{{\theta}}(s_{t+1}\mid s_t,b_t) \bigg|s_0=s,b_0=b\bigg]
\end{equation}
from Theorem~\ref{thm:hg-configurable}, the second gradient can be regarded as a $|\theta|$-dimensional vector-valued function for state $s$ with vector-valued immediate rewards defined as: $\bm{r}(s_t,b_t,s_{t+1}):=\nabla_{{\theta}}r_F^{{\theta}}(s_t,b_t) 
+ \gamma_F V_F^{{\theta}\dag}(s_{t+1})\nabla_{{\theta}}\log p^{{\theta}}(s_{t+1}\mid s_t,b_t)$.
Hence, this can be estimated by function approximation.
In the experiment, we employ a neural network model and the TD method for the update strategy.

\textsc{HPGD (MC)} estimates the leader’s Q-function and value function via the Monte Carlo method (Algorithm~\ref{algo:benefit_estimation_montecarlo}).  
Meanwhile, \textsc{HPGD (SA)} and \textsc{HPGD (TD)} estimate the Q-function using SARSA and the TD method, respectively.
Then, the value function is computed by taking the expectation of the Q-values over the follower actions (Algorithm~\ref{algo:benefit_estimation_sarsa}).

\paragraph{\textsc{SoBiRL}~\citep{Yang2025-lr}}
\textsc{SoBiRL} employs a hypergradient where the final term is expressed as:
\begin{equation}
    \dfrac{1}{\beta} \mathbb{E}_{\tau} \left[ Q_L^{\theta\dagger}(s_0, b_0) \sum_{(s_t, b_t) \in \tau} \left( \nabla_{\theta} r_F^{\theta}(s_t, b_t) - \mathbb{E}_{b \sim g^{\theta\dagger}(\cdot \mid s_t)} \left[ \nabla_{\theta} r_F^{\theta}(s_t, b) \right] \right) \right],
\end{equation}
where the terms $\nabla_{\theta} r_F^{\theta^i}(s_t, b_t)$ and $\mathbb{E}_{b \sim g^{\theta^\dagger}(\cdot \mid s_t)} \left[ \nabla_{\theta} r_F^{\theta^i}(s_t, b) \right]$  
are intended to approximate the gradients of the Q-function and value function, respectively, via one-step truncation under the follower’s policy:
\begin{equation}
    \nabla_{\theta} r_F^{\theta}(s_t, b_t) \approx \nabla_{\theta} Q_F^{\theta\dagger}(s_t, b_t), \quad
    \mathbb{E}_{b \sim g^{\theta\dagger}(\cdot \mid s_t)} \left[ \nabla_{\theta} r_F^{\theta}(s_t, b) \right] \approx \nabla_{\theta} V_F^{\theta\dagger}(s_t).
\end{equation}
Note that this approximation is infeasible when the transition function is parameterized by the leader.
This is because estimating $\nabla_{\theta}V_F^{\theta\dagger}(s_t)$ requires computing the expectation $\E_b\left[\E_{s'\sim p^{\theta}(\cdot\mid s_t,b)}\left[ \nabla_{\theta}\log p^{\theta}(s'\mid s_t,b) \right]\right]$, which is infeasible when $p^{\theta}$ is unknown.

\paragraph{\textsc{Bi-AC}~\citep{Zhang2020-hw}}
Bi-level Actor-Critic (Bi-AC) and its predecessor, Bi-level Q-learning, were originally proposed by \citet{Zhang2020-hw} as value-based algorithms following the centralized-training-decentralized-execution paradigm.  
The update rule for Bi-level Q-learning is given by:
\begin{equation}
Q_i(s, a, b) \gets (1 - \alpha) Q_i(s, a, b) + \alpha \left( r_i(s, a, b) + \gamma_i Q_i(s', a', b') \right), \quad (i \in \{L, F\}),
\end{equation}
where the next actions are selected as:
\begin{equation}
a' \gets \arg\max_{a} Q_L(s', a, \arg\max_{b} Q_F(s', a, b)), \quad
b' \gets \arg\max_{b} Q_F(s', a', b).
\end{equation}
Here, $(a', b')$ can be interpreted as the Stackelberg equilibrium of the stage game defined by the utility matrices $(Q_L(s', \cdot, \cdot), Q_F(s', \cdot, \cdot))$.

Bi-AC extends Bi-level Q-learning by introducing a parametric actor for the follower.

We adapt Bi-AC to align with our setting, in which the leader employs a stochastic policy as its actor, and the follower’s actor is always set to the best response $g^{\theta\dagger}$ to the current leader policy.  
Accordingly, the leader’s critic $Q_{\omega_L}(s, a, b)$ is updated to minimize the discrepancy from the target $y_L(s, a, b)$, computed as:
\begin{align}
y_L(s, a, b) \gets r_L(s, a, b) + \gamma_L Q_{\omega_L}(s', a', b'),
\end{align}
where $(s, a, b, s')$ is a sampled transition, and the next actions are selected as:
\begin{equation}
a' \gets \arg\max_{a} Q_{\omega_L}(s', a, \arg\max_{b} g^{\theta\dagger}(b \mid s', a)), \quad
b' \gets \arg\max_{b} g^{\theta\dagger}(b \mid s', a').
\end{equation}

Finally, the leader’s actor is updated using the standard policy gradient with the partial derivative  
$\partial_{\theta} J_L(\theta, g)\big|_{g = g^{\theta\dagger}}$,  
in the same manner as in \textsc{Naive-PGD}.

\subsection{Four-Rooms Environment}\label{apdx:fourrooms}

\paragraph{Four-Rooms Task}
The follower's reward function is defined as $r_F^{\theta}(s,b):=\mathbb{I}_{\{s=G\}}+\tilde{r}^{\theta}(s,b)$, where $\tilde{r}^{\theta}:\mathcal{S}\times\mathcal{B}\to\R^-$ represents the penalty imposed on the follower.
This penalty function $\tilde{r}^{\theta}$ is parameterized via a softmax transformation of a vector $x\in\R^{|\mathcal{S}|+1}$. Here, the $i$-th entry of $x$ corresponds to the $i$-th cell in the state space $\mathcal{S}$, while the additional dimension $|\mathcal{S}|+1$ is reserved for allocating ineffective penalties.
Specifically, the penalty is defined as:
\begin{equation}
    \tilde{r}^{\theta}(s,b):=-0.2\times \mathrm{softmax}(s;x).
\end{equation}
This parameterization explicitly constrains the total penalty budget to $-0.2$, following \citet{Thoma2024-em}.
The leader's reward function is defined as:
\begin{equation}
    r_L^{\theta}(s_t, a_t) := \mathbb{I}_{\{s_t = s^{+1}\}} - \lambda \mathbb{I}_{\{s_t = G\}} \sum_{s,a} |\tilde{r}^{\theta}(s,a)|,
\end{equation}
where $s^{+1}$ denotes the target cell, $\lambda$ is a penalty coefficient, set to $\lambda=5.0$ in our experiments.

The transition dynamics are stochastic: the follower's action (up, down, right, or left) fails with a probability of $\frac{1}{3}$, in which case it is randomly replaced by one of the remaining actions.

This task is a modification of the problem proposed by \citet{Thoma2024-em}.
In their original task definition, one of two goal states is randomly selected as a context for each episode. In our experiment, we modify this to have a single goal state. This modification is twofold. First, contexts are not included in our problem formulation; although introducing a context is straightforward, we omit it for conciseness. Second, we observe that the presence of context (multiple goals) encourages the follower to take different actions at the same state, which is favorable for previous works and masks a limitation of those approaches: the requirement for trajectories with diverse actions at the same state. Indeed, the performance of previous approaches degrades when the context is removed. We adapt the task to highlight these limitations and avoid notational complications.

\paragraph{Settings}
For each approach, the leader’s learning rate $\alpha$ and the SARSA learning rate $\alpha_{SA}$ (Algorithm~\ref{algo:critic_update_sarsa}) were selected via a grid search from the sets $\{0.5, 0.1, 0.05, 0.01\}$ and $\{0.1, 0.3, 0.5, 0.7, 0.9, 1.0\}$, respectively.
Both the leader's and the follower's discount rates were set to $0.99$.
The estimated gradient $\widehat{\nabla_{\theta} J_L}$ was clipped to ensure a maximum norm of 1.0.
For \textsc{HPGD (oracle)}, additional trajectories of size $10^4$ were sampled from uniformly distributed initial states.  
These trajectories were utilized to estimate the leader’s Q and value functions via the Monte Carlo method, as well as to compute the gradients of the follower’s Q and value functions with respect to the leader’s parameters.



\newpage

\begin{figure}[H]
    \centering

    \begin{subfigure}[b]{0.33\textwidth}
        \includegraphics[width=\textwidth]{images/FourRooms/UL_rewards_grid_search_reg_lambda_5.0_beta_0_001_steps_100.pdf}
        \caption{$\beta=1\times10^{-3}, \text{BatchSize}=100$}
    \end{subfigure}%
    \hfill
    \begin{subfigure}[b]{0.33\textwidth}
        \includegraphics[width=\textwidth]{images/FourRooms/UL_rewards_grid_search_reg_lambda_5.0_beta_0_003_steps_100.pdf}
        \caption{$\beta=3\times10^{-3}, \text{BatchSize}=100$}
    \end{subfigure}%
    \hfill
    \begin{subfigure}[b]{0.33\textwidth}
        \includegraphics[width=\textwidth]{images/FourRooms/UL_rewards_grid_search_reg_lambda_5.0_beta_0_005_steps_100.pdf}
        \caption{$\beta=5\times10^{-3}, \text{BatchSize}=100$}
    \end{subfigure}%
    \\
    \begin{subfigure}[b]{0.33\textwidth}
        \includegraphics[width=\textwidth]{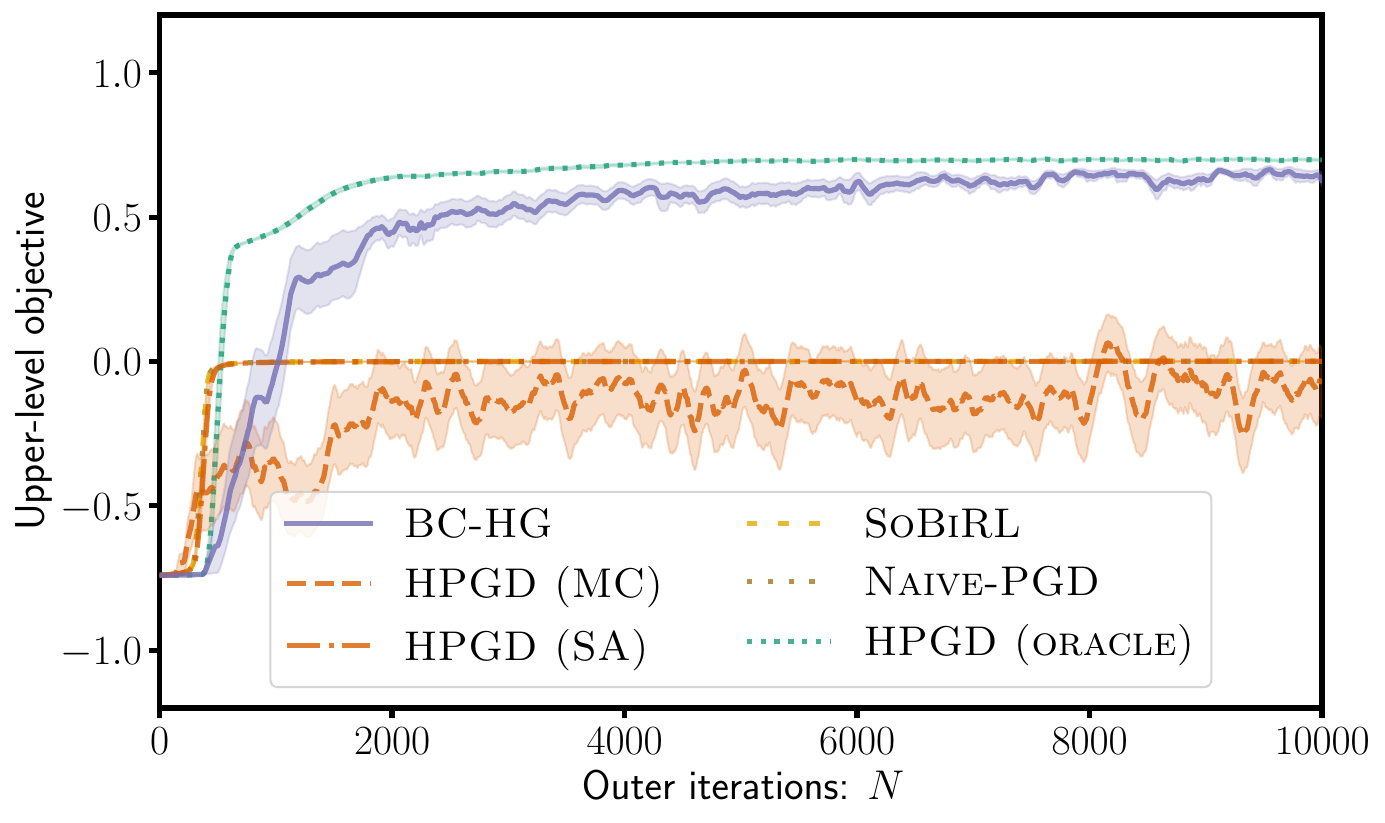}
        \caption{$\beta=1\times10^{-3}, \text{BatchSize}=200$}
    \end{subfigure}%
    \hfill
    \begin{subfigure}[b]{0.33\textwidth}
        \includegraphics[width=\textwidth]{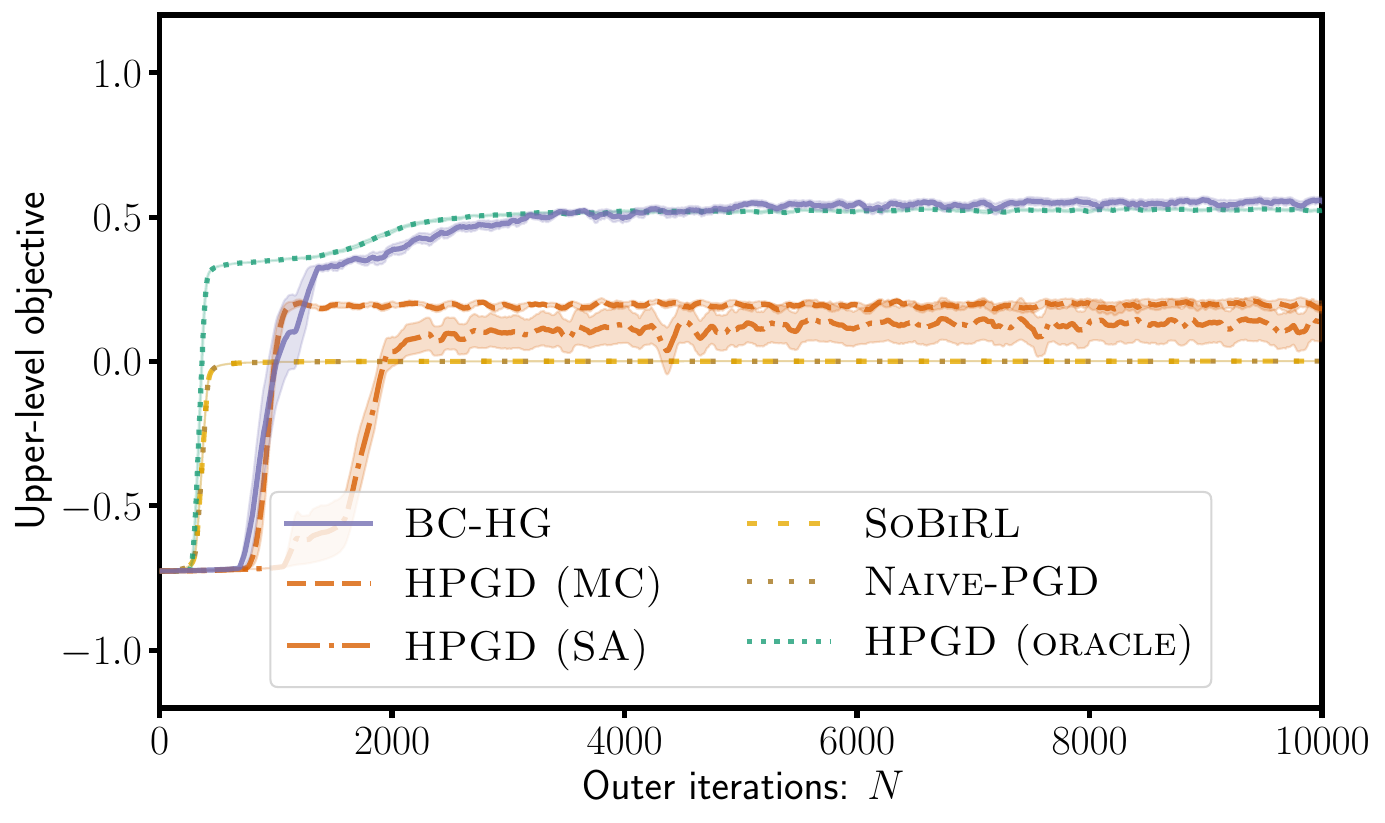}
        \caption{$\beta=3\times10^{-3}, \text{BatchSize}=200$}
    \end{subfigure}%
    \hfill
    \begin{subfigure}[b]{0.33\textwidth}
        \includegraphics[width=\textwidth]{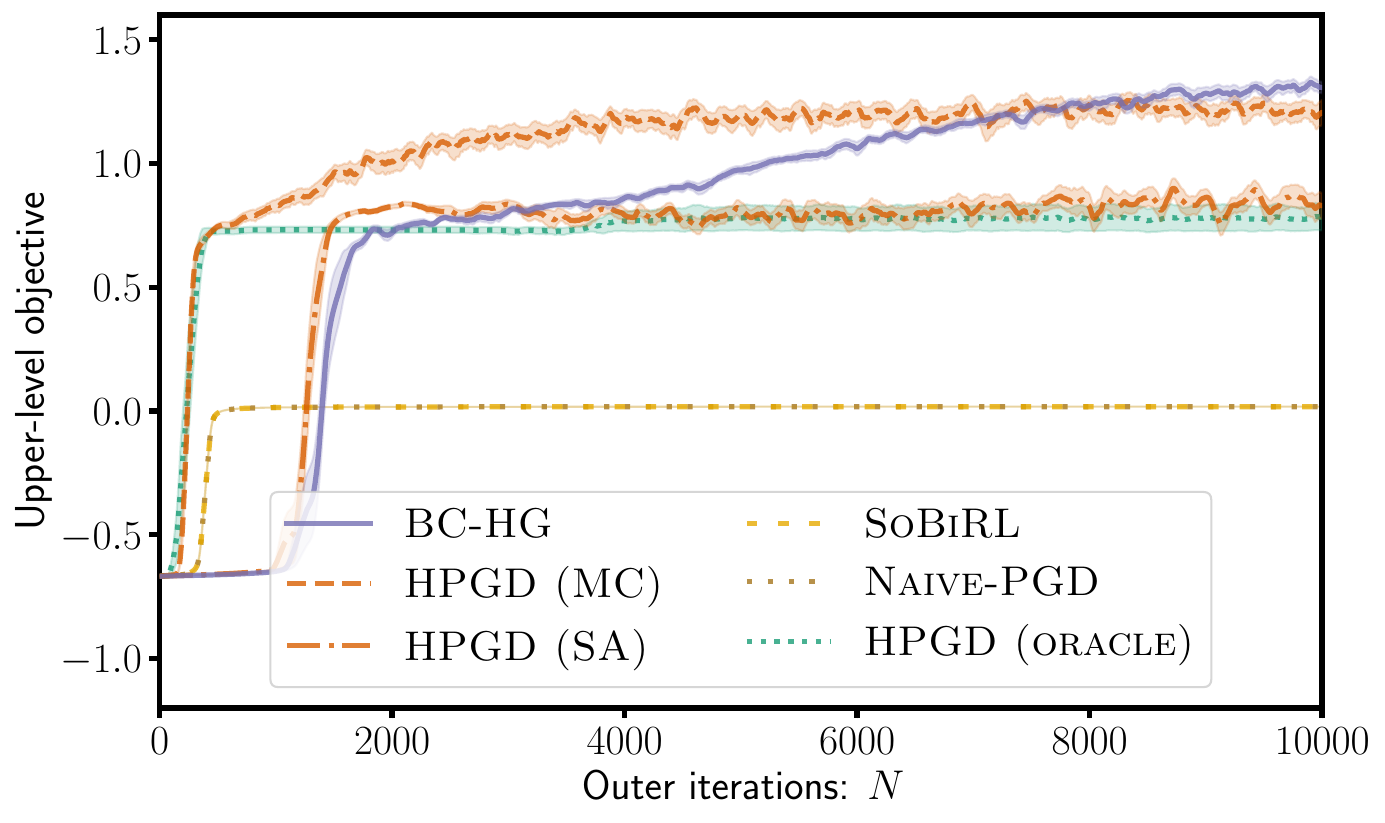}
        \caption{$\beta=5\times10^{-3}, \text{BatchSize}=200$}
    \end{subfigure}%
    \\
    \begin{subfigure}[b]{0.33\textwidth}
        \includegraphics[width=\textwidth]{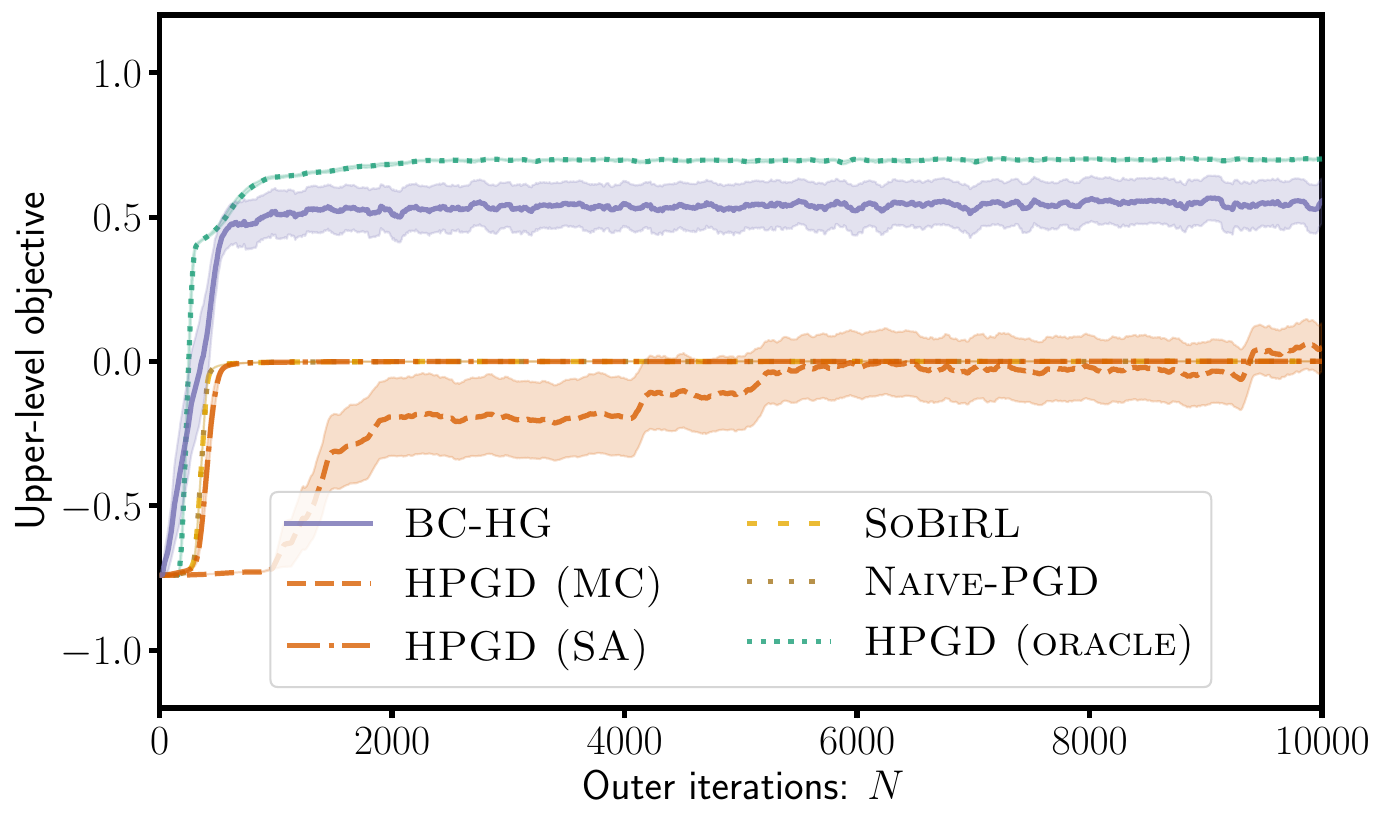}
        \caption{$\beta=1\times10^{-3}, \text{BatchSize}=400$}
    \end{subfigure}%
    \hfill
    \begin{subfigure}[b]{0.33\textwidth}
        \includegraphics[width=\textwidth]{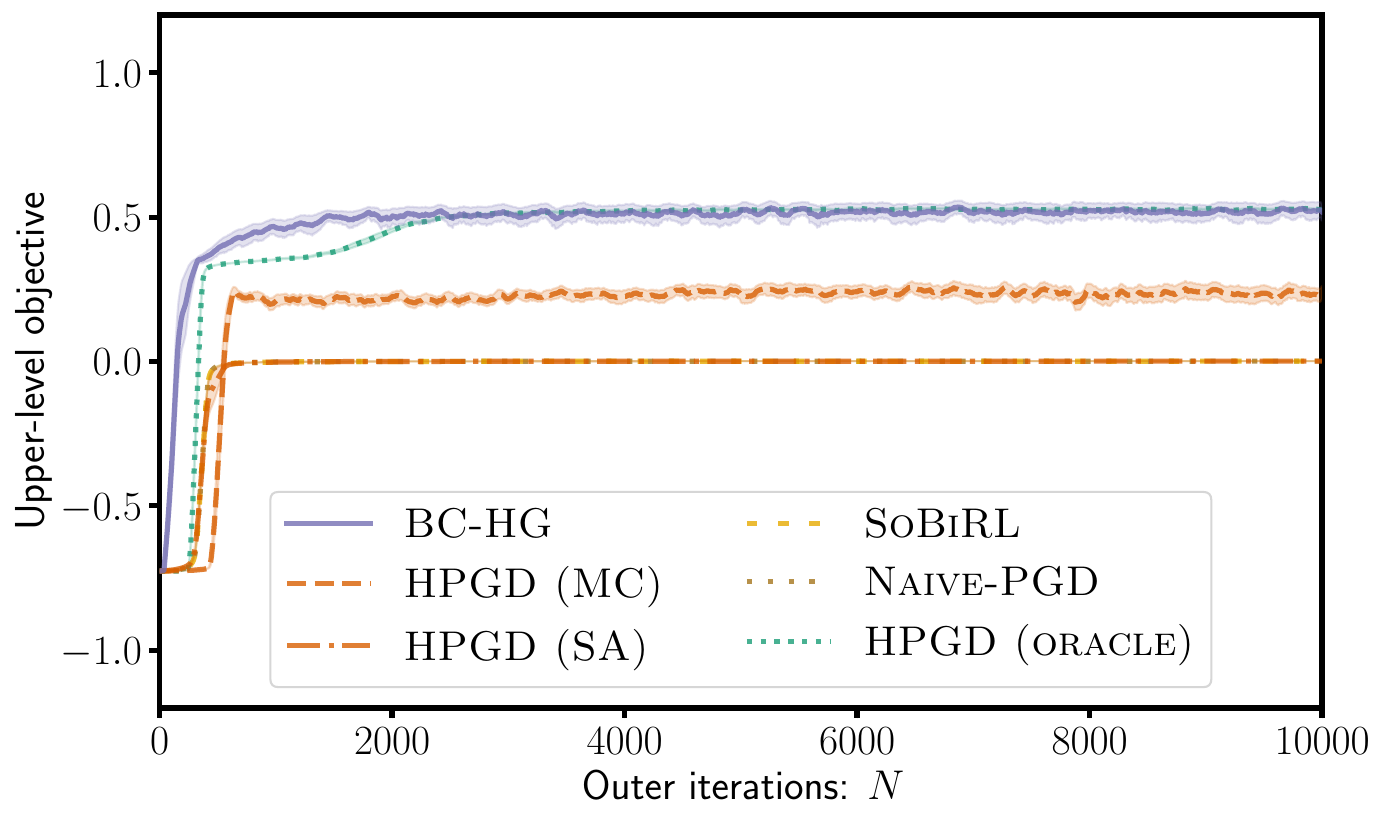}
        \caption{$\beta=3\times10^{-3}, \text{BatchSize}=400$}
    \end{subfigure}%
    \hfill
    \begin{subfigure}[b]{0.33\textwidth}
        \includegraphics[width=\textwidth]{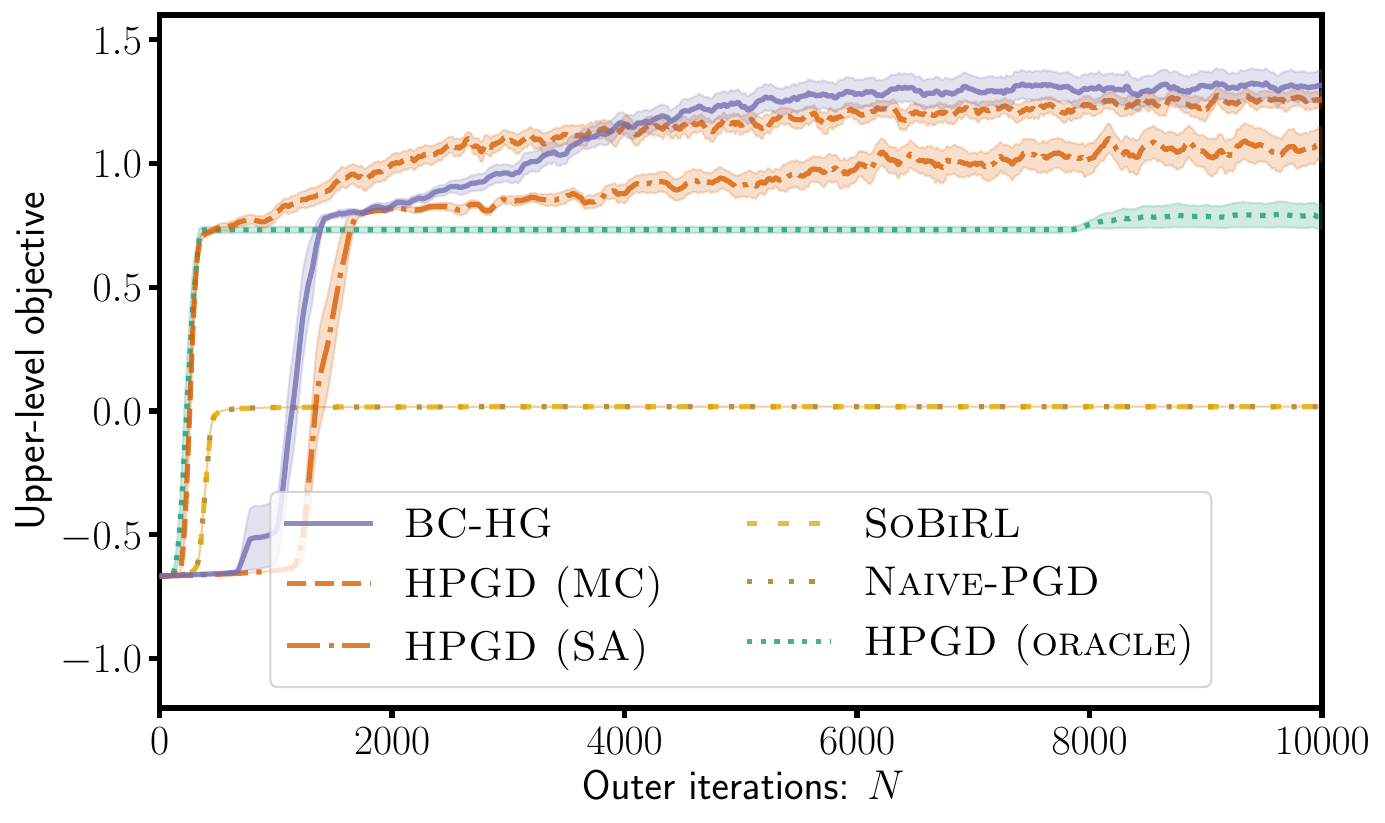}
        \caption{$\beta=5\times10^{-3}, \text{BatchSize}=400$}
    \end{subfigure}%
    \\
    \begin{subfigure}[b]{0.33\textwidth}
        \includegraphics[width=\textwidth]{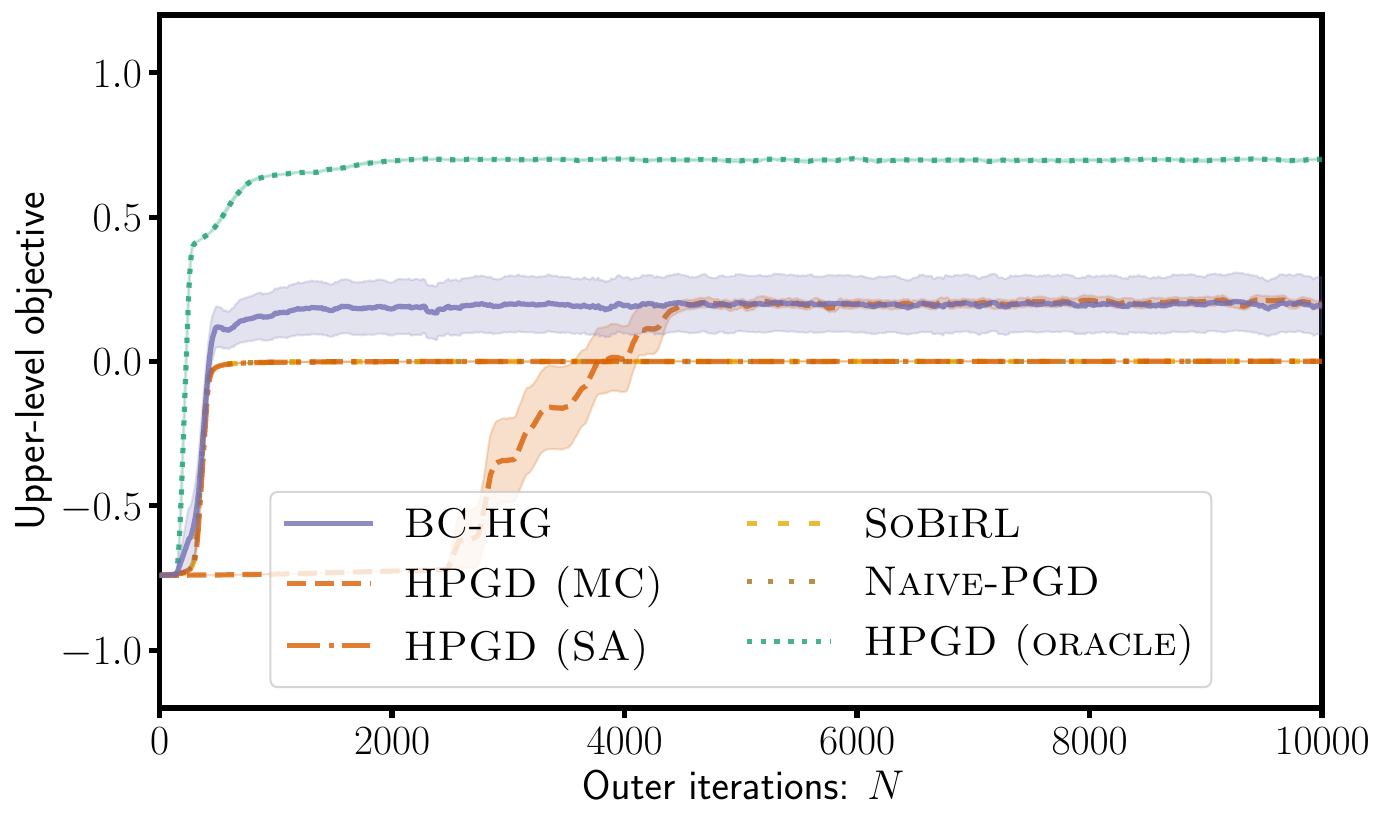}
        \caption{$\beta=1\times10^{-3}, \text{BatchSize}=1000$}
    \end{subfigure}%
    \hfill
    \begin{subfigure}[b]{0.33\textwidth}
        \includegraphics[width=\textwidth]{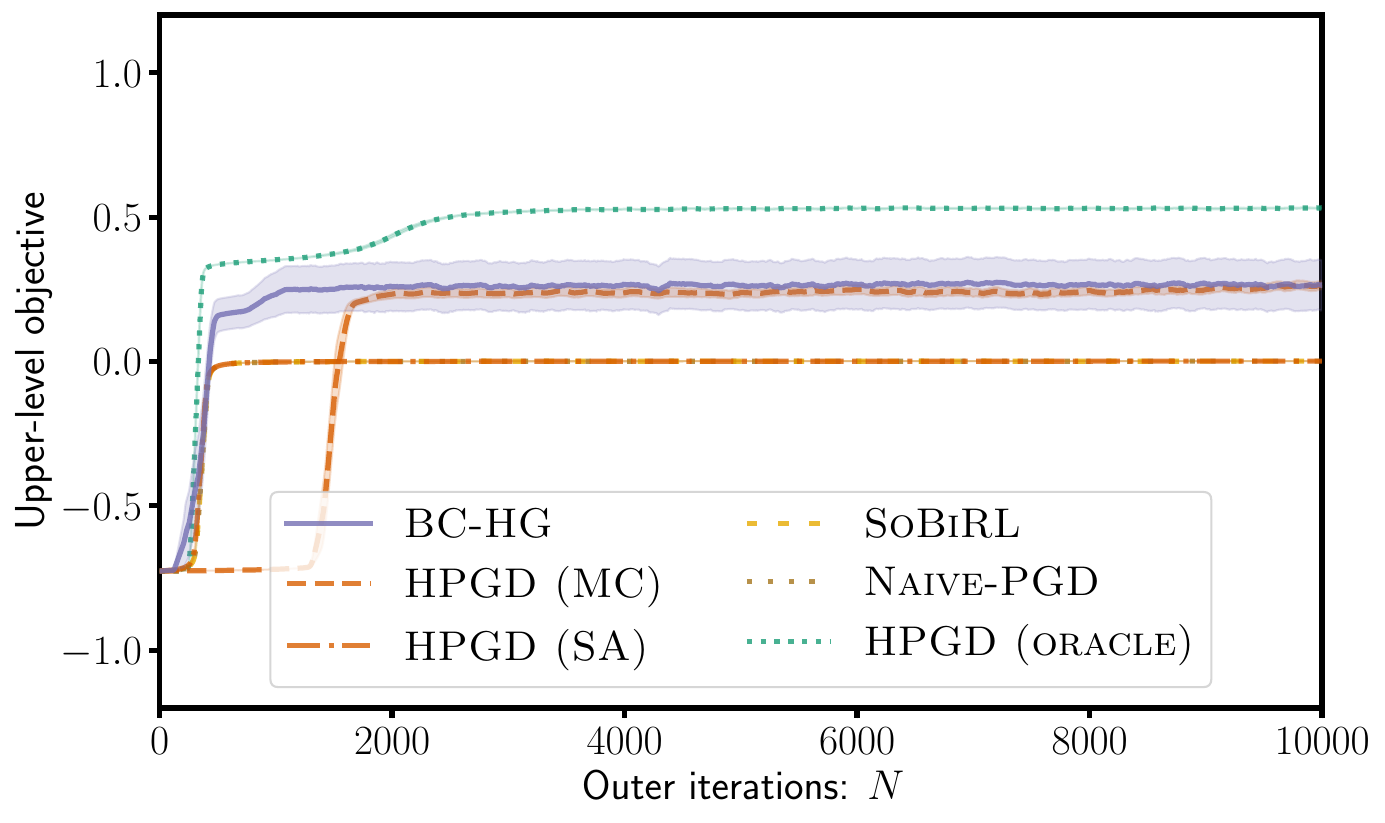}
        \caption{$\beta=3\times10^{-3}, \text{BatchSize}=1000$}
    \end{subfigure}%
    \hfill
    \begin{subfigure}[b]{0.33\textwidth}
        \includegraphics[width=\textwidth]{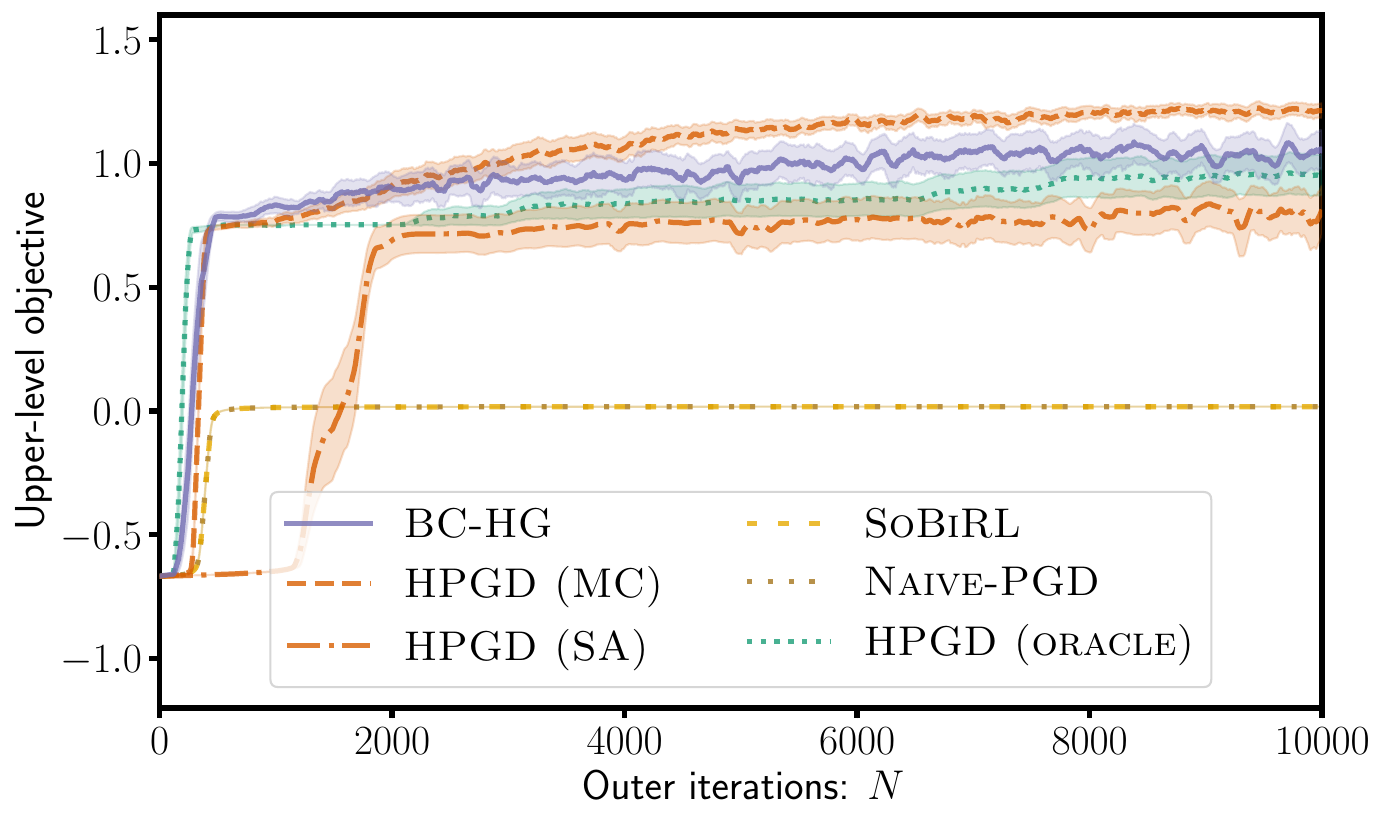}
        \caption{$\beta=5\times10^{-3}, \text{BatchSize}=1000$}
        \label{fig:four_rooms_UL_rewards_grid_search_reg_lambda_5.0_beta_0_005_steps_1000}
    \end{subfigure}%

    \caption{Four-Rooms Task Results across Different BatchSizes ($100, 200, 400, 1000$) and Entropy Regularization Parameters $\beta$ ($1\times10^{-3}, 3\times10^{-3}, 5\times10^{-3}$). \textsc{HPGD (oracle)} uses additional trajectories of size $10^4$ per outer iteration sampled from uniformly distributed initial states.}%
    \label{fig:four_rooms_combined}%
\end{figure}

\paragraph{Results}
The evaluation metric, referred to as the ``Upper-level Objective,'' is the expected value of the initial state under the leader’s reward, estimated via value iteration. Since the initial state is fixed to $S$, this metric corresponds to the value of state $S$.  
Note that this value may exceed 1.0 due to the stochastic nature of the follower’s best response.

For each method, we plot learning curves depicting the mean and standard deviation across 10 random seeds.
The results for batch sizes in $\{100, 200, 400, 1000\}$ and entropy regularization coefficients $\beta \in \{1\times10^{-3}, 3\times10^{-3}, 5\times10^{-3}\}$ are presented in Figure~\ref{fig:four_rooms_combined}.

\subsection{Toy Markov Game Task}\label{apdx:toymarkovgame}

\paragraph{Settings}
For each approach, the number of actor and critic updates per outer iteration was selected via a grid search over the set \{1, 2, 5, 10, 20\}.  
The learning rates for the actor and critic were fixed at $1 \times 10^{-4}$ and $1 \times 10^{-3}$, respectively.
In each outer iteration, the follower is first optimized under the current leader policy.  
Subsequently, three trajectories are sampled using the current leader policy and the optimized follower policy.  
Each trajectory consists of 150 steps and is stored in a trajectory buffer.  
Transition mini-batches and trajectory segments sampled from this buffer are utilized for gradient estimation.
For on-policy methods, including \textsc{BC-HG}, \textsc{Naive-PGD (on-policy)}, and \textsc{Bi-AC (on-policy)}, the buffer size was set to 450 steps to ensure on-policy sampling.  
For the off-policy methods, the buffer size was set to $10^6$ steps.  
The transition mini-batch size was set to 64; accordingly, 64 trajectory segments were used for hypergradient estimation.
The leader’s actor and critic, along with their corresponding target networks, were implemented as neural networks with two hidden layers of 64 units each.

Table~\ref{tab:params_ToyMG} summarizes the hyperparameters shared across all methods.

\paragraph{Results}
Figure~\ref{fig:param_sensitivity_ToyMG} illustrates the parameter sensitivity of each method.  
The plots display the performance for varying numbers of actor and critic updates per outer iteration.
All baseline methods exhibited significant sensitivity to this parameter, with performance often deviating from the optimal return despite starting from relatively favorable initial leader policies.  
In contrast, the proposed method consistently maintained high performance across nearly all settings.

\begin{table}[t]
    \centering
    \begin{tabular}{rll}\toprule
        follower's entropy regularization $\beta$ & $5\times10^{-2}$\\
        mini-batch size & 64\\
        actor learning rate & $1\times10^{-4}$\\
        critic learning rate & $1\times10^{-3}$\\
        leader’s discount rate $\gamma_L$ & 0.99\\
        follower’s discount rate $\gamma_F$ & 0.99\\
        target network smoothing factor & $1\times10^{-2}$\\ \bottomrule
    \end{tabular}
    \caption{Hyperparameters in Toy Markov Game Task}
    \label{tab:params_ToyMG}
\end{table}

\begin{figure}[t]
    \centering
    
    \begin{subfigure}[t]{0.33\textwidth}%
        \centering
        \includegraphics[width=\textwidth]{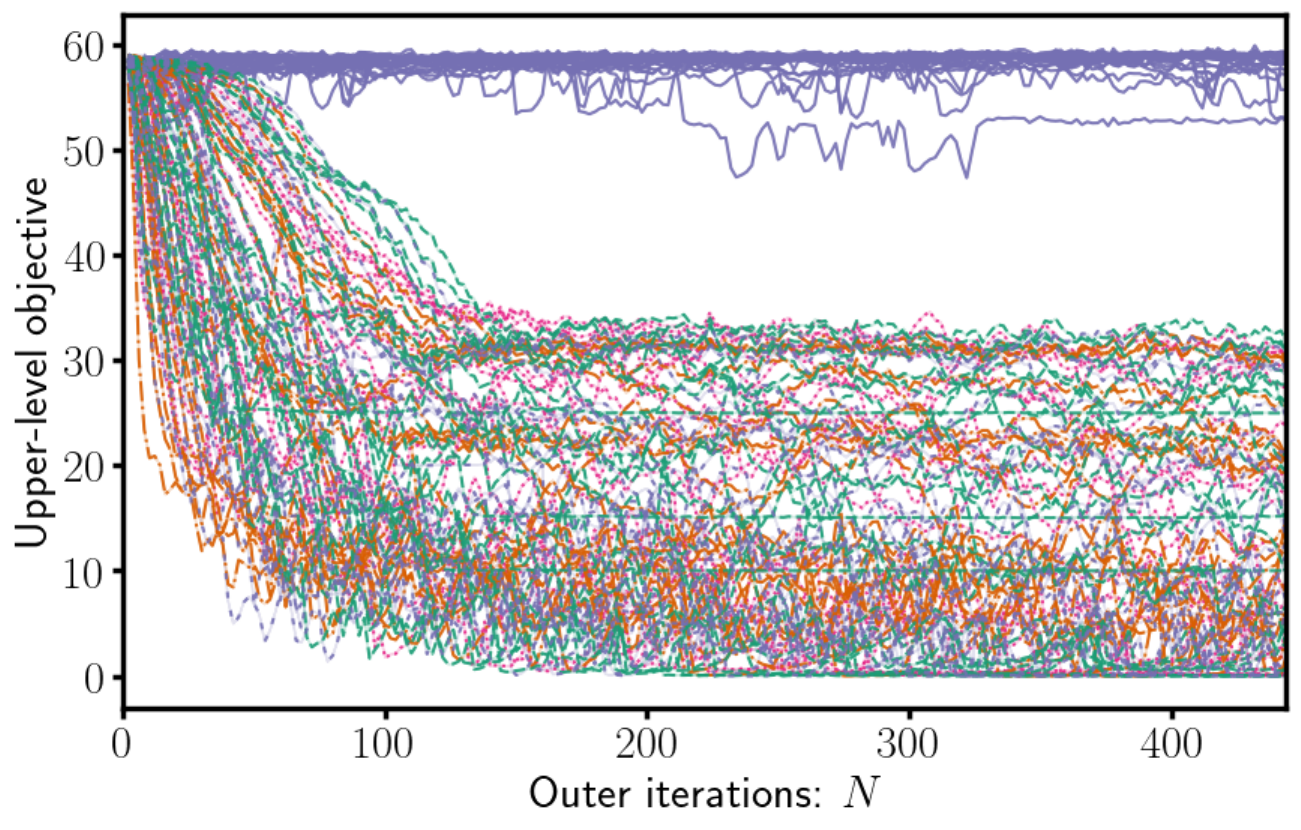}%
        \caption{All Methods}%
    \end{subfigure}%
    \hfill
    \begin{subfigure}[t]{0.33\textwidth}%
        \centering
        \includegraphics[width=\textwidth]{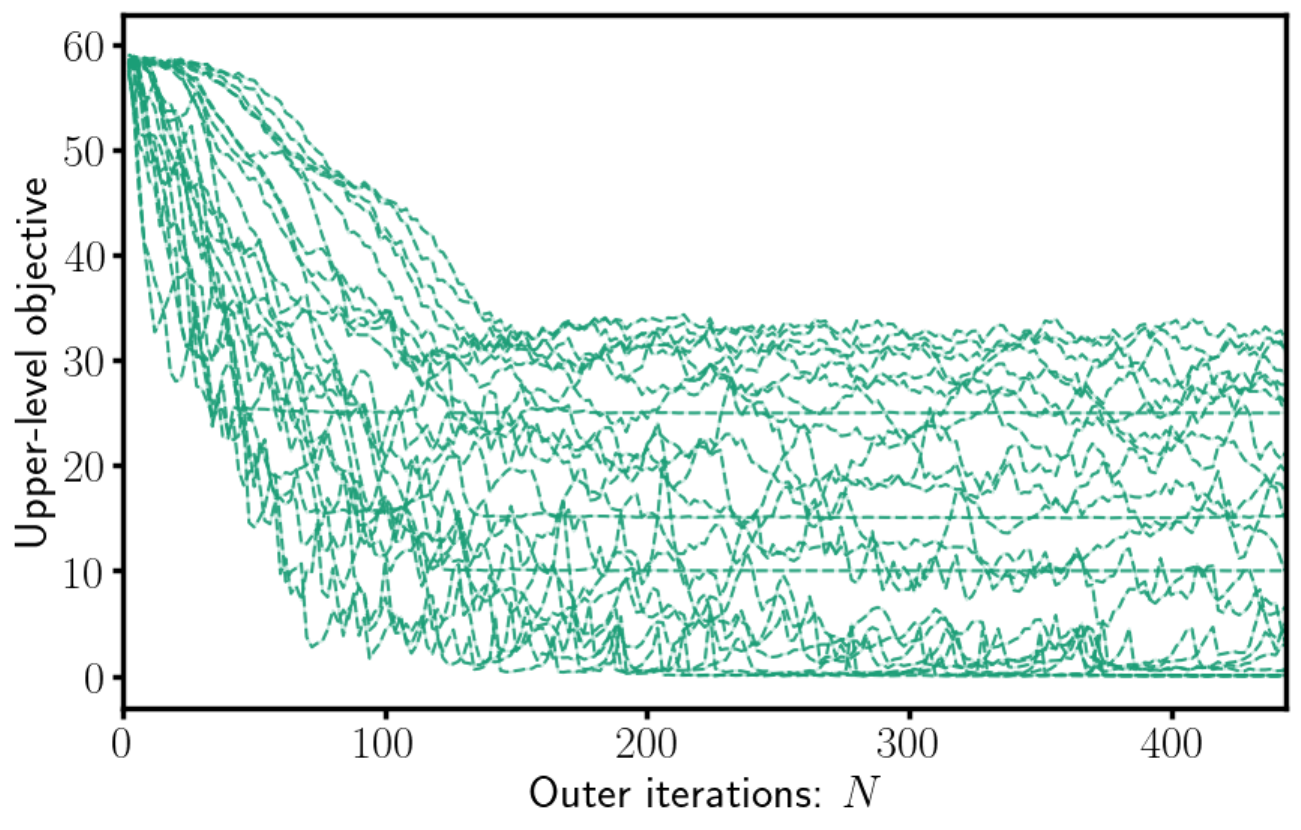}%
        \caption{Naive-PGD (On-Policy)}%
    \end{subfigure}%
    \hfill
    \begin{subfigure}[t]{0.33\textwidth}%
        \centering
        \includegraphics[width=\textwidth]{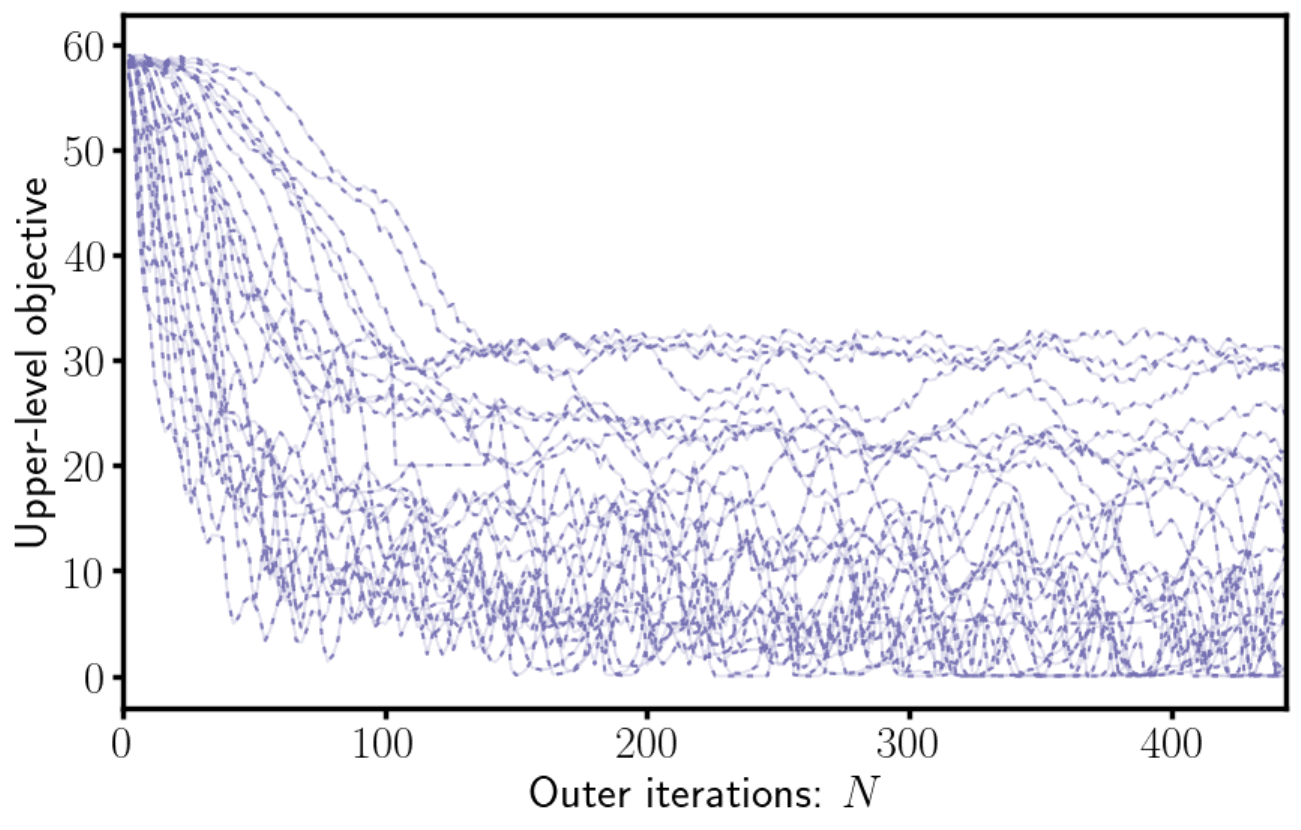}%
        \caption{Naive-PGD (Off-Policy)}
    \end{subfigure}%
    \\
    \begin{subfigure}[t]{0.33\textwidth}%
        \centering
        \includegraphics[width=\textwidth]{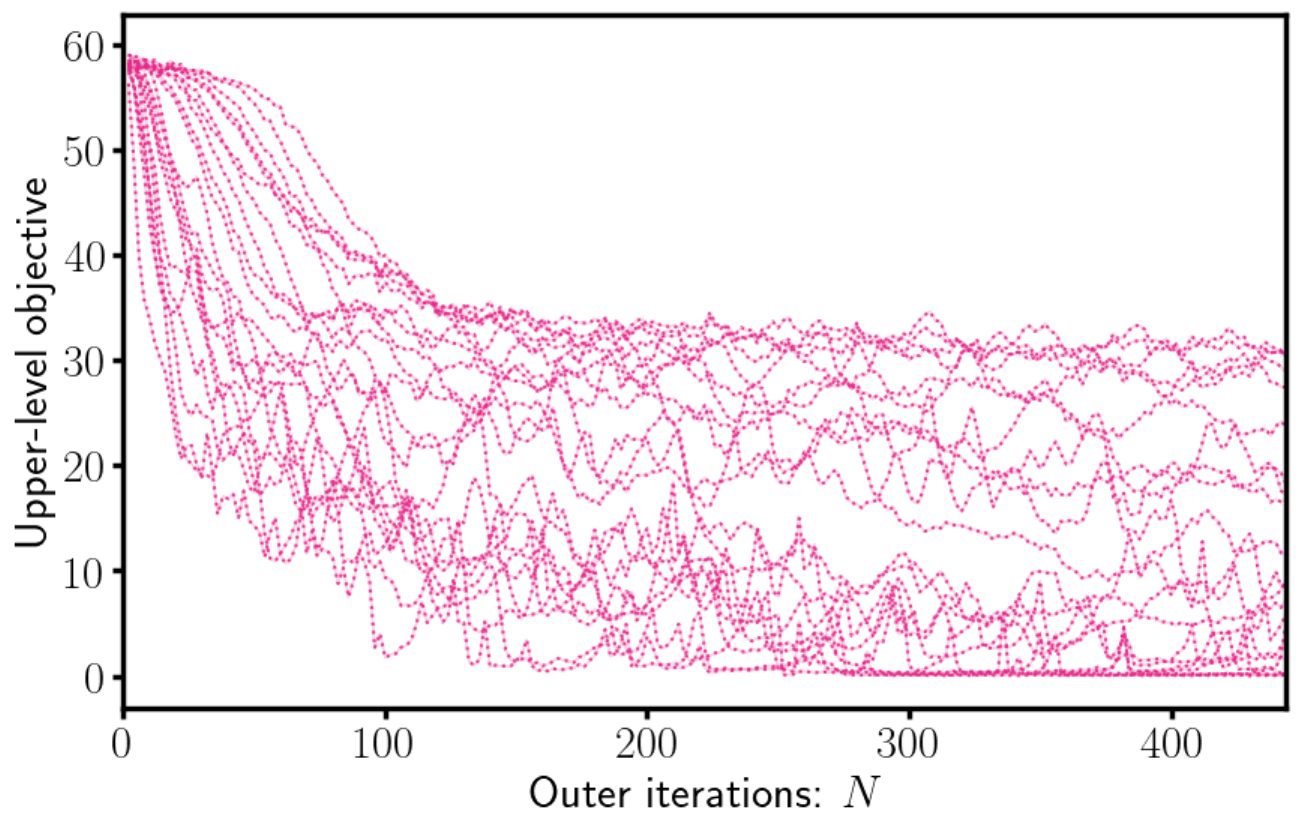}%
        \caption{Bi-AC (On-Policy)}%
    \end{subfigure}%
    \hfill
    \begin{subfigure}[t]{0.33\textwidth}%
        \centering
        \includegraphics[width=\textwidth]{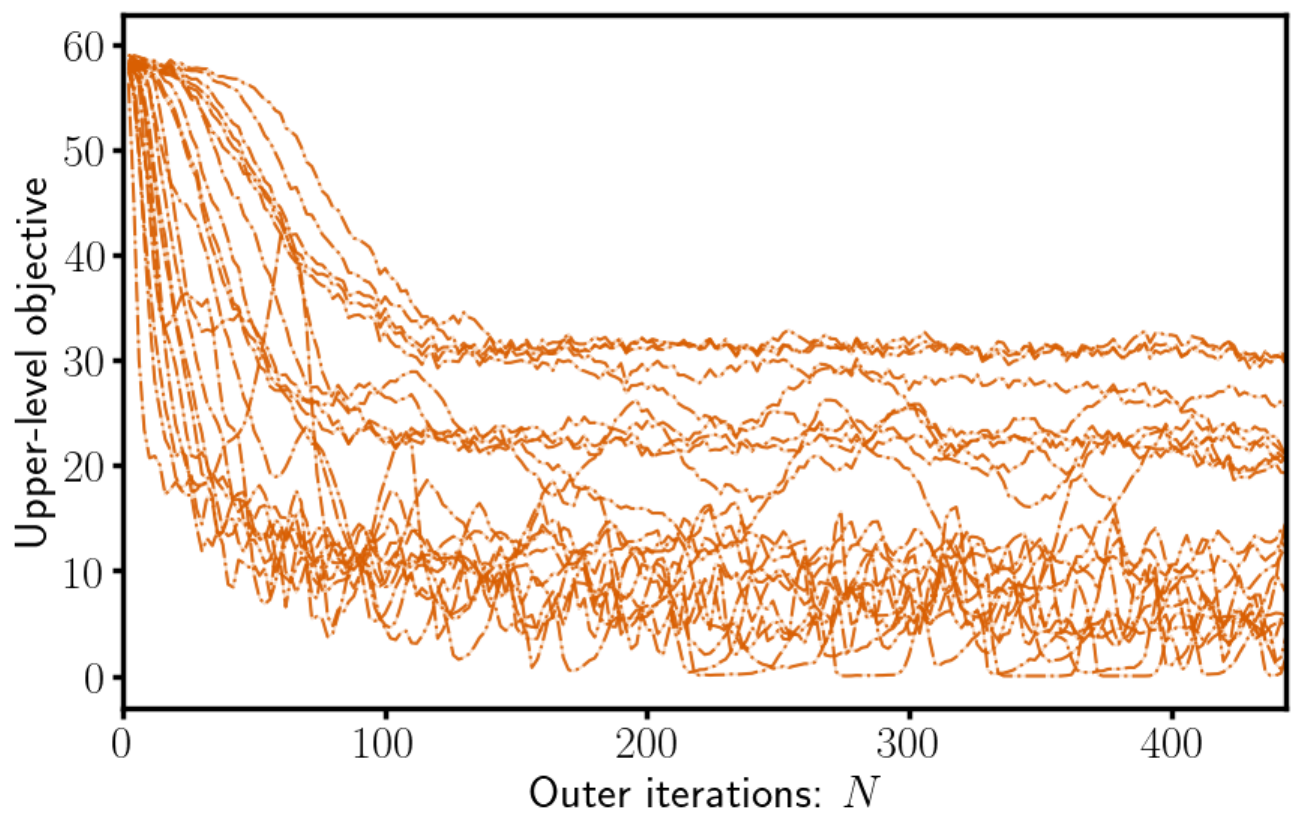}%
        \caption{Bi-AC (Off-Policy)}
    \end{subfigure}%
    \hfill
    \begin{subfigure}[t]{0.33\textwidth}%
        \centering
        \includegraphics[width=\textwidth]{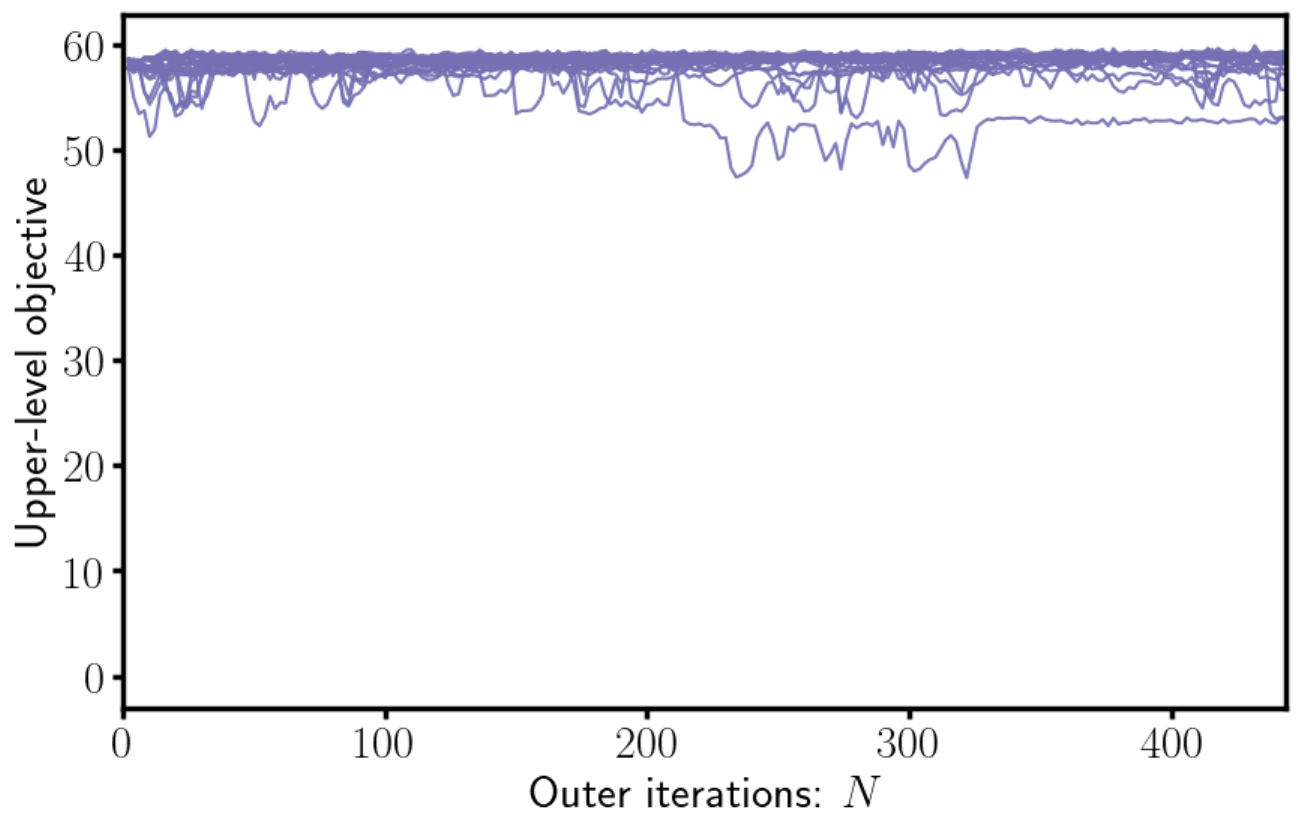}%
        \caption{BC-HG}%
    \end{subfigure}%
    
    \caption{Parameter Sensitivity in the Toy Markov Game Task}
    \label{fig:param_sensitivity_ToyMG}
\end{figure}

\subsection{Entropy-Regularized LQR with Observable and Unobservable Disturbance}\label{apdx:lqr}

We describe the bi-level extension of an infinite-horizon discounted Linear Quadratic Regulator (LQR) task with entropy regularization.
LQRs are widely applied in real-world scenarios due to their tractability.
Let $s \in \R^n$ denote the state and $b \in \R^m$ denote the follower's action. 
The state transition dynamics are governed by the linear equation:
\begin{equation}
    s_{t+1} = A s_{t} + B b_{t} + w_t,  \quad w_t\sim  \mathcal{N}(0,UU^\top), \label{eq:lqr_with_disturbance}
\end{equation}
where $A \in \R^{n\times n}$ is the state transition matrix, $B \in \R^{n \times m}$ is the control input matrix, and $w_t$ represents a disturbance with covariance $W=UU^\top$.
The follower's reward function is defined as the negative of a quadratic cost function:
\begin{equation}
    r_F(s, b) = -\ell(s,b) := - (s^{\top} \bar{Q} s + b^{\top} \bar{R} b),
\end{equation}
where $\bar{Q} \in \R^{n\times n}$ and $\bar{R} \in \R^{m\times m}$ are assumed to be positive semi-definite and positive definite matrices, respectively. 
LQRs are characterized by \emph{linear} state transitions and \emph{quadratic} reward (or cost) functions.

In the configurable MDP setting, the leader's parameter $\theta\in\R^d$ defines the transition dynamics and reward function as follows:
\begin{gather}
    s_{t+1} = A_{\theta} s_{t} + B_{\theta} b_{t} + w_t,  \quad b_{t}\sim g(\cdot\mid s_{t}),\enspace w_t\sim  \mathcal{N}(0,UU^\top), \label{eq:bilevel_lqr_cmdp_transition}\\
    r_F^{\theta}(s, b) = -\ell^{\theta}(s,b) := - (s^{\top} \bar{Q}_{\theta} s + b^{\top} \bar{R}_{\theta} b).
\end{gather}
In contrast, in the 2-player Markov Game setting, the leader's action $a\in\R^k$ indirectly influences the transition:
\begin{gather}
    s_{t+1} = A s_{t} + B b_{t} + C a_{t},\quad b_{t}\sim g(\cdot\mid s_{t},a_{t}),\enspace a_{t}\sim f_{\theta}(\cdot\mid s_{t}), \label{eq:bilevel_lqr_mg_transition}\\
    r_F(s, a, b) = -\ell(s,b) :=  - (s^{\top} \bar{Q} s + b^{\top} \bar{R} b),
\end{gather}
where $f_{\theta}$ denotes the leader's policy and $C\in\R^{n\times k}$.
However, by modeling the leader's policy $f_\theta(\cdot\mid s)$ as a Gaussian distribution $\mathcal{N}(K_\theta s, U_{\theta}U_{\theta}^\top)$, the transition in \eqref{eq:bilevel_lqr_mg_transition} can be reformulated into the structure of \eqref{eq:lqr_with_disturbance} with a disturbance term:
\begin{subequations}
\begin{align}
    s_{t+1} &= A s_{t} + B b_{t} + C a_{t}, &b_{t}\sim g(\cdot\mid s_{t},a_{t}),\enspace a_{t}\sim \mathcal{N}(K_{\theta}s_{t},U_{\theta}U_{\theta}^\top),\\
    &= (A + CK_{\theta}) s_{t} + B b_{t} + C w_{t}, &b_{t}\sim g(\cdot\mid s_{t},w_{t}),\enspace w_{t}\sim\mathcal{N}(0,U_{\theta}U_{\theta}^\top),\\
    &= \tilde{A} s_{t} + B b_{t} + \tilde{w}_{t}, &b_{t}\sim g(\cdot\mid s_{t},\tilde{w}_{t}),\enspace \tilde{w}_{t}\sim\mathcal{N}(0,\tilde{U}_{\theta}\tilde{U}_{\theta}^\top),\label{eq:bilevel_lqr_mg_transition_observable}
\end{align}
\end{subequations}
where $a_t$ is decomposed as $a_t=K_{\theta}s_{t}+w_{t}$, and we define $\tilde{A}:=A+CK_{\theta}$, $\tilde{w}_t:=Cw_{t}$, and $\tilde{U}_{\theta}:=CU_{\theta}$.

A key distinction between configurable MDPs and 2-player MGs in the context of bi-level LQR lies in the observability of the disturbance. In configurable MDPs (Eq.~\eqref{eq:bilevel_lqr_cmdp_transition}), the disturbance $w_{t}$ is unobservable to the follower prior to action selection ($g(\cdot\mid s_{t})$). Conversely, in MGs (Eq.~\eqref{eq:bilevel_lqr_mg_transition_observable}), the effective disturbance $\tilde{w}_{t}$ is observable ($g(\cdot\mid s_{t},\tilde{w}_{t})$).
The follower's best response under entropy regularization can be analytically derived for both settings.
For configurable MDPs, the best response $g^{\theta\dag}(b \mid s)$ follows a Gaussian distribution $\mathcal{N}(- K_s^{\theta} s, (\beta / 2) (S^{\theta})^{-1})$. Here, $K_s^{\theta}:= \gamma_F (S^{\theta})^{-1} B_{\theta}^\top P^{\theta} A_{\theta}$ represents the control gain, $S^{\theta}:= \bar{R}_{\theta} + \gamma_F B_{\theta}^\top P^{\theta} B_{\theta}$, and $P^{\theta} = \text{\textsc{Riccati}}(A_{\theta}, B_{\theta}, \bar{Q}_{\theta}, \bar{R}_{\theta}, \gamma_F)$ is the unique solution to the Riccati equation for an infinite-horizon discounted LQR. 
For conciseness, the Riccati equation, denoted as $\text{\textsc{Riccati}}(A, B, \bar{Q}, \bar{R}, \gamma)$, is defined by:
\begin{equation}
P = \bar{Q} + \gamma A^\top P A - \gamma^2 A^\top P B (\bar{R} + \gamma B^\top P B)^{-1} B^\top P A .\label{eq:riccati_equation}
\end{equation}
In the case of 2-player MGs, the optimal policy $g^{\theta\dag}(b \mid s, a)$ is given by $\mathcal{N}(- K_s s - K_a a, (\beta / 2) S^{-1})$. Here, the control gains are $K_s := \gamma_F S^{-1} B^\top P A$ and $K_a := \gamma_F S^{-1} B^\top P C$, with $S := \bar{R} + \gamma_F B^\top P B$. The matrix $P$ is the unique solution to the Riccati equation $\text{\textsc{Riccati}}(A+C K_\theta, B, \bar{Q}, \bar{R}, \gamma_F)$, where the state transition matrix $A$ is effectively replaced by $A+C K_\theta$.

In the following, we derive the optimal policy and value function for entropy-regularized LQR with unobservable disturbance in Proposition~\ref{thm:opt_f_policy_unobservable_disturbance} and for the case with observable disturbance in Proposition~\ref{thm:opt_f_policy_observable_disturbance}.

\begin{proposition}\label{thm:opt_f_policy_unobservable_disturbance}
    Consider an LQR task with an unobservable disturbance, whose state transition is defined as $s_{t+1} = A s_t + B b_t + w_t$, where $w_t \sim \mathcal{N}(0, UU^\top)$ are independently distributed, and cost function (negative of a reward function) is defined as $\ell(s, b):= s^\top \bar{Q} s + b^\top \bar{R} b$.
    Consider the minimization of the expectation of the infinite-horizon discounted cumulative reward under entropy regularization
    \begin{equation}
        \min_{g} \E\left[\sum_{t=0}^\infty \gamma^t \left(\ell(s_t, b_t) - \beta H(g(\cdot \mid s_t)) \right)\right] ,
    \end{equation}
    where $\gamma \in [0, 1)$ is the discount factor.
    Suppose that $\bar{R} \succ 0$, $\bar{Q} \succcurlyeq 0$, $(A, B)$ stabilizable and $(A, \bar{Q}^{1/2})$ detectable (the usual assumptions for the standard LQR to guarantee the existence of the stabilizing solution of the discounted Riccati equation).
    Then, the optimal policy is given by 
    \begin{equation}
        g(b \mid s) = \mathcal{N}(b \mid - K_s s, (\beta/2) S^{-1}), \label{eq:lqr_opt_follower_cmdp}
    \end{equation}
    where $K_s := \gamma S^{-1} B^\top P A$ is the control gains, $S := \bar{R} + \gamma B^\top P B$, and $P = \text{\textsc{Riccati}}(A, B, \bar{Q}, \bar{R}, \gamma)$. 
    Moreover, the optimal value function is 
    \begin{equation}
        V(s) = s^\top \bar{Q} s + (A s)^\top (\gamma P - \gamma^2 P B S^{-1} B^\top P) (A s) 
        + \frac{1}{1 - \gamma} \left( - \frac{\beta m}{2} \log(\pi \beta) +  \frac{\beta}{2} \log \det(S) + \gamma \Tr(U^\top P U) \right).
    \end{equation}
\end{proposition}

\begin{proof}
    Let $V(s)$ denote the optimal value function.
    The Bellman optimality equation is 
    \begin{equation}
    V(s) = \min_{g} \E_{b}\left[ \ell(s, b) + \beta \log g(b \mid s) + \gamma \E_{w}\left[ V(As+Bb+w)\right]\right],
    \end{equation}
    where $w \sim \mathcal{N}(0, UU^\top)$ represents the disturbance.
    We assume that the optimal value function is a quadratic function
    \begin{equation}
    V(s) = s^\top P_{ss} s + v ,
    \end{equation}
    where $P_{ss} \succcurlyeq 0$ is symmetric. 
    
    Define the action value function as
    \begin{equation}
    Q(s, b) := \ell(s, b) + \gamma \E_{w}[V(As + Bb + w)].
    \end{equation}
    Under the quadratic ansatz and $\E[w] = 0$ and $\Cov[w] = UU^\top$, we obtain
    \begin{equation}
    \E_{w}[V(As + Bb + w)] = (A s + Bb)^\top P_{ss} (A s + Bb) + \Tr(P_{ss} UU^\top) + v .
    \end{equation}
    Thus, by grouping quadratic and linear terms in $b$, we can write
    \begin{equation}
    Q(s, b) = b^\top S b + 2 b^\top \phi(s) + \xi(s),
    \end{equation}
    where
    \begin{align}
    S &:= \bar{R} + \gamma B^\top P_{ss} B,\\
    \phi(s) &:= \gamma B^\top P_{ss} A s,\\
    \xi(s) &:= s^\top \bar{Q} s + \gamma ((A s)^\top P_{ss} (A s) + \Tr(P_{ss} UU^\top) + v).
    \end{align}
    Because $\bar{R} \succ 0$ and $P_{ss} \succcurlyeq 0$, $S \succ 0$ is guaranteed.
    
    The optimal policy is then derived. With the derived Q-function $Q(s, b)$, the objective is rewritten as
    \begin{equation}
    \E_{s_0}\left[  \E_{b_0} \left[ Q(s_0, b_0) - \beta H(g(\cdot \mid s_0)) \mid s_0\right] \right].
    \end{equation}
    Therefore, its optimal policy is known to be given by the Boltzmann distribution, $g(b \mid s) \propto \exp( - Q(s, b) / \beta)$. Because $S$ is positive definite, the Q-function $Q$ is convex quadratic with respect to $b$. Therefore, by rearranging the terms of $Q$ as
    \begin{equation}
    Q(s, b) = (b + S^{-1} \phi(s))^\top S (b + S^{-1} \phi(s)) - \phi(s)^\top S^{-1} \phi(s) + \xi(s),
    \end{equation}
    we obtain the Gaussian distribution
    \begin{equation}
    g(b \mid s) = \mathcal{N}(b \mid - S^{-1} \phi(s), (\beta / 2) S^{-1}).
    \end{equation}
    
    We show that the quadratic ansatz satisfies the Bellman optimality equation for a specific $P_ss$, $P_{sw}$, $P_{ww}$, and $v$. 
    Because of the variational identity, for the optimal value function, we have
    \begin{equation}
    V(s) = - \beta \log \left( \int \exp \left( - \frac{1}{\beta} Q(s, b) \right) \rmd b\right) .
    \end{equation}
    Evaluating the right-hand side, we obtain 
    \begin{equation}
    V(s) = \xi(s) - \phi(s)^\top S^{-1} \phi(s) - \frac{\beta m}{2} \log(\pi \beta) + \frac{\beta}{2} \log \det S . 
    \end{equation}
    By comparing both sides of the above equality, we obtain
    \begin{align}
    P_{ss} &= \bar{Q} + \gamma A^\top P_{ss} A - \gamma^2 A^\top P_{ss} B S^{-1} B^\top P_{ss} A,\\
    v &= \frac{1}{1 - \gamma} \left( - \frac{\beta m}{2} \log(\pi \beta) +  \frac{\beta}{2} \log \det(S) + \gamma \Tr(P_{ss} UU^\top) \right).
    \end{align}
    Note that the equation for $P$ is identical to the Riccati equation. Under the stated assumptions ($\bar{R} \succ 0$, $\bar{Q} \succcurlyeq 0$, $(A, B)$ stabilizable and $(A, \bar{Q}^{1/2})$ detectable), the Riccati equation has a unique symmetric positive semidefinite stabilizing solution $P$. Therefore, $P_{ss} = \text{\textsc{Riccati}}(A, B, \bar{Q}, \bar{R}, \gamma)$. All the other terms are then identified. The rest follows the uniqueness of the fixed point of the Bellman optimality equation. 
\end{proof}

\begin{proposition}\label{thm:opt_f_policy_observable_disturbance}
    Consider an LQR task with an observable disturbance, whose state transition is defined as $s_{t+1} = A s_t + B b_t + w_t$, where $w_t \sim \mathcal{N}(0, UU^\top)$ are independently distributed, and cost function (negative of a reward function) is defined as $\ell(s, b):= s^\top \bar{Q} s + b^\top \bar{R} b$.
    Consider the minimization of the expectation of the infinite-horizon discounted cumulative reward under entropy regularization
    \begin{equation}
        \min_{g} \E\left[\sum_{t=0}^\infty \gamma^t \left(\ell(s_t, b_t) - \beta H(g(\cdot \mid s_t, w_t)) \right)\right] ,
    \end{equation}
    where the disturbance is observable before the action and $\gamma \in [0, 1)$ is the discount factor.
    Suppose that $\bar{R} \succ 0$, $\bar{Q} \succcurlyeq 0$, $(A, B)$ stabilizable and $(A, \bar{Q}^{1/2})$ detectable (the usual assumptions for the standard LQR to guarantee the existence of the stabilizing solution of the discounted Riccati equation).
    Then, the optimal policy is given by 
    \begin{equation}
        g(b \mid s, w) = \mathcal{N}(b \mid - K_s s - K_w w, (\beta/2) S^{-1}), \label{eq:lqr_opt_follower_mg}
    \end{equation}
    where $K_s := \gamma S^{-1} B^\top P A$ and $K_w := \gamma S^{-1} B^\top P$ are the control gains, $S := \bar{R} + \gamma B^\top P B$, and $P = \text{\textsc{Riccati}}(A, B, \bar{Q}, \bar{R}, \gamma)$. 
    Moreover, the optimal value function is 
    \begin{equation}
        V(s, w) = s^\top \bar{Q} s + (A s + w)^\top (\gamma P - \gamma^2 P B S^{-1} B^\top P) (A s + w) 
        + \frac{1}{1 - \gamma} \left( - \frac{\beta m}{2} \log(\pi \beta) +  \frac{\beta}{2} \log \det(S) + \gamma \Tr(U^\top P_{ww}U) \right).
    \end{equation}
\end{proposition}

In our situation, we replace $A$ and $w$ with $A + C K_\theta$ and $C (a - K_\theta s)$, respectively. 
Then, the mean of the best response action is 
\begin{subequations}
\begin{align}
    - K_s s - K_w w 
    &= - \gamma S^{-1} B^\top P (A + C K_\theta) s - \gamma S^{-1} B^\top P C (a - K_\theta s)\\
    &= - \gamma S^{-1} B^\top P A s - \gamma S^{-1} B^\top P C a\\
    &= - K_s s - K_a a .
\end{align}
\end{subequations}

\begin{proof}
    Let $V(s, w)$ denote the optimal value function.
    The Bellman optimality equation is 
    \begin{equation}
    V(s, w) = \min_{g} \E_{b}\left[ \ell(s, b) + \beta \log g(b \mid s, w) + \gamma \E_{w'}\left[ V(As+Bb+w, w')\right]\right],
    \end{equation}
    where $w' \sim \mathcal{N}(0, UU^\top)$ represents the next disturbance.
    We assume that the optimal value function is a quadratic function
    \begin{equation}
    V(s, w) = \begin{bmatrix} s\\ w \end{bmatrix}^\top \begin{bmatrix} P_{ss} & P_{sw}\\ P_{sw}^\top & P_{ww} \end{bmatrix} \begin{bmatrix} s\\ w \end{bmatrix} + v ,
    \end{equation}
    where $P_{ss} \succcurlyeq 0$ and $P_{ww} \succcurlyeq 0$ are symmetric. 
    
    Define the Q-function as
    \begin{equation}
    Q(s, w, b) := \ell(s, b) + \gamma \E_{w'}[V(As + Bb + w, w')].
    \end{equation}
    Under the quadratic ansatz and $\E[w'] = 0$ and $\Cov[w'] = UU^\top$, we obtain
    \begin{equation}
    \E_{w'}[V(As + Bb + w, w')] = (A s + Bb + w)^\top P_{ss} (A s + Bb + w) + \Tr(P_{ww} UU^\top) + v .
    \end{equation}
    Thus, by grouping quadratic and linear terms in $b$, we can write
    \begin{equation}
    Q(s, w, b) = b^\top S b + 2 b^\top \phi(s, w) + \xi(s, w),
    \end{equation}
    where
    \begin{align}
    S &:= \bar{R} + \gamma B^\top P_{ss} B,\\
    \phi(s, w) &:= \gamma B^\top P_{ss} (A s + w),\\
    \xi(s, w) &:= s^\top \bar{Q} s + \gamma ((A s + w)^\top P_{ss} (A s + w) + \Tr(P_{ww} UU^\top) + v).
    \end{align}
    Because $\bar{R} \succ 0$ and $P_{ss} \succcurlyeq 0$, $S \succ 0$ is guaranteed.
    
    The optimal policy is then derived. With the derived Q-function $Q(s, w, b)$, the objective is rewritten as
    \begin{equation}
    \E_{s_0,w_0}\left[  \E_{b_0} \left[ Q(s_0, w_0, b_0) - \beta H(g(\cdot \mid s_0, w_0)) \mid s_0, w_0\right] \right].
    \end{equation}
    Therefore, its optimal policy is known to be given by the Boltzmann distribution, $g(b \mid s, w) \propto \exp( - Q(s, w, b) / \beta)$. Because $S$ is positive definite, the Q-function $Q$ is convex quadratic in $b$. Therefore, by rearranging the terms of $Q$ as
    \begin{equation}
    Q(s, w, b) = (b + S^{-1} \phi(s, w))^\top S (b + S^{-1} \phi(s, w)) - \phi(s, w)^\top S^{-1} \phi(s, w) + \xi(s, w),
    \end{equation}
    we obtain the Gaussian distribution
    \begin{equation}
    g(b \mid s, w) = \mathcal{N}(b \mid - S^{-1} \phi(s, w), (\beta / 2) S^{-1}).
    \end{equation}
    
    We show that the quadratic ansatz satisfies the Bellman optimality equation for a specific $P_ss$, $P_{sw}$, $P_{ww}$, and $v$. 
    Because of the variational identity, for the optimal value function, we have
    \begin{equation}
    V(s, w) = - \beta \log \left( \int \exp \left( - \frac{1}{\beta} Q(s, w, b) \right) \rmd b\right) .
    \end{equation}
    Evaluating the right-hand side, we obtain 
    \begin{equation}
    V(s, w) = \xi(s, w) - \phi(s, w)^\top S^{-1} \phi(s, w) - \frac{\beta m}{2} \log(\pi \beta) + \frac{\beta}{2} \log \det S . 
    \end{equation}
    By comparing both sides of the above equality, we obtain
    \begin{align}
    P_{ss} &= \bar{Q} + \gamma A^\top P_{ss} A - \gamma^2 A^\top P_{ss} B S^{-1} B^\top P_{ss} A\\
    P_{sw} &= \gamma A^\top P_{ss} - \gamma^2 A^\top P_{ss} B S^{-1} B^\top P_{ss}\\
    P_{ww} &= \gamma P_{ss} - \gamma^2 P_{ss} B S^{-1} B^\top P_{ss}\\
    v &= \frac{1}{1 - \gamma} \left( - \frac{\beta m}{2} \log(\pi \beta) +  \frac{\beta}{2} \log \det(S) + \gamma \Tr(P_{ww} UU^\top) \right).
    \end{align}
    Note that the equation for $P$ is identical to the Riccati equation. Under the stated assumptions ($\bar{R} \succ 0$, $\bar{Q} \succcurlyeq 0$, $(A, B)$ stabilizable and $(A, \bar{Q}^{1/2})$ detectable), the Riccati equation has a unique symmetric positive semidefinite stabilizing solution $P$. Therefore, $P_{ss} = \text{\textsc{Riccati}}(A, B, \bar{Q}, \bar{R}, \gamma)$. All the other terms are then identified. The rest follows the uniqueness of the fixed point of the Bellman optimality equation. 
\end{proof}

\subsection{$n$-Zone Building Thermal Control Task
}\label{apdx:building_thermal_control}

\paragraph{Task Description}
Let $s_t\in\mathbb{R}^n$ denote the temperature deviations (in~$^\circ$C)
from nominal setpoints in the $n$ zones, and
$b_t\in\mathbb{R}^m$ the actuation levels of $m$ HVAC units.
The discrete-time thermal dynamics over a sampling period (e.g., 5--10~minutes)
are modeled as
\begin{equation}
    s_{t+1} = A_{\theta}s_t + Bb_t + w_t,
    \qquad w_t\sim\mathcal{N}(0,UU^\top),
\end{equation}
where $A_{\theta}$ depends on the leader's parameters
$\theta=(\alpha,\beta)$, with $\alpha\in[0,1]^n$ and $\beta\in[0,1]^n$. Here, $\alpha$, referred to as the ``insulation level'', represents the improvement in insulation for each zone, while
$\beta$, referred to as the ``airflow level'', represents the adjustment of airflow coupling between adjacent zones.
In our experiments, the numbers of zones and HVAC units are set to $n=4$ and $m=2$, respectively.
Assuming $b_t=0$ and $w_t=0$, the evolution of temperature deviations from $s_t=(s_t^1,s_t^2,s_t^3,s_t^4)$ to $s_{t+1}$ is defined as follows, using positive coefficients $k_i$ and $h_{ij}$ $(i,j\in\{1,2,3,4\})$:
\begin{align}
    &s_{t+1}^1 = s_t^1 + k_1(1-\alpha_1)s_t^1 + h_{12}\beta_1(s_t^2 - s_t^1) + h_{41}\beta_4(s_t^4 - s_t^1),\\
    &s_{t+1}^2 = s_t^2 + k_2(1-\alpha_2)s_t^2 + h_{23}\beta_2(s_t^3 - s_t^2) + h_{21}\beta_1(s_t^1 - s_t^2),\\
    &s_{t+1}^3 = s_t^3 + k_3(1-\alpha_3)s_t^3 + h_{34}\beta_3(s_t^4 - s_t^3) + h_{32}\beta_2(s_t^2 - s_t^3),\\
    &s_{t+1}^4 = s_t^4 + k_4(1-\alpha_4)s_t^4 + h_{41}\beta_4(s_t^1 - s_t^4) + h_{43}\beta_3(s_t^3 - s_t^4),
\end{align}
where $\alpha$ modulates the temperature change due to heat exchange with the external environment, and $\beta$ controls the heat exchange with adjacent zones.
Specifically, a higher insulation level $\alpha$ mitigates the amplification of temperature deviations caused by external temperatures, whereas a higher airflow level $\beta$ facilitates thermal exchange between neighboring zones.
Accordingly, $A_{\theta}$ is defined as:


\begin{small}
\begin{multline}
    A_{\theta}  := 
    \begin{bmatrix}
        1 + \Delta(\alpha_1) - h_{12}\beta_1 - h_{41}\beta_4 & h_{12}\beta_1 & 0 & h_{41}\beta_4\\
        h_{21}\beta_1 & 1 + \Delta(\alpha_2) - h_{23}\beta_2 - h_{21}\beta_1 & h_{23}\beta_2 & 0\\
        0     & h_{32}\beta_2 & 1 + \Delta(\alpha_3) - h_{34}\beta_3 - h_{32}\beta_2 & h_{34}\beta_3\\
        h_{41}\beta_4 & 0     & h_{43}\beta_3 & 1 + \Delta(\alpha_4) - h_{41}\beta_4 - h_{43}\beta_3 
    \end{bmatrix},
\end{multline}
\end{small}
where $\Delta(\alpha_i):=k_i(1-\alpha_i)$ for $i\in\{1,2,3,4\}$.

In our experiments, the coefficients are set as follows:
\begin{gather}
    k_1=0.04,\quad h_{12}=0.05,\quad h_{41}=0.05,\\
    k_2=0.03,\quad h_{23}=0.04,\quad h_{21}=0.03,\\
    k_3=0.06,\quad h_{34}=0.06,\quad h_{32}=0.04,\\
    k_4=0.05,\quad h_{41}=0.05,\quad h_{43}=0.03.
\end{gather}
The other parameters governing the transition and the follower's reward function are set as:
\begin{equation}
B =
\begin{bmatrix}
0.10 & 0\\
0.6  & 0\\
0    & 0.55\\
0    & 0.30
\end{bmatrix},
\quad U =
\begin{bmatrix}
0.02 & 0 & 0 & 0\\
0 & 0.02 & 0 & 0\\
0 & 0 & 0.02 & 0\\
0 & 0 & 0 & 0.02
\end{bmatrix},
\quad \bar{Q} =
\begin{bmatrix}
8 & 0 & 0 & 0\\
0 & 1 & 0 & 0\\
0 & 0 & 5 & 0\\
0 & 0 & 0 & 6
\end{bmatrix},
\quad \bar{R} =
\begin{bmatrix}
0.01 & 0\\
0  & 0.01
\end{bmatrix}.
\end{equation}

The leader's reward function is defined as
\begin{gather}
    r_L(s_t,b_t):= \mathrm{Stability}(s_t) -0.5\times\mathrm{HVACEnergyCost}(b_t) -0.1\times\mathrm{ConfigurationCost}(\theta),
\end{gather}
where
\begin{align}
    \mathrm{Stability}(s_t):=-\dfrac{1}{n}\sum_{i=1}^n (s_t^i - \dfrac{1}{n}\sum_{j=1}^n s_t^j)^2, \quad
    \mathrm{HVACEnergyCost}(b_t):=||b_t||_2^2, \quad
    \mathrm{ConfigurationCost}(\theta):=||\alpha||_2^2+||\beta||_2^2.
\end{align}
The leader's objective is to maximize the discounted cumulative reward $r_L$, which corresponds to minimizing the variance of the temperature deviations across zones and the cost of the follower's actions, while also managing the configuration cost.
In our experiments, with an episode length of 100 steps, the leader's cumulative reward typically ranges from approximately -600 to -150 (see Figure~\ref{fig:BTC_UL_return_histogram}).

\begin{figure}[t]
    \centering
    \includegraphics[width=0.45\linewidth]{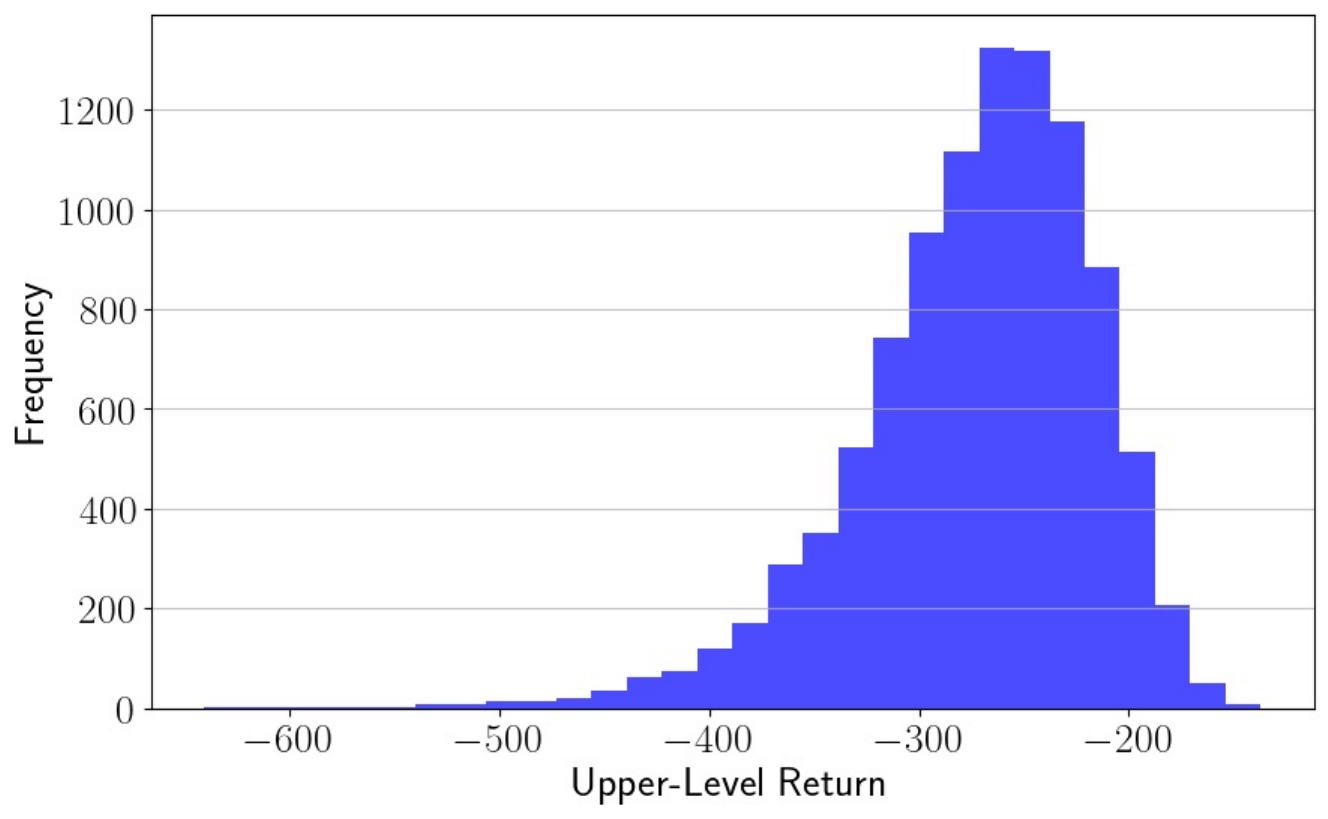}%
    \caption{Histogram of the Leader's Objective Value in the Building Thermal Control Task. This plots the $10^4$ uniformly random leader parameters according to each leader's discounted cumulative reward under the best response.}%
    \label{fig:BTC_UL_return_histogram}%
\end{figure}

\paragraph{Experimental Settings}
The leader’s parameters are initialized by applying the sigmoid function to standard normal noise: $\theta^{0}\gets\sigma(\phi_L)$, where $\phi_L\sim \mathcal{N}(\bm{0}, I)$.
The follower’s best response is obtained via \eqref{eq:lqr_opt_follower_cmdp}, using the solution to the Riccati equation $P=\text{\textsc{Riccati}}(A_{\theta}, B, \bar{Q}, \bar{R}, \gamma_F)$, which is computed using the following Riccati iteration:
\begin{equation}
    P_{k+1} = \bar{Q} + \gamma_F A_{\theta}^\top P_k A_{\theta} - \gamma_F^2 A_{\theta}^\top P_k B (\bar{R} + \gamma_F B^\top P_k B)^{-1} B^\top P_k A_{\theta}.
\end{equation}

Following the computation of the best response, 100 trajectories are sampled, each consisting of 100 steps.
The initial state is sampled randomly as $s_0\sim \mathcal{N}(\bm{0}, I\times 5^2)$.  
All function approximators, including the leader’s critic and the follower's value gradient estimator, are implemented as neural networks with two hidden layers, each containing 64 units.
These networks are initialized at each outer iteration and updated $5\times 10^3$ times by minimizing TD errors using the sampled trajectories, following the implementation of \textsc{HPGD} by \citet{Thoma2024-em}.
For computing the Benefit in \textsc{BC-HG}, the leader's value function is estimated by averaging the leader's Q-values over 256 sampled follower actions.  
The estimated gradients of the leader's objective are clipped to a maximum norm of 1.0.

For each approach, the learning rate of the leader's parameter is selected via grid search from $\{10^{-1},10^{-2},10^{-3}\}$.
For each hyperparameter configuration, the leader is trained using 10 random seeds.  
Each learned policy is evaluated by computing the leader’s return averaged over 50 rollouts. 
Table~\ref{tab:params_lqr_cmdp} summarizes the hyperparameters that are shared across all methods.

\begin{table}[t]
    \centering
    \begin{tabular}{rll}\toprule
        follower's entropy regularization $\beta$ & $1\times10^{-1}$\\
        batch size per outer iteration & $10^4$\\
        learning rate for function approximator & $1\times10^{-4}$\\
        leader’s discount rate $\gamma_L$ & 0.9\\
        follower’s discount rate $\gamma_F$ & 0.9\\
        maximum Riccati iterations & $5\times 10^4$\\ \bottomrule
    \end{tabular}
    \caption{Hyperparameters in Building Thermal Control}
    \label{tab:params_lqr_cmdp}
\end{table}

\paragraph{Results}
Figure~\ref{fig:results_all_params_BTC} presents the results of the $n$-zone Building Thermal Control task across different actor learning rates $\{10^{-1}, 10^{-2}, 10^{-3}\}$. The figure displays the leader's objective (discounted cumulative leader reward) along with its four constituent components: Stability, HVAC energy cost, insulation level cost, and airflow level cost. Each component is reported as a discounted cumulative value. The plots illustrate the mean and standard deviation over 8 training runs, where the maximum and minimum values from 10 independent runs were excluded at each outer iteration to mitigate the impact of outliers.

The proposed method consistently outperformed all baselines at learning rates of $10^{-1}$ and $10^{-2}$, converging to a leader's cumulative reward of approximately -200. A learning rate of $10^{-3}$ appeared insufficient, resulting in negligible improvements across all methods. Notably, with a learning rate of $10^{-1}$, \textsc{HPGD (MC/TD)} primarily reduced the configuration cost but failed to achieve high stability or low energy costs. A similar trend was observed for \textsc{HPGD (MC)} across other learning rates.

\begin{figure}[H]
    \centering

    \begin{subfigure}[t]{\textwidth}%
        \centering
        \includegraphics[width=\textwidth,trim=0 0 0 40,clip]{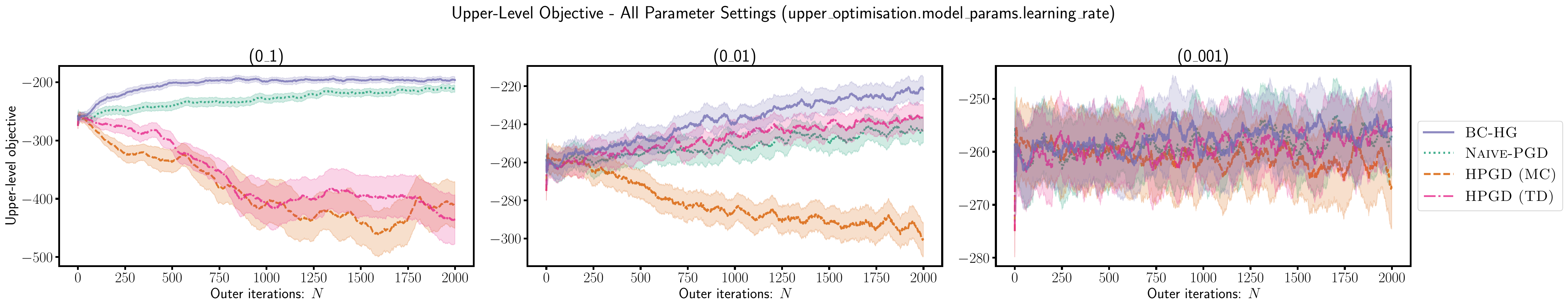}%
        \caption{Leader's Objective (Discounted Cumulative Leader Reward)}%
        \label{fig:results_all_params_BTC_return_UL}%
    \end{subfigure}%
    \\
    \begin{subfigure}[t]{\textwidth}%
        \centering
        \includegraphics[width=\textwidth,trim=0 0 0 40,clip]{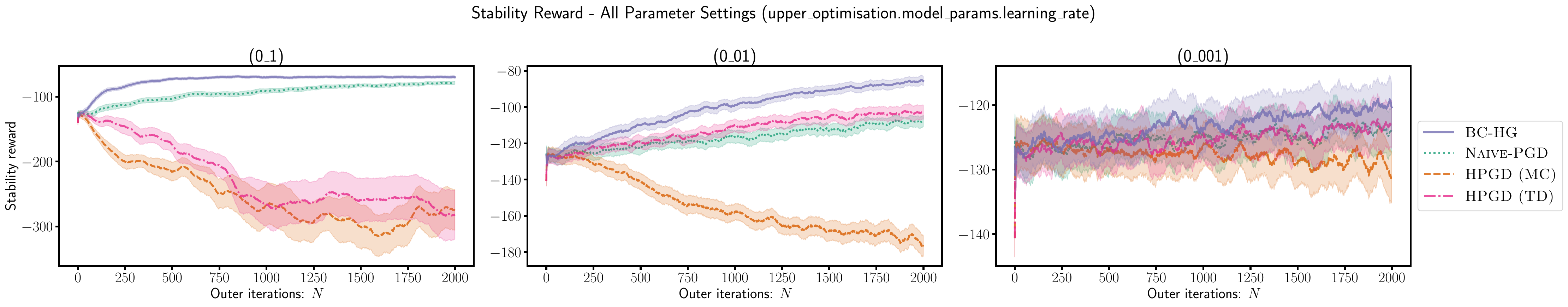}%
        \caption{Stability}%
        \label{fig:results_all_params_BTC_stability}%
    \end{subfigure}%
    \\
    \begin{subfigure}[t]{\textwidth}%
        \centering
        \includegraphics[width=\textwidth,trim=0 0 0 40,clip]{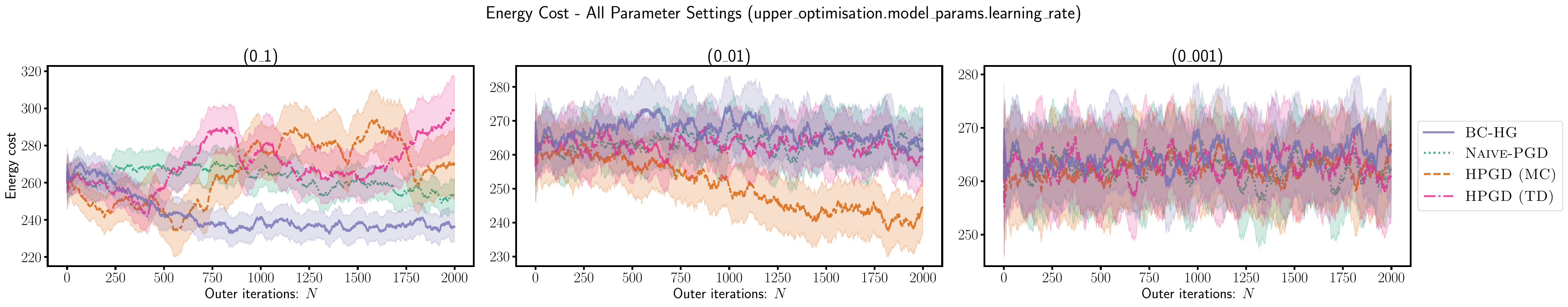}%
        \caption{HVAC Energy Cost}%
        \label{fig:results_all_params_BTC_energy_cost}%
    \end{subfigure}%
    \\
    \begin{subfigure}[t]{\textwidth}%
        \centering
        \includegraphics[width=\textwidth,trim=0 0 0 40,clip]{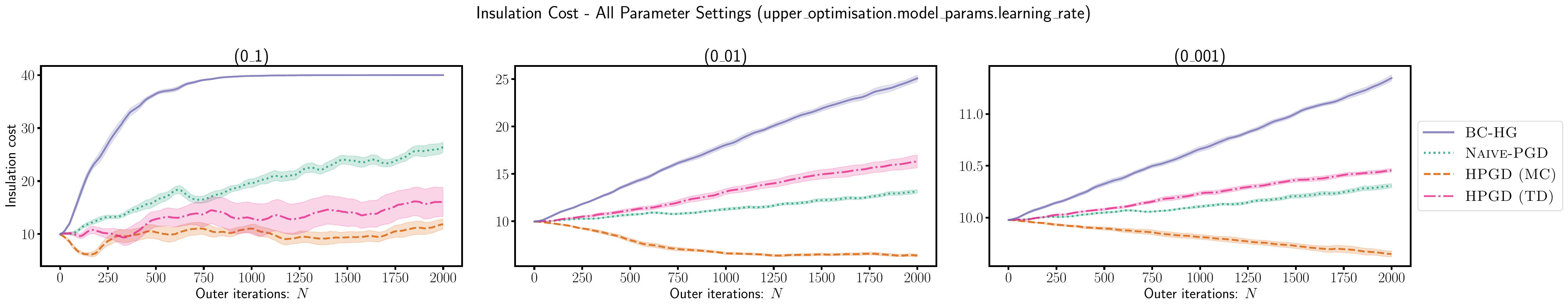}%
        \caption{Insulation Level $\alpha$ Cost}%
        \label{fig:results_all_params_BTC_IL_cost}%
    \end{subfigure}%
    \\
    \begin{subfigure}[t]{\textwidth}%
        \centering
        \includegraphics[width=\textwidth,trim=0 0 0 40,clip]{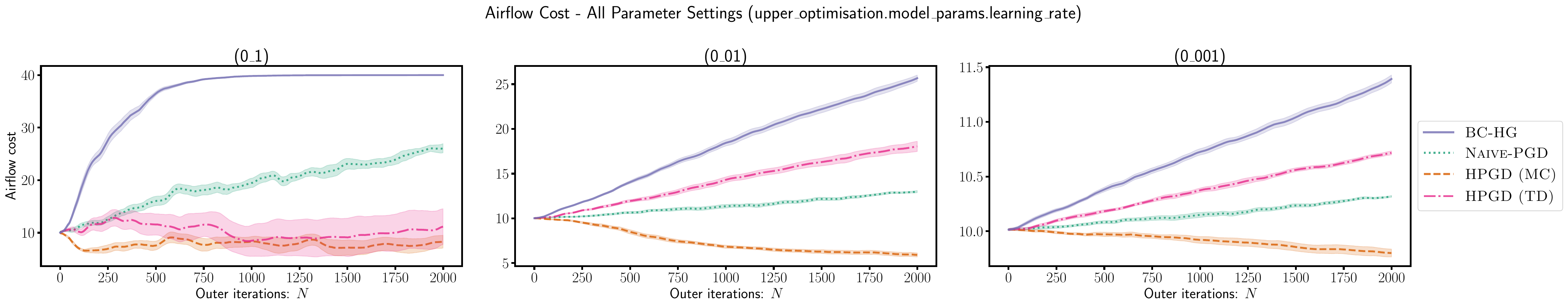}%
        \caption{Airflow Level $\beta$ Cost}%
        \label{fig:results_all_params_BTC_AA_cost}%
    \end{subfigure}%

    \caption{Results of $n$-Zone Building Thermal Control Task on Each Hyperparameter (Learning Rate in $\{10^{-1}, 10^{-2}, 10^{-3}\}$). The top figures show the leader's discounted cumulative reward, and the others show the discounted cumulative values of the four components of the leader's reward. The return and stability are aimed to be maximized; the costs are aimed to be minimized.}%
    \label{fig:results_all_params_BTC}
\end{figure}

\subsection{Bi-Level LQR Task in 2-Player Markov Games}\label{apdx:bilevel_lqr_task_mg}

\begin{figure}[t]
    \centering
    \includegraphics[width=0.45\hsize]{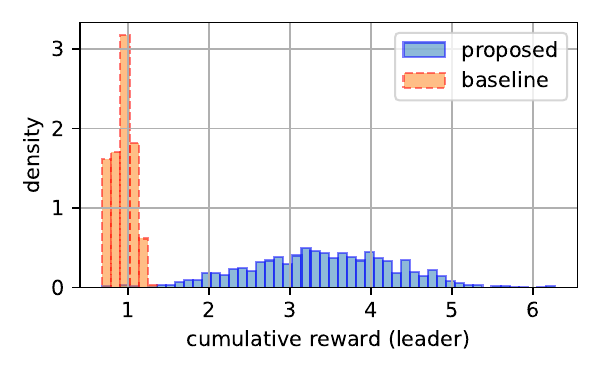}%
    \caption{Histogram of the Cumulative Rewards of the Leader's Policies with a Fixed Initial State on the Bi-Level LQR Task (\textsf{proposed}: \textsc{BC-HG}, \textsf{baseline}: \textsc{Naive-PGD (on-policy)}). The best policies in terms of the average performance for the proposed and baseline approaches are taken.}
    \label{fig:lqr_hist}
\end{figure}

The parameters for the state transition and the follower’s reward function are set as follows:
\begin{equation}
    A=\begin{bmatrix}
        0 & 1\\
        0 & 0
    \end{bmatrix},
    \enspace
    B=\begin{bmatrix}
        0 \\
        0.1 
    \end{bmatrix},
    \enspace
    C=\begin{bmatrix}
        0.1 \\
        0 
    \end{bmatrix},
    \enspace
    \bar{Q}=\begin{bmatrix}
        1 & 0\\
        0 & 1
    \end{bmatrix},
    \enspace
    \bar{R}=\begin{bmatrix}
        1
    \end{bmatrix}.
\end{equation}
The initial and target states are set to $s_0 = (1, 0)$ and $s^\star = (-0.8, 0)$, respectively.  
The leader's objective is to induce oscillatory behavior in the follower, alternating between the initial state $(1, 0)$ and the target state $(-0.8, 0)$, before the follower eventually stabilizes at the origin.  
The weighting matrix for the leader’s reward distribution is set to $\Sigma = \left( \begin{bmatrix}
    5 & 0\\
    0 & 100
\end{bmatrix} \right)^2$.

The leader’s policy is parameterized as a Gaussian distribution $f_{\theta}(a \mid s) = \mathcal{N}(K_{\theta} s, UU^\top)$,  
where matrix $U$ is fixed at $[10^{-3}]$ throughout the experiments. 
The mean is determined by $K_{\theta} = (\theta^1, \theta^2)$, with parameters initialized randomly from a standard normal distribution, $(\theta^1, \theta^2) \sim \mathcal{N}(\bm{0}, I)$.
The follower’s best response is derived using \eqref{eq:lqr_opt_follower_mg}, based on the solution $P=\text{\textsc{Riccati}}(A+C K_{\theta}, B, \bar{Q}, \bar{R}, \gamma_F)$ of the Riccati equation. This solution is computed via the following Riccati iteration:
\begin{equation}
    P_{k+1} = \bar{Q} + \gamma_F A'^\top P_k A' - \gamma_F^2 A'^\top P_k B (\bar{R} + \gamma_F B^\top P_k B)^{-1} B^\top P_k A',
\end{equation}
where $A':=A+C K_{\theta}$.

Following the computation of the best response, four trajectories are sampled and stored in a trajectory buffer.  
Each episode has a duration of 100 steps. 
The buffer size is set to 400 steps for the on-policy methods (\textsc{BC-HG} and \textsc{Naive-PGD (on-policy)}) and $10^6$ steps for the off-policy method (\textsc{Naive-PGD (off-policy)}).
For the initial $4 \times 10^3$ steps, the leader executes actions selected uniformly at random to facilitate exploration.

For each approach, the number of actor updates per outer iteration was selected via grid search from the set $\{1, 2, 5, 10\}$, while the number of critic updates was fixed at 5.
The estimated gradients were clipped to a maximum norm of 1.0.  
The leader’s critic was implemented using a neural network with two hidden layers, each consisting of 64 units.

For each hyperparameter configuration, the leader was trained using 20 random seeds.  
The performance of each learned policy was evaluated by averaging the leader’s return over 10 rollouts.  
Figure~\ref{fig:lqr} presents the mean and standard deviation of the leader’s return across these 20 training runs.

Table~\ref{tab:params_lqr_mg} summarizes the hyperparameters common to all methods.

\begin{table}[t]
    \centering
    \begin{tabular}{rll}\toprule
        follower's entropy regularization $\beta$ & $1\times10^{-1}$\\
        mini-batch size & 64\\
        actor/critic learning rate & $1\times10^{-1}$\\
        leader’s discount rate $\gamma_L$ & 0.99\\
        follower’s discount rate $\gamma_F$ & 0.95\\
        maximum Riccati iterations & $10^3$\\
        target network smoothing factor & $1\times10^{-2}$\\ \bottomrule
    \end{tabular}
    \caption{Hyperparameters in Bi-Level LQR in Markov Games}
    \label{tab:params_lqr_mg}
\end{table}

\section{Hypergradient for Markov Games with Deterministic Leader Policies}

In this section, as a supplementary result not included in the accepted ICAPS 2026 paper~\citep{kudo2026-icaps}, we present the hypergradient for 2-player MGs with deterministic leader policies.

In high-dimensional continuous control tasks in single-agent RL, deterministic policies are often employed for efficient optimization.
The following theorem provides the hypergradient for the 2-player MG setting when the leader policy is restricted to deterministic policies.

\begin{theorem}[Hypergradient for MGs with deterministic leader policies]
\label{thm:hg-mg-deterministic}
Consider the 2-player MG setting, and suppose that the leader policy is deterministic, i.e., $f_{\theta}:\mathcal{S}\to\mathcal{A}$.
Then, the hypergradient of $J_L$ with respect to $\theta$ is given by
\begin{equation}
    \nabla_{\theta}J_L({\theta}, g^{\theta\dagger})
    =\dfrac{1}{1-\gamma_L}\E\bigg[\left.\nabla_a Q_L^{{\theta}\dagger}(s,a,b)\right|_{a=f_{\theta}(s)}\nabla_{\theta}f_{\theta}(s)
    +\dfrac{1}{\beta} B_L^{{\theta}\dagger}(s,f_{\theta}(s),b) \nabla_{\theta}Q_F^{\theta\dagger}(s,f_{\theta}(s),b)\bigg]\label{eq:hg_MG_det}
\end{equation}
where the expectation $\E$ is taken over
$(s,b)\sim g^{{\theta}\dagger}(b\mid s,f_{\theta}(s))
d_{\gamma_L}^{\theta\dagger}(s)$, and
\begin{equation}
    \nabla_{\theta}Q_F^{\theta\dagger}(s,f_{\theta}(s),b) \label{eq:nablaQF_MG_det}
    =\E_{\tau}\Bigg[\sum_{t=0}^\infty \gamma_F^t
    \left.\nabla_a Q_F^{\theta\dag}(s_t,a,b_t)\right|_{a=f_{\theta}(s_t)}\nabla_{\theta}f_{\theta}(s_t)\Bigg|\begin{array}{l}s_0=s\\b_0=b\end{array}\Bigg].
\end{equation}
\end{theorem}

Similarly to Theorem~\ref{thm:hg-mg-stochastic}, the first term can be interpreted as the direct improvement induced by changes in the leader policy, whereas the second term represents the indirect improvement induced by changes in the follower's optimal response.

\begin{proof}

In the following, we denote
$\E_{(s,b) \sim g^{{\theta}\dagger}(b\mid s,f_{\theta}(s)) d_{\gamma_\ast}^{\theta\dagger}(s)}$
by $\E_d^{\theta\gamma_\ast}$.

First, we derive the following expression for
$\nabla_{\theta} J_L(\theta,g^{\theta\dagger})$:
\begin{align}
    \nabla_{\theta} J_L(\theta,g^{\theta\dagger})
    =\dfrac{1}{1-\gamma_L}\E_d^{\theta\gamma_L}\Bigg[\nabla_{\theta}f_{\theta}(s)^{\T}\nabla_{a}Q_L^{\theta\dagger}(s,a,b)\big|_{a=f_{\theta}(s)}
    +
    \nabla_{\theta}\log g^{\theta\dagger}(b|s,f_{\theta}(s))Q_L^{\theta\dagger}(s,f_{\theta}(s),b)\Bigg].
    \label{eq:SPG_qr_bbf_det_first}
\end{align}
From the Bellman expectation equation for the leader value function,
\begin{equation}
    V_L^{\theta\dagger}(s)
    =\int_{\mathcal{B}}g^{\theta\dagger}(b|s,f_{\theta}(s))\left(r_L(s,f_{\theta}(s),b)+\gamma_L\int_{\mathcal{S}}p(s'|s,f_{\theta}(s),b)V_L^{\theta\dagger}(s') \rmd s' \right) \rmd b,
\end{equation}
taking the derivative of both sides with respect to $\theta$ yields
\begin{subequations}
\begin{align}
    &\nabla_{\theta} V_L^{\theta\dagger}(s)\nonumber\\
    &\begin{multlined}
        =\int_{\mathcal{B}} g^{\theta\dagger}(b\mid s,f_{\theta}(s))\bigg(\nabla_{\theta}\log g^{\theta\dagger}(b|s,f_{\theta}(s))\left(r_L(s,f_{\theta}(s),b)+\gamma_L\int_{\mathcal{S}} p(s'\mid s,f_{\theta}(s),b)V_L^{\theta\dagger}(s')\rmd s' \right)
        \\+\nabla_{\theta}f_{\theta}(s)^{\T}\left.\nabla_{a}r_L(s,a,b)\right|_{a=f_{\theta}(s)}
        +\gamma_L\int_{\mathcal{S}}\left(\nabla_{\theta}f_{\theta}(s)^{\T}\left.\nabla_{a}p(s'|s,a,b)\right|_{a=f_{\theta}(s)}V_L^{\theta\dagger}(s')
        +p(s'|s,a,b)\nabla_{\theta}V_L^{\theta\dagger}(s')\right)\rmd s' \bigg) \rmd b
    \end{multlined}\\
    &\begin{multlined}
        =\int_{\mathcal{B}} g^{\theta\dagger}(b\mid s,f_{\theta}(s))\bigg(\nabla_{\theta}\log g^{\theta\dagger}(b|s,f_{\theta}(s))\left(r_L(s,f_{\theta}(s),b)+\gamma_L\int_{\mathcal{S}} p(s'\mid s,f_{\theta}(s),b)V_L^{\theta\dagger}(s')\rmd s' \right)
        \\+\nabla_{\theta}f_{\theta}(s)^{\T}\nabla_{a}\left.\left(r_L(s,a,b)+\gamma_L\int_{\mathcal{S}}p(s'|s,a,b)V_L^{\theta\dagger}(s')\rmd s'\right)\right|_{a=f_{\theta}(s)}
        +\gamma_L\int_{\mathcal{S}}p(s'|s,f_{\theta}(s),b)\nabla_{\theta}V_L^{\theta\dagger}(s')\rmd s'\bigg) \rmd b
    \end{multlined}\\
    &\begin{multlined}
        =\int_{\mathcal{B}} g^{\theta\dagger}(b\mid s,f_{\theta}(s))\bigg(\nabla_{\theta}\log g^{\theta\dagger}(b|s,f_{\theta}(s))Q_L^{\theta\dagger}(s,f_{\theta}(s),b)
        \\+\nabla_{\theta}f_{\theta}(s)^{\T}\nabla_{a}\left.Q_L^{\theta\dagger}(s,a,b)\right|_{a=f_{\theta}(s)}
        +\gamma_L\int_{\mathcal{S}}p(s'|s,f_{\theta}(s),b)\nabla_{\theta}V_L^{\theta\dagger}(s')\rmd s'\bigg) \rmd b.
    \end{multlined}
\end{align}
\end{subequations}
This is a Bellman expectation equation with respect to $\nabla_{\theta} V_L^{\theta\dagger}$.
Therefore, by the uniqueness of the solution to the Bellman expectation equation, we obtain
\begin{equation}
    \nabla_{\theta} V_L^{\theta\dagger}(s)
    =\E_{\tau}\Bigg[\sum_{t=0}^{\infty}\gamma_L^t\Big(\nabla_{\theta}\log g^{\theta\dagger}(b_t|s_t,f_{\theta}(s_t))Q_L^{\theta\dagger}(s_t,f_{\theta}(s_t),b_t)
    +\nabla_{\theta}f_{\theta}(s_t)^{\T}\nabla_{a}Q_L^{\theta\dagger}(s_t,a,b_t)\big|_{a=f_{\theta}(s_t)}\Big)\Bigg|s_0=s\Bigg].
\end{equation}
By the definition of $J_L(\theta,g^{\theta\dagger})$, we have
$\nabla_{\theta}J_L(\theta,g^{\theta\dagger})=\E_{s\sim \rho_0}\left[\nabla_{\theta}V_L^{\theta\dagger}(s)\right]$.
Hence,
\begin{subequations}
\begin{align}
    \nabla_{\theta} J_L(\theta,g^{\theta\dagger})
    &=\E_{\tau}\Bigg[\sum_{t=0}^{\infty}\gamma_L^t\Big(\nabla_{\theta}\log g^{\theta\dagger}(b_t|s_t,f_{\theta}(s_t))Q_L^{\theta\dagger}(s_t,f_{\theta}(s_t),b_t)
    +\nabla_{\theta}f_{\theta}(s_t)^{\T}\nabla_{a}Q_L^{\theta\dagger}(s_t,a,b_t)\big|_{a=f_{\theta}(s_t)}\Big)\Bigg]\\
    &=\dfrac{1}{1-\gamma_L}\E_d^{\theta\gamma_L}\Bigg[\nabla_{\theta}f_{\theta}(s)^{\T}\nabla_{a}Q_L^{\theta\dagger}(s,a,b)\big|_{a=f_{\theta}(s)}
    +\nabla_{\theta}\log g^{\theta\dagger}(b|s,f_{\theta}(s))Q_L^{\theta\dagger}(s,f_{\theta}(s),b)\Bigg].
    \label{eq:SPG_qr_bbf_det_first}
\end{align}
\end{subequations}

Next, from
$g^{\theta\dagger}(b|s,a)=\exp\left(\dfrac{1}{\beta}\left(Q_F^{\theta\dagger}(s,a,b)-V_F^{\theta\dagger}(s,a)\right)\right)$
and the soft Bellman optimality equation for the follower value function,
\begin{equation}
    V_F^{\theta\dagger}(s,a)=\beta\log\int_{\mathcal{B}}\exp\left(\dfrac{1}{\beta}\left(Q_F^{\theta\dagger}(s,a,b)\right)\right)\rmd b,
\end{equation}
we obtain
\begin{equation}
    \nabla_{\theta}\log g^{\theta\dagger}(b|s,f_{\theta}(s))
    =\dfrac{1}{\beta}\left(\nabla_{\theta}Q_F^{\theta\dagger}(s,f_{\theta}(s),b) -\E_{b\sim g^{\theta\dagger}(\cdot|s,f_{\theta}(s))}\left[\nabla_{\theta}Q_F^{\theta\dagger}(s,f_{\theta}(s),b)\right]\right)\label{eq:nabla_1_log_g_qr_bbf_with_nabla_q}
\end{equation}
Substituting Eq.~\eqref{eq:nabla_1_log_g_qr_bbf_with_nabla_q}
into Eq.~\eqref{eq:SPG_qr_bbf_det_first} and applying the Boltzmann covariance trick yield
\begin{subequations}
\begin{align}
    \nabla_{\theta}J_L(\theta,g^{\theta\dagger})
    &\begin{multlined}[t]
        =\dfrac{1}{1-\gamma_L}\E_d^{\theta\gamma_L}\bigg[\nabla_{\theta}f_{\theta}(s)^{\T}\nabla_{a}Q_L^{\theta\dagger}(s,a,b)\big|_{a=f_{\theta}(s)}
        \\+\dfrac{1}{\beta}\bigg(\nabla_{\theta}Q_F^{\theta\dagger}(s,f_{\theta}(s),b) 
        -\E_{b\sim g^{\theta\dagger}(\cdot|s,f_{\theta}(s))}\left[\nabla_{\theta}Q_F^{\theta\dagger}(s,f_{\theta}(s),b)\right]\bigg)Q_L^{\theta\dagger}(s,f_{\theta}(s),b)\bigg]\\
    \end{multlined}\\
    &\begin{multlined}
        =\dfrac{1}{1-\gamma_L}\E_d^{\theta\gamma_L}\bigg[\nabla_{\theta}f_{\theta}(s)^{\T}\nabla_{a}Q_L^{\theta\dagger}(s,a,b)\big|_{a=f_{\theta}(s)}
        \\+\dfrac{1}{\beta}\E_{b\sim g^{\theta\dagger}(\cdot|s,f_{\theta}(s))}\bigg[\bigg(\nabla_{\theta}Q_F^{\theta\dagger}(s,f_{\theta}(s),b) 
        -\E_{b\sim g^{\theta\dagger}(\cdot|s,f_{\theta}(s))}\left[\nabla_{\theta}Q_F^{\theta\dagger}(s,f_{\theta}(s),b)\right]\bigg)Q_L^{\theta\dagger}(s,f_{\theta}(s),b)\bigg]\bigg]
    \end{multlined}\\
    &\begin{multlined}
        =\dfrac{1}{1-\gamma_L}\E_d^{\theta\gamma_L}\bigg[\nabla_{\theta}f_{\theta}(s)^{\T}\nabla_{a}Q_L^{\theta\dagger}(s,a,b)\big|_{a=f_{\theta}(s)}
        \\+\dfrac{1}{\beta}\E_{b\sim g^{\theta\dagger}(\cdot|s,f_{\theta}(s))}\bigg[\nabla_{\theta}Q_F^{\theta\dagger}(s,f_{\theta}(s),b) 
        \bigg(Q_L^{\theta\dagger}(s,f_{\theta}(s),b) 
        - \E_{b\sim g^{\theta\dagger}(\cdot|s,f_{\theta}(s))}\left[Q_L^{\theta\dagger}(s,f_{\theta}(s),b)\right]\bigg)\bigg]\bigg]
    \end{multlined}\\
    &\begin{multlined}
        =\dfrac{1}{1-\gamma_L}\E_d^{\theta\gamma_L}\bigg[\nabla_{\theta}f_{\theta}(s)^{\T}\nabla_{a}Q_L^{\theta\dagger}(s,a,b)\big|_{a=f_{\theta}(s)}
        \\+\dfrac{1}{\beta}\nabla_{\theta}Q_F^{\theta\dagger}(s,f_{\theta}(s),b) 
        \bigg(Q_L^{\theta\dagger}(s,f_{\theta}(s),b) 
        - \E_{b\sim g^{\theta\dagger}(\cdot|s,f_{\theta}(s))}\left[Q_L^{\theta\dagger}(s,f_{\theta}(s),b)\right]\bigg)\bigg].
    \end{multlined}\label{eq:nabla1_J1_nabla1_Q2_det}
\end{align}
\end{subequations}

Now, differentiating both sides of the soft Bellman optimality equation for the follower Q-function,
\begin{equation}
    Q_F^{\theta\dagger}(s,f_{\theta}(s),b)=r_F(s,f_{\theta}(s),b)+\gamma_F\int_{\mathcal{S}}p(s'|s,f_{\theta}(s),b)V_F^{\theta\dagger}(s',f_{\theta}(s'))\rmd s',
\end{equation}
with respect to $\theta$, we obtain
\begin{subequations}
\begin{align}
    \nabla_{\theta}Q_F^{\theta\dagger}(s,f_{\theta}(s),b)
    &\begin{multlined}[t]
        =\left(\nabla_{\theta}f_{\theta}(s)^{\T}\nabla_{a}r_F(s,a,b)
        +\gamma_F\int_{\mathcal{S}}\nabla_{\theta}f_{\theta}(s)^{\T}\nabla_{a}p(s'|s,a,b)V_F^{\theta\dagger}(s',f_{\theta}(s'))\rmd s'\right)\bigg|_{a=f_{\theta}(s)}
        \\+\gamma_F\int_{\mathcal{S}}p(s'|s,f_{\theta}(s),b)\nabla_{\theta}V_F^{\theta\dagger}(s',f_{\theta}(s'))\rmd s'
    \end{multlined}\\
    &\begin{multlined}
        =\nabla_{\theta}f_{\theta}(s)^{\T}\nabla_{a}Q_F^{\theta\dagger}(s,a,b)\big|_{a=f_{\theta}(s)}
        \\+\gamma_F\int_{\mathcal{S}} p(s'|s,f_{\theta}(s),b) \int_{\mathcal{B}} g^{\theta\dagger}(b'|s',f_{\theta}(s'))\nabla_{\theta}Q_F^{\theta\dagger}(s',f_{\theta}(s'),b')\rmd b' \rmd s'.
    \end{multlined}\label{eq:nabla_1_Q_F_qr_bbf_dpg_bellman_eq}
\end{align}
\end{subequations}
Equation~\eqref{eq:nabla_1_Q_F_qr_bbf_dpg_bellman_eq}
is a Bellman expectation equation with respect to
$\nabla_{\theta}Q_F^{\theta\dagger}$.
Therefore, by the uniqueness of the solution to the Bellman expectation equation, we obtain
\begin{equation}
    \nabla_{\theta}Q_F^{\theta\dagger}(s,f_{\theta}(s),b)
    =\E_{\tau}\left[\sum_{t=0}^\infty \gamma_F^t\nabla_{\theta}f_{\theta}(s_t)^{\T}\nabla_{a}Q_F^{\theta\dagger}(s_t,a,b_t)\big|_{a=f_{\theta}(s_t)}\middle|\begin{array}{l}s_0=s\\b_0=b\end{array}\right].
    \label{eq:nabla_1_Q_F_qr_bbf_dpg}
\end{equation}
Finally, substituting
Eq.~\eqref{eq:nabla_1_Q_F_qr_bbf_dpg}
into Eq.~\eqref{eq:nabla1_J1_nabla1_Q2_det}, we obtain
\begin{multline*}
    \nabla_{\theta}J_L(\theta,g^{\theta\dagger})
    =\dfrac{1}{1-\gamma_L}\E_d^{\theta\gamma_L}\Bigg[\nabla_{\theta}f_{\theta}(s)^{\T}\nabla_{a}Q_L^{\theta\dagger}(s,a,b)|_{a=f_{\theta}(s)}\\
    +\dfrac{1}{\beta}\left(Q_L^{\theta\dagger}(s,f_{\theta}(s),b)-\E_{b\sim g^{\theta\dagger}(\cdot|s,f_{\theta}(s))}\left[Q_L^{\theta\dagger}(s,f_{\theta}(s),b)\right]\right)\\
    \cdot\E_{\tau}\left[\sum_{t=0}^\infty \gamma_F^t\nabla_{\theta}f_{\theta}(s_t)^{\T}\nabla_{a}Q_F^{\theta\dagger}(s_t,a,b_t)\big|_{a=f_{\theta}(s_t)}\middle|\begin{array}{l}s_0=s\\b_0=b\end{array}\right]\Bigg],
\end{multline*}
which completes the proof.

\end{proof}

\end{document}